\documentclass{article}

    \PassOptionsToPackage{numbers, compress}{natbib}

\usepackage[final]{neurips_2021}

\usepackage[utf8]{inputenc} 
\usepackage[T1]{fontenc}    
\usepackage{hyperref}       
\hypersetup{colorlinks,linkcolor={blue},citecolor={red},urlcolor={red}}  

\usepackage{url}            
\usepackage{booktabs}       
\usepackage{amsfonts}       
\usepackage{nicefrac}       
\usepackage{microtype}      
\usepackage{xcolor}         

\usepackage[inline]{enumitem}

\usepackage{amsmath,amsfonts,bm}
\usepackage{amssymb,amsthm}
\usepackage{mathtools}
\usepackage{algorithm,algorithmic}
\usepackage{natbib}
\usepackage{xcolor}
\usepackage{hyperref}
\usepackage{url}
\usepackage{multirow}
\usepackage{array}
\usepackage{cancel}
\usepackage{etoc}
\usepackage{pifont}
\usepackage[capitalize]{cleveref}

\crefname{proposition}{Proposition}{Propositions}
\crefname{theorem}{Theorem}{Theorems}
\crefname{lemma}{Lemma}{Lemmas}
\crefname{update_rule}{Update}{Updates}
\crefname{algorithm}{Algorithm}{Algorithms}
\crefname{figure}{Figure}{Figures}

\def\eqref#1{equation~\ref{#1}}

\def\1{\bm{1}}

\DeclareMathAlphabet{\mathsfit}{\encodingdefault}{\sfdefault}{m}{sl}
\SetMathAlphabet{\mathsfit}{bold}{\encodingdefault}{\sfdefault}{bx}{n}

\def\gA{{\mathcal{A}}}
\def\gB{{\mathcal{B}}}

\def\gE{{\mathcal{E}}}
\def\gF{{\mathcal{F}}}

\def\gH{{\mathcal{H}}}

\def\gM{{\mathcal{M}}}

\def\gP{{\mathcal{P}}}

\def\gS{{\mathcal{S}}}

\def\gX{{\mathcal{X}}}

\def\sI{{\mathbb{I}}}

\def\sR{{\mathbb{R}}}

\newcommand{\R}{\mathbb{R}}

\newcommand{\softmax}{\mathrm{softmax}}

\DeclareMathOperator*{\argmax}{arg\,max}

\newtheorem{theorem}{Theorem}
\newtheorem{lemma}{Lemma}
\newtheorem{definition}{Definition}
\newtheorem{proposition}{Proposition}
\newtheorem{remark}{Remark}

\newtheorem{assumption}{Assumption}

\newtheorem{update_rule}{Update}

\newtheorem{conjecture}{Conjecture}

\DeclareMathOperator*{\expectation}{\mathbb{E}}

\def\rvzero{{\mathbf{0}}}

\def\diagonalmatrix{\text{diag}}

\DeclareMathOperator*{\probability}{Pr}

\newcommand{\cS}{\mathcal{S}}

\newlength\tocrulewidth
\newcommand{\Bo}[1]{{\color{blue} [Bo: #1]}}

\newcommand{\cmark}{\ding{51}}%
\newcommand{\xmark}{\ding{55}}%

\title{Understanding the Effect of Stochasticity \\
in Policy Optimization}

\author{%
  Jincheng Mei$^{\, 1 \, 3}$\thanks{Correspondence to: Jincheng Mei and Csaba Szepesv{\'a}ri  \texttt{\{jmei2,szepesva\}@ualberta.ca}}
  \hspace{1mm}
  Bo Dai$^{\, 3}$
  \hspace{1mm}
  Chenjun Xiao$^{\, 1 \, 3}$ 
  \hspace{1mm}
  Csaba Szepesv{\'a}ri$^{\, 2 \, 1 \dag \, *}$
  \hspace{1mm}
  Dale Schuurmans$^{\, 3 \, 1 \dag}$ \\
  \\
  \hspace{-7mm} $^1$\normalfont{University of Alberta} \hspace{1mm} $^2$DeepMind \hspace{1mm} $^3$Google Research, Brain Team \hspace{1mm} $^\dag$equal advising 
}

\begin{document}
\etocdepthtag.toc{mtmainpaper}

\maketitle

\begin{abstract}

We study the effect of stochasticity in on-policy policy optimization, and make the following four contributions. 
\textcolor{blue}{\textit{First}},
we show that 
the 
preferability
of optimization methods 
depends critically 
on whether stochastic versus exact gradients are used. 
In particular, 
unlike the true gradient setting, geometric information \emph{cannot} be easily exploited 
in the stochastic case
for accelerating policy 
optimization
without detrimental consequences or impractical assumptions.
%
\textcolor{blue}{\textit{Second}},
to explain these findings we introduce the concept of \textit{committal rate} for stochastic policy optimization, and show that this can serve as a criterion for determining almost sure convergence to global optimality.
\textcolor{blue}{\textit{Third}},
we show that 
in the absence of external oracle information, which allows an algorithm to determine the difference between optimal and sub-optimal actions given only on-policy samples,
there 
is
an inherent trade-off between exploiting geometry to accelerate convergence versus achieving optimality almost surely.  That is, an uninformed algorithm either converges to a globally optimal policy with probability $1$ but at a rate no better than $O(1/t)$, or it achieves 
faster than $O(1/t)$ convergence 
but then must fail to converge to the globally optimal 
policy with some positive probability.
\textcolor{blue}{\textit{Finally}},
we use 
the 
committal rate theory to explain why practical policy optimization methods are sensitive to random initialization, 
then develop an ensemble method 
that can be guaranteed to achieve near-optimal solutions with high probability.
\end{abstract}


\section{Introduction}
\label{sec:introduction}

Policy optimization is a central problem in reinforcement learning (RL) that provides a foundation for both policy-based and actor-critic RL methods. 
Until recently it had generally been assumed that methods based on following the
policy gradient (PG) \citep{sutton2000policy} could not be guaranteed to converge to globally optimal solutions,
given that the policy value function is not concave. 
However, this assumption has been contradicted by recent findings that policy gradient methods can indeed prove to converge to 
global optima,
at least in the tabular setting.
In particular, the standard softmax PG method with a constant learning rate has been shown to converge to a globally optimal policy at a $\Theta(1/t)$ rate for finite MDPs \citep{mei2020global}, albeit with challenging problem and initialization dependent constants \citep{mei2020escaping,li2021softmax}. 
Several techniques have been developed to further improve standard PG and achieve better rates and constants. 
For example, adding \textit{entropy regularization} has been shown to produce faster 
$O(e^{-c \cdot t})$
convergence ($c > 0$) to the optimal regularized policy \citep{mei2020global,cen2020fast,lan2021policy}. 
By exploiting natural geometries based on Bregman divergences,
\textit{natural PG (NPG) or mirror descent (MD)}
have been shown to achieve better constants than standard PG \citep{agarwal2020optimality,cen2020fast} and faster 
$O(e^{-c \cdot t})$
rates, with 
\citep{cen2020fast,lan2021policy} and without regularization \citep{khodadadian2021linear}. 
Alternative policy parameterizations, such as the escort parameterization, have been shown to improve the constants achieved by softmax and yield
faster plateau escaping \citep{mei2020escaping}. 
More recently, a \textit{geometry-aware normalized PG (GNPG)} approach has been proposed to exploit the non-uniformity of the value function,
achieving faster 
$O(e^{-c \cdot t})$
rates with improved constants \citep{mei2021leveraging}.

A key observation is that each of these four techniques---%
\textit{(i)} entropy regularization, \textit{(ii)} NPG (or MD), \textit{(iii)} alternative policy parameterization, and \textit{(iv)} GNPG---%
accelerates the convergence of standard softmax PG by better exploiting the geometry of the optimization landscape. 
In particular, entropy regularization makes the regularized objective behave more like a quadratic \citep{mei2020global,cen2020fast,lan2021policy}, which significantly improves the near-linear character of the softmax policy value \citep{mei2020global}.
Natural PG (or MD) performs non-Euclidean updates in the parameter space, which is quite different from the Euclidean geometry characterizing standard softmax PG updates \citep{agarwal2020optimality,cen2020fast,khodadadian2021linear}. 
The escort policy parameterization induces an alternative policy-parameter relation 
\citep{mei2020escaping}. 
GNPG exploits the non-uniform smoothness in the optimization landscape via a simple gradient normalization operation~\citep{mei2021leveraging}. 

However, these advantages have only been established for the true gradient setting. 
A natural question therefore is whether geometry can also be exploited to accelerate convergence to global optimality in \emph{stochastic} gradient settings. 
In this paper, we show that in a certain fundamental sense, the answer is \emph{no}.
That is, there exists a fundamental trade-off between leveraging geometry to accelerate convergence and overcoming the noise introduced by stochastic gradients (possibly infinite); in particular, no uninformed algorithm can improve the $O(1/t)$ convergence rate 
without incurring a positive probability of failure (i.e. diverging or converging to a sub-optimal stationary point).

The conditions used in vanilla stochastic gradient convergence analysis, \emph{i.e.}, unbiased and variance-bounded gradient estimator \citep{nemirovski2009robust}, has been exploited to attempt to explain such a trade-off in policy gradients \citep{abbasi2019politex,lan2021policy}. However, the bounded variance requires the sample policy to be bounded away from zero everywhere, which is impractical. Meanwhile, a variant of NPG can converge even with unbounded variance \citep{chung2020beyond}. These gaps raise the question that if not the bounded variance, then what is the key factor to ensure the convergence of stochastic policy optimization algorithms?
Motivated by this question, we introduce the concept \textit{committal rate} to characterize the update behaviors, which significantly affect whether convergence to a correct solution can be guaranteed in the stochastic on-policy setting.
%
%
%
In particular, we make the following contributions.
\begin{itemize}[noitemsep,topsep=1pt,parsep=1pt,partopsep=0pt, leftmargin=18pt]
	\item \emph{First}, we illustrate the anomaly that the preferability of policy optimization algorithms (softmax PG vs. NPG and GNPG) changes dramatically depending on whether true versus on-policy stochastic gradients are considered, and reveal the impracticality and unnecessity of a bounded variance requirement in \cref{sec:turnover_results};
	\item \emph{Second}, we introduce the concept of the \textit{committal rate} in \cref{sec:committal_rate} to characterize the aggressiveness of an update, which provide us tools for analyzing the stochasticity effect in convergences; 
	\item \emph{Third}, we use the committal rate to study general stochastic policy optimization behaviors rigorously and reveal the inherent geometry-convergence trade-off in \cref{sec:geometry_convergence_tradeoff};
	\item \emph{Finally}, we explain the sensitivity to random initialization in practical policy optimization algorithms. From these results, we then develop an ensemble method that can achieve fast convergence to global optima with high probability in \cref{sec:ensemble_method}. 
\end{itemize}

\section{Understanding Algorithm Preferability in On-line Policy Optimization}
\label{sec:turnover_results}

To illustrate the key aspects of policy optimization methods and their comparative preferability, it suffices to consider
deterministic, single-state, finite-action Markov decision processes (MDPs). 
The main results extend to general finite MDPs, but for clarity of exposition we restrict attention to one-state MDPs.


A deterministic, single-state, finite-action MDP can be simply be specified by 
an action space is $[K] \coloneqq \left\{ 1, 2, \dots K \right\}$ and
a $K$-dimensional reward vector $r \in \sR^K$.
The problem is to maximize the expected reward of a parametric policy $\pi_\theta$,
\begin{align}
\label{eq:expected_reward_objective}
    \max_{\theta : [K] \to \sR}{ \expectation_{a \sim \pi_{\theta}(\cdot)}{ [ r(a) ]  } }.
\end{align}
where $\pi_\theta$ is parameterized by $\theta$ using the standard softmax transform,
\begin{align}
\pi_\theta(a) = \frac{ \exp\{ \theta(a) \} }{ \sum_{a^\prime \in [K]}{ \exp\{ \theta(a^\prime) } \} } \mbox{, \quad   for all } a \in [K].
\end{align}
Without loss of generality, we assume there exists a unique optimal action $a^* = \argmax_{a \in [K]}{ r(a) }$, 
hence there exists a unique optimal deterministic policy $\pi^*$ such that ${\pi^*}^\top r = \sup_{\theta \in \sR^K}{ \pi_\theta^\top r} = r(a^*)$. We make the following assumption on the reward.
\begin{assumption}[Positive reward]
\label{asmp:positive_reward}
$r(a) \in (0,1]$, $\forall a \in [K]$.
\end{assumption}

\subsection{Exact Gradient Setting}

It is known that \cref{eq:expected_reward_objective} is a non-concave maximization over the policy parameter $\theta$ \citep{mei2020global}. Nevertheless, it has recently become better understood how policy gradient (PG) methods still converge to global optima for \cref{eq:expected_reward_objective} when exact gradients are used.
To illustrate the main considerations, we focus on the following three representative algorithms that have recently been proved to achieve convergence to global optima but at different rates: 
softmax policy gradient (PG), natural PG (NPG), and geometry-aware normalized PG (GNPG), while similar conclusions can be drawn for other variants \citep{chung2020beyond,denisov2020regret}.

\subsubsection{Softmax PG}

The standard softmax PG method is specified by the following update.
\begin{update_rule}[Softmax PG, true gradient]
\label{update_rule:softmax_pg_special_true_gradient}
$\theta_{t+1} \gets \theta_{t} + \eta \cdot \frac{d \pi_{\theta_t}^\top r}{d \theta_t}$, where $\frac{d \pi_{\theta}^\top r}{d \theta} = \left( \diagonalmatrix{\left( \pi_{\theta} \right)} - \pi_{\theta} \pi_{\theta}^\top \right) r$, and thus $\frac{d \pi_{\theta}^\top r}{d \theta(a)} = \pi_{\theta}(a) \cdot (r(a) - \pi_{\theta}^\top r )$ for all $a \in [K]$.
\end{update_rule}
As shown in \citet{mei2020global},
the convergence of this update to a globally optimal policy, 
given exact gradients, 
can be established by considering the following
non-uniform \L{}ojasiewicz (N\L{}) inequality,
\begin{lemma}[N\L{}, \citep{mei2020global}]
\label{lem:non_uniform_lojasiewicz_softmax_special}
$\Big\| \frac{d \pi_\theta^\top r}{d \theta} \Big\|_2 \ge \pi_\theta(a^*) \cdot ( \pi^* - \pi_\theta )^\top r$.
\end{lemma}
By considering smoothness of $\pi_\theta^\top r$,
\citet{mei2020global} shows that
the progress in each iteration of PG can be lower bounded by the squared norm of the gradient, $\Big\| \frac{d \pi_{\theta_t}^\top r}{d \theta_t} \Big\|_2^2$, which leads to a $O(1/ t)$ rate.
\begin{proposition}[PG upper bound \cite{mei2020global}]
\label{prop:upper_bound_softmax_pg_special_true_gradient}
Using \cref{update_rule:softmax_pg_special_true_gradient} with $\eta = 2/5$, 
we have 
$( \pi^* - \pi_{\theta_t} )^\top r \le 5 / (c^2 \cdot t)$
for all $t \ge 1$,
such that $c = \inf_{t\ge 1} \pi_{\theta_t}(a^*) > 0$ is a constant that depends on $r$ and $\theta_1$, but it does not depend on the time $t$. 
In particular, if $\pi_{\theta_1}(a) = 1/K$ $\forall a$ then $c \ge 1/K$.
\end{proposition}
\begin{proposition}[PG lower bound \cite{mei2020global}]
\label{prop:lower_bound_softmax_pg_special_true_gradient}
For sufficiently large $t \ge 1$,  \cref{update_rule:softmax_pg_special_true_gradient} with $\eta \in ( 0 , 1]$
exhibits $(\pi^* - \pi_{\theta_t})^\top r \ge \Delta^2 / \left( 6 \cdot t \right)$, 
where $\Delta = r(a^*) - \max_{a \not= a^*}{ r(a) } > 0$ is the reward gap of $r$.
\end{proposition}
\begin{remark}
The constant dependence of PG follows a $\Omega(1/c)$ lower bound for one-state MDPs \citep{mei2020escaping}, 
while $c$ can be exponentially small in terms of the number of states for general finite MDPs \citep{li2021softmax}.  
\end{remark}
To summarize, using $\eta \in O(1)$, softmax PG achieves convergence to a global optima, but with a $\Theta(1/t)$ rate that exhibits poor constant dependence.

\subsubsection{Natural PG (NPG)}
An alternative method, natural PG (NPG) \citep{kakade2002natural}, provides the prototype for many practical policy optimization methods, such as TRPO and PPO \citep{schulman2015trust,schulman2017proximal}.
NPG is based on the following update.
\begin{update_rule}[Natural PG (NPG), true gradient]
\label{update_rule:softmax_natural_pg_special_true_gradient}
$\theta_{t+1} \gets \theta_{t} + \eta \cdot r$.
\end{update_rule}
For softmax policies, it turns out that \cref{update_rule:softmax_natural_pg_special_true_gradient} is identical to mirror descent (MD) with a Kullback-Leibler (KL) divergence. 
Therefore a standard MD analysis shows that \cref{update_rule:softmax_natural_pg_special_true_gradient} achieves convergence to a global optimum at a rate of $O(1/t)$ \citep{agarwal2020optimality}. 
Very recently, work concurrent to this submission \citep{khodadadian2021linear} has shown that \cref{update_rule:softmax_natural_pg_special_true_gradient} actually enjoys a much faster 
$O(e^{-c \cdot t})$
rate. 
In fact, here too we can establish the same 
$O(e^{-c \cdot t})$
rate, but using a simpler argument based on the following variant of the N\L{} inequality for natural gradients.
These results are new to this paper. {\bf Due to space limitation, we postpone all the proofs to the appendix.}
\begin{lemma}[Natural N\L{} inequality, continuous]
\label{lem:natural_lojasiewicz_continuous_special}
$\big\langle \frac{d \pi_\theta^\top r}{d \theta}, r \big\rangle \ge \pi_\theta(a^*) \cdot \Delta \cdot \left( \pi^* - \pi_\theta \right)^\top r$.
\end{lemma}
\begin{lemma}[Natural N\L{}, discrete]
\label{lem:natural_lojasiewicz_discrete_special}
Let $\pi^\prime(a) \coloneqq \frac{ \pi(a) \cdot e^{\eta \cdot r(a)} }{ \sum_{a^\prime}{ \pi(a^\prime) \cdot e^{\eta \cdot r(a^\prime)} } }$, $\forall \ a \in [K]$, where $\eta > 0$. 
Then,
\begin{align}
    \left( \pi^\prime - \pi \right)^\top r &\ge \left[ 1 - \frac{1}{ \pi(a^*) \cdot \left( e^{ \eta \cdot \Delta } - 1 \right) + 1} \right] \cdot \left( \pi^* - \pi \right)^\top r.
\end{align}
\end{lemma}
In particular, by using a non-Euclidean update and analysis, the progress of each iteration of NPG can be lower bounded by the larger bound $\big\langle \frac{d \pi_{\theta_t}^\top r}{d \theta_t}, r \big\rangle$
instead of the weaker bound $\Big\| \frac{d \pi_{\theta_t}^\top r}{d \theta_t} \Big\|_2^2$ established for standard PG. 
Based on this inequality, one can easily establish a much faster
$O(e^{-c \cdot t})$
convergence to a globally optimal solution for NPG,
making it far preferable to PG if true gradients are available.
\begin{theorem}[NPG upper bound]
\label{thm:npg_rate_discrete_special}
Using \cref{update_rule:softmax_natural_pg_special_true_gradient} with any $\eta > 0$, we have, for all $t \ge 1$,
\begin{align}
    \left( \pi^* - \pi_{\theta_t} \right)^\top r \le \left( \pi^* - \pi_{\theta_{1}} \right)^\top r \cdot e^{ - c \cdot (t - 1) },
\end{align}
where $c \coloneqq \log{ \left( \pi_{\theta_{1}}(a^*) \cdot \left( e^{ \eta \cdot \Delta } - 1 \right) + 1 \right) } > 0$ for any $\eta > 0$, and $\Delta = r(a^*) - \max_{a \not= a^*}{ r(a) } > 0$.
\end{theorem}

\subsubsection{Geometry-aware Normalized PG (GNPG)}

The Geometry-aware Normalized PG (GNPG) update is investigated in \citep{mei2021leveraging} to accelerate the convergence of PG
by exploiting local smoothness properties of the optimization landscape.
\begin{update_rule}[Geometry-aware Normalized PG (GNPG), true gradient]
\label{update_rule:softmax_gnpg_special_true_gradient}
$\theta_{t+1} \gets \theta_t + \eta \cdot \frac{d \pi_{\theta_t}^\top r}{d {\theta_t}} \big/ \Big\| \frac{d \pi_{\theta_t}^\top r}{d {\theta_t}} \Big\|_2$.
\end{update_rule}
The analysis in \citep{mei2021leveraging} focuses on exploiting non-uniform smoothness (NS) rather than improving the N\L{} inequality as for NPG above.
\begin{lemma}[NS, \citep{mei2021leveraging}]
\label{lem:non_uniform_smoothness_softmax_special}
The spectral radius of Hessian matrix $\frac{d^2 \pi_\theta^\top r}{d \theta^2}$ is upper bounded by $3 \cdot \Big\| \frac{d \pi_\theta^\top r}{d \theta} \Big\|_2$.
\end{lemma}
Given this NS property, \citep{mei2021leveraging} shows that the progress in GNPG can be lower bounded by the larger quantity $\Big\| \frac{d \pi_{\theta_t}^\top r}{d \theta_t} \Big\|_2$ instead of the weaker $\Big\| \frac{d \pi_{\theta_t}^\top r}{d \theta_t} \Big\|_2^2$ for standard PG. 
Then, using the same N\L{} inequality as for PG, GNPG also converges to a globally optimal solution at rate
$O(e^{-c \cdot t})$.
Again, one naturally concludes that GNPG is preferable to PG if exact gradients are used.
\begin{proposition}[GNPG upper bound \citep{mei2021leveraging}]
\label{prop:upper_bound_softmax_gnpg_special_true_gradient}
Using \cref{update_rule:softmax_gnpg_special_true_gradient} with $\eta = 1/6$, we have, for all $t \ge 1$,  
\begin{align}
    ( \pi^* - \pi_{\theta_t} )^\top r \le \left( \pi^* - \pi_{\theta_{1}}\right)^\top r \cdot e^{ - \frac{  c \cdot (t-1) }{12} },
\end{align}
where $c = \inf_{t\ge 1} \pi_{\theta_t}(a^*) > 0$ does not depend on $t$.
If $\pi_{\theta_1}(a) = 1/K$, $\forall a$, then $c \ge 1/K$.
\end{proposition}

\subsection{The Anomalous Behaviour of Some On-policy Stochastic Gradient Updates}

Although the above results show that exploiting geometric information can allow linear convergence to an optimal solution given true gradients---%
obviously $O(e^{-c \cdot t})$
represents an exponential speedup over the $\Omega(1/t)$ lower bound for standard PG---%
it is critical to understand whether such advantages can also be obtained in the more natural stochastic gradient setting.
Given the previous results, it would seem natural to prefer accelerated algorithms over PG in practice,
and there is some evidence that such thinking has become mainstream based on the popularity of TRPO and PPO over PG. Indeed, TRPO and PPO are often interpreted as instances of NPG and the faster convergence of NPG is used to explain their empirical success.
However, by more closely examining the behavior of these algorithms when true gradients are replaced by on-policy stochastic estimates,
serious shortcomings begin to emerge, as empirically observed in \citet{chung2020beyond}, and it is far from obvious that similar advantages from the true gradient case might be recoverable in the more practical stochastic scenario.

We begin by examining the behavior of the previous algorithms in the context of on-policy stochastic gradients.
To enable this analysis, first note that each of the above PG methods, \cref{update_rule:softmax_pg_special_true_gradient,update_rule:softmax_natural_pg_special_true_gradient,update_rule:softmax_gnpg_special_true_gradient},
can be adapted to the stochastic setting by using on-policy importance sampling (IS) to provide an unbiased estimate of the true reward. We do not make assumptions like each action is sufficiently explored, since $\pi_{\theta_t}$ is the behaviour policy as well as the policy to be optimized. It is possible that $\pi_{\theta_t}$ approaches a near deterministic policy, ruling out positive results based on such assumptions  \citep{abbasi2019politex}.

\begin{definition}[On-policy IS]
\label{def:on_policy_importance_sampling}
At iteration $t$, sample one action $a_t \sim \pi_{\theta_t}(\cdot)$. The IS reward estimator $\hat{r}_t$ is constructed as $\hat{r}_t(a) = \frac{ \sI\left\{ a_t = a \right\} }{ \pi_{\theta_t}(a) } \cdot r(a)$ for all $a \in [K]$.
\end{definition}
\begin{remark}
We consider sampling one action in each iteration, but the results continue to hold for sampling a constant $B > 0$ mini-batch of actions. 
A significant limitation of our results is that the reward is observed without noise, which is an idealized case. 
It remains to be seen which conclusions of this work can be extended to the more general case when the rewards are observed in noise.
\end{remark}
In the next subsections we consider the mentioned three on-policy update rules. As we shall see, only the first update rule, vanilla policy gradient with softmax parameterization is sound.
\subsubsection{Softmax PG}

\begin{update_rule}[Softmax PG, on-policy stochastic gradient]
\label{update_rule:softmax_pg_special_on_policy_stochastic_gradient}
$\theta_{t+1} \gets  \theta_{t} + \eta \cdot \frac{d \pi_{\theta_t}^\top \hat{r}_t}{d \theta_t}$, where $\frac{d \pi_{\theta_t}^\top \hat{r}_t}{d \theta_t(a)} = \pi_{\theta_t}(a) \cdot ( \hat{r}_t(a) - \pi_{\theta_t}^\top \hat{r}_t )$ for all $a \in [K]$.
\end{update_rule}
Using the IS reward estimate, the softmax PG is unbiased and bounded by constant:
\begin{lemma}
\label{lem:unbiased_bounded_variance_softmax_pg_special_on_policy_stochastic_gradient}
Let $\hat{r}$ be the IS estimator using on-policy sampling $a \sim \pi_{\theta}(\cdot)$. The stochastic softmax PG estimator is unbiased and bounded, i.e., $\expectation_{a \sim \pi_\theta(\cdot)}{ \left[ \frac{d \pi_\theta^\top \hat{r} }{d \theta} \right] } = \frac{d \pi_\theta^\top r}{d \theta}$, and $\expectation_{a \sim \pi_\theta(\cdot)}{ \left\| \frac{d \pi_\theta^\top \hat{r} }{d \theta} \right\|_2^2 } \le 2$.
\end{lemma}
These observations imply that stochastic softmax PG converges to a global optimum in probability, which was also proved by \citet{chung2020beyond}. Here we use the non-uniform smoothness in \cref{lem:non_uniform_smoothness_softmax_special} to prove that $\expectation_{a_t \sim \pi_{\theta_t}(\cdot)}{ \big[ \left( \pi^* - \pi_{\theta_t} \right)^\top r \big] } \in O(1/\sqrt{t}) \to 0$ as $t \to \infty$, which implies that $\lim_{t \to \infty}{ \probability{\big( \left( \pi^* - \pi_{\theta_t} \right)^\top r > 0 \big) } } \to 0$, i.e., sub-optimality converges to $0$ in probability.
\begin{theorem}
\label{thm:almost_sure_global_convergence_softmax_pg_special_on_policy_stochastic_gradient}
Using \cref{update_rule:softmax_pg_special_on_policy_stochastic_gradient}, $\left( \pi^* - \pi_{\theta_t} \right)^\top r \to 0$ as $t \to \infty$ in probability.
\end{theorem}

\subsubsection{NPG}

Similarly, we can use on-policy IS estimation to adapt NPG to the stochastic setting.
\begin{update_rule}[NPG, on-policy stochastic gradient]
\label{update_rule:softmax_natural_pg_special_on_policy_stochastic_gradient}
$\theta_{t+1} \gets \theta_{t} + \eta \cdot \hat{r}_t$.
\end{update_rule}
Although the NPG is unbiased, its variance can be possibly unbounded in the on-policy setting.
\begin{lemma}
\label{lem:bias_variance_softmax_natural_pg_special_on_policy_stochastic_gradient}
For NPG, we have, $\expectation_{a \sim \pi_\theta(\cdot)}{ \left[ \hat{r} \right] } = r$, and $\expectation_{a \sim \pi_\theta(\cdot)}{ \left\| \hat{r} \right\|_2^2 } = \sum_{a \in [K]}{ \frac{ r(a)^2 }{ \pi_{\theta}(a) }  }$.
\end{lemma}
The variance becomes unbounded as $\pi_\theta(a)\to 0$, which predicts trouble when using the standard analysis for stochastic gradient methods\footnote{Standard treatment of stochastic approximation algorithms does deal with unbounded noise in a controlled way to still get positive results \citep{BMP90}, which means that bounded variance is far from being necessary.} (e.g., \citep{nemirovski2009robust}). In fact, we provide a more direct result showing that stochastic NPG has a positive probability of converging to a sub-optimal deterministic policy.
\begin{theorem}
\label{thm:failure_probability_softmax_natural_pg_special_on_policy_stochastic_gradient}
Using \cref{update_rule:softmax_natural_pg_special_on_policy_stochastic_gradient}, we have: \textbf{(i)} with positive probability, $\sum_{a \not= a^*}{ \pi_{\theta_t}(a)} \to 1$ as $t \to \infty$; \textbf{(ii)} $\forall a \in [K]$, with positive probability, $\pi_{\theta_t}(a) \to 1$, as $t \to \infty$.
\end{theorem}
This result extends the result 
of \citep{chung2020beyond} for the two-action ($K=2$) case only.
The intuition is that the stochastic NPG accumulates too much probability on sampled sub-optimal actions and cannot recover due to the ``vicious circle'' between sampling and updating.

\subsubsection{GNPG}

Finally, we consider the stochastic version of GNPG. 
\begin{update_rule}[GNPG, on-policy stochastic gradient]
\label{update_rule:softmax_gnpg_special_on_policy_stochastic_gradient}
$\theta_{t+1} \gets \theta_t + \eta \cdot \frac{d \pi_{\theta_t}^\top \hat{r}_t}{d {\theta_t}} \Big/ \Big\| \frac{d \pi_{\theta_t}^\top \hat{r}_t}{d {\theta_t}} \Big\|_2$.
\end{update_rule}
Unfortunately, this estimator involves a ratio of random variables, and its bias can be large. 
As for NPG we can show that stochastic GNPG fails with positive probability in the stochastic case.
\begin{theorem}
\label{thm:failure_probability_softmax_gnpg_special_on_policy_stochastic_gradient}
Using \cref{update_rule:softmax_gnpg_special_on_policy_stochastic_gradient}, 
we have, $\forall a \in [K]$, with positive probability, $\pi_{\theta_t}(a) \to 1$, as $t \to \infty$.
\end{theorem}

\subsection{Why Consider the On-policy Stochastic Setting?}

The findings of the previous sections are summarized in \cref{tab:turnover_results_summary}. The two methods that converge faster when the exact gradient is available are exactly those that fail in the worse possible way in the on-policy setting. This raises the question of  should one even consider the on-policy setting?

\begin{table}[ht]
\begin{center}
\begin{tabular}{ m{8em} | m{9em} | m{8em} | m{8em} } 
\hline
 & Softmax PG & NPG & GNPG \\
\hline
True gradient & converges $\Theta(1/t)$ \cmark\cmark & 
converges $O(e^{-c \cdot t})$ \cmark\cmark\cmark
& converges $O(e^{-c \cdot t})$ \cmark\cmark\cmark
\\
\hline
Stochastic on-policy & converges in prob. \cmark & fails w.p.\ $>0$ \xmark & fails w.p.\ $>0$ \xmark \\
\hline
\end{tabular}
\end{center}
\caption{Convergence properties of softmax PG, NPG and GNPG in the alternative settings.}
\label{tab:turnover_results_summary}
\vspace{-0.2in}
\end{table}

One possible reason to consider this setting is because on-policy sampling is the simplest and most straightforward approach to extend algorithms developed for the ``exact gradient'' setting
and with a minor twist, Occam's razor dictates that one should consider simple solutions before considering more complex ones.
Indeed, off-policy algorithms are more complex with many more choices to be made and while having the extra freedom may ultimately be useful (and even perhaps necessary), it is worthwhile to first thoroughly examine whether this complexity can be avoided.
Indeed, there is some empirical evidence that the simple, on-policy approach may sometimes be a reasonable one: The method PPO \citep{schulman2017proximal} uses on-policy sampling and yet, remarkably, it achieved outstanding results on challenging tasks, a good example of which is to learn dexterous in-hand manipulation \citep{andrychowicz2020learning}.

A second reason is that the on-policy setting presents unique challenges and as such is interesting on its own for learning about how to design and reason about stochastic methods. Indeed, the standard approach in analyzing stochastic update rules, such as SGD, is to start with the assumption that 
the gradient estimates are unbiased and have a uniformly bounded variance. This has been used both in the analysis of SGD \citep{nemirovski2009robust}, and later adopted to policy gradient methods \citep{abbasi2019politex,lan2021policy,zhang2020sample,zhang2021convergence}. 
However, such conditions are
only \emph{sufficient} and not necessary as the numerous results in the literature of the analysis and design of stochastic approximation methods also show \citep{BMP90}.
In fact, the bounded variance assumption can be difficult to satisfy.
For example, in the problems studied here this assumption requires that the probabilities induced by a behaviour policy are bounded away from $0$ everywhere \citep{chung2020beyond}, which is impractical for large state and action spaces and impossible when they are infinite. 



Another observation that suggests that it is worthwhile to consider methods which potentially unbounded variance is made by \citet{chung2020beyond} who explored the role of baselines in policy optimization. 
They show that variance reduction techniques are not able to overcome unbounded variance, while NPG can still achieve global convergence almost surely with a judicious choice of baseline even though its variance remains \emph{unbounded} (see \cref{update_rule:softmax_natural_pg_special_on_policy_stochastic_gradient_oracle_baseline} for details). This is another example that shows that bounded variance is not necessary for convergence, and some other factors rather than variance account for the convergence behaviour of stochastic policy optimization algorithms.

This leave us an important question to be answered to bridge the gap between theory and practice,
\begin{center}
\vspace{-0.5em}
    \emph{What are the key factors determining the convergence of stochastic policy optimization?}
\vspace{-0.5em}
\end{center}
As an answer to this question we propose a new notion, the
\emph{committal rate} of policy optimization methods and will demonstrate that small committal rates are necessary to ensure the convergent behavior of policy optimization methods.

\section{Committal Rate of Stochastic Policy Optimization Algorithms}
\label{sec:committal_rate}

Although the baseline study \citep{chung2020beyond} only focuses on two- and three-action bandits primarily,
it develops a useful intuition that stochastic policy optimization in practical settings consists of separate ``sampling'' and ``updating'' steps that become coupled in the on-policy setting.
Building from this observation, and seeking to explain the outcomes in \cref{sec:turnover_results}, 
we formalize the following ``committal rate'' function of a policy optimization algorithm. The main idea is to decouple the ``sampling'' and ``updating'' by fixing sampling one action and characterizing the aggressiveness of an update in a deterministic way. 
Thus, in what follows, by a policy optimization algorithm $\gA$ we mean a mapping from 
all sequences of pairs of action-reward pairs to the set of parameter vectors.
\begin{definition}[Committal Rate]
\label{def:committal_rate}
Fix a reward function $r \in (0, 1]^K$ 
and an initial parameter vector $\theta_1 \in \sR^K$.
Consider a policy optimization algorithm $\gA$. 
Let action $a$ be the sampled action \textbf{forever} after initialization and let $\theta_t$ be the resulting parameter vector obtained by using $\gA$ on the first $t$ observations.
The committal rate of algorithm $\gA$ on action $a$ (given $r$ and $\theta_1$) is then defined as  
\begin{align}
    \kappa(\gA, a) = \sup\left\{ \alpha \ge 0: \limsup_{t \to \infty}{ t^\alpha \cdot  \left[ 1 - \pi_{\theta_t}(a) \right] < \infty} \right\}.
\end{align}
\end{definition}
Note that in the definition we have suppressed the dependence of $\kappa$ on the rewards and the initial parameter vector.
\cref{def:committal_rate} accounts for \textbf{how aggressive an update rule is}:
An algorithm with committal rate $\alpha$ will make 
$\pi_{\theta_t}(a)$ approach $1$ at the polynomial rate of $1/t^\alpha$ 
provided that the sampling rule only chooses action $a$.
Thus,
a larger value of $\kappa(\gA, a)$ indicates an algorithm that quickly commits to the action $a$.
For example, if $\pi_{\theta_t}(a) = 1 - 1/(t \cdot \log{(t)})$, then $\kappa(\gA, a) = 1$. 
Similarly, if $\pi_{\theta_t}(a) = 1 - 1/e^t$, then $\kappa(\gA, a) = \infty$, which means $\pi_{\theta_t}(a)$ approaches $1$ extremely quickly. 
On the other hand, if $1 - \pi_{\theta_t}(a) \in \Omega(1)$, then $\kappa(\gA, a) = 0$, implying that $\pi_{\theta_t}$ never becomes committal, since $\pi_{\theta_t}(a)$ never approaches $1$.

Our next results shows that a small committal rate with respect to
sub-optimal actions is necessary for 
almost sure convergence to a globally optimal policy.
\begin{theorem}[Committal rate main theorem]
\label{thm:committal_rate_main_theorem}
Consider a policy optimization method $\gA$, together with $r\in (0,1]^K$ and an initial parameter vector $\theta_1\in \sR^K$.
Then,
\begin{align}
\max_{a:r(a)<r(a^*),\pi_{\theta_1}(a)>0} \kappa(\gA, a) \le 1
\label{eq:neccond}
\end{align}
is a necessary condition for ensuring the almost sure convergence of the policies obtained using $\gA$ and online sampling to the global optimum starting from $\theta_1$.
\end{theorem}
In words, \cref{eq:neccond} shows that slow reaction to constantly sampling sub-optimal actions is necessary for the success of policy optimization methods when they are used with online sampling.

Using this result, we can now interrogate the committal rates of the previously listed algorithms.
\begin{theorem}
\label{thm:committal_rate_stochastic_npg_gnpg}
Let \cref{asmp:positive_reward} holds. For the stochastic updates NPG and GNPG from \cref{update_rule:softmax_natural_pg_special_on_policy_stochastic_gradient,update_rule:softmax_gnpg_special_on_policy_stochastic_gradient}  
we obtain $\kappa(\text{NPG}, a) = \infty$ and $\kappa(\text{GNPG}, a) = \infty$ for all $a \in [K]$ respectively.
\end{theorem}
\cref{thm:committal_rate_stochastic_npg_gnpg} explains why stochastic NPG and GNPG have a  non-zero failure probability in the on-policy stochastic setting: they do not obey a necessary condition for almost sure global convergence. 
Intuitively, these algorithms can fail by prematurely allocating too much probability to a sub-optimal action:
each sampling of an action $a\in[K]$ increments its parameter by $\Theta(1)$, so if $a$ is sampled $t$ times successively,
then we have $1 - \pi_{\theta_t}(a) \in O(e^{-c \cdot t})$,
which means $\kappa(\gA, a) = \infty$. 
According to \cref{thm:committal_rate_main_theorem}, there is a positive probability that a single sub-optimal action can receive a long enough sampling run to ensure the other actions will never again be sampled.

By contrast, we can compare these outcomes to the committal rate of the softmax PG algorithm.
\begin{theorem}
\label{thm:committal_rate_softmax_pg}
Let $r(a)>0$ and $\pi_{\theta_1}(a)>0$. Softmax PG obtains
$\kappa(\text{PG}, a) = 1$ for all $a \in [K]$.
\end{theorem}
\cref{thm:committal_rate_main_theorem,thm:committal_rate_softmax_pg} provide (partial) explanations of the observations in \cref{sec:turnover_results}: stochastic NPG and GNPG can fail while PG almost surely converges to a global optimum, but their committal rates lie on different sides of the necessary condition. 
Since $\kappa(\text{PG}, a) = 1$ for softmax PG, it follows that $\prod_{t=1}^{\infty}{ \pi_{\theta_t}(a) } = 0$ (see \cref{lem:positive_infinite_product_4}), hence it is not possible to sample sub-optimal actions forever, and the optimal action $a^*$ always has a sufficient chance to be sampled, which ensures learning.

Next, following \citep{chung2020beyond}, we consider NPG using an ``oracle baseline'', which assumes the knowledge of the gap $\Delta$. 
\citet{chung2020beyond} considered this baseline to point out that convergence in on-policy stochastic gradient methods can happen even if the variance of the gradient estimates ``explodes'':
\begin{update_rule}[NPG with oracle baseline]
\label{update_rule:softmax_natural_pg_special_on_policy_stochastic_gradient_oracle_baseline}
$\theta_{t+1} \gets \theta_{t} + \eta \cdot \big( \hat{r}_t - \hat{b}_t \big)$, where $\hat{b}_t(a) = \left( \frac{ \sI\left\{ a_t = a \right\} }{ \pi_{\theta_t}(a) } - 1 \right) \cdot b$ for all $a \in [K]$, and $b \in (r(a^*) - \Delta, r(a^*)) $.
\end{update_rule}
\begin{theorem}
\label{thm:almost_sure_global_convergence_softmax_natural_pg_special_on_policy_stochastic_gradient_oracle_baseline}
Using \cref{update_rule:softmax_natural_pg_special_on_policy_stochastic_gradient_oracle_baseline}, $\left( \pi^* - \pi_{\theta_t} \right)^\top r \to 0$ as $t \to \infty$ with probability $1$.
\end{theorem}
As noted, while the variance of the updates provably explodes \citep{chung2020beyond},
the necessary condition in \cref{thm:committal_rate_main_theorem} is satisfied. 
Indeed, if $a_t\ne a^*$, $\pi_{\theta_{t+1}}(a_t) < \pi_{\theta_{t}}(a_t)$, 
while the optimal action's probability always increases after any update. 
Therefore, we have $\kappa(\gA, a^*) = \infty$ and $\kappa(\gA, a) = 0$ for all $a \not= a^*$. This example shows that the committal rate gives useful information regardless of whether the variance of the update stays bounded.

\section{The Geometry-Convergence Trade-off in Stochastic Policy Optimization}
\label{sec:geometry_convergence_tradeoff}

\cref{thm:committal_rate_softmax_pg} raises the question of whether 
$\kappa(\gA, a) \le 1$ for all sub-optimal actions $a \in [K]$ is sufficient to ensure an algorithm $\gA$ converges to an optimal
policy almost surely.
Unfortunately, this is not the case, and the complete picture of global optimality in stochastic policy optimization
is more complex and 
requires detailed study of different iteration behaviors.

\subsection{Iteration Behaviours}

\begin{remark}
The condition that 
$\kappa(\gA, a) \le 1$ for all sub-optimal actions $a \in [K]$ is \textbf{not} sufficient for ensuring almost sure convergence to global optimality. 
In addition to ``convergence to a sub-optimal policy with positive probability'' and ``convergence to a globally optimal policy with probability $1$'' there exist other possible optimization behaviours, such as ``not converging to any policy''.
\end{remark}
In particular, consider the following update behaviors.
\paragraph{Staying.} For the stationary update $\gA: \theta_{t+1} \gets \theta_{t}$ 
we obtain $\kappa(\gA, a) = 0 \le 1$ for all $a \in [K]$,
yet $\pi_{\theta_t} = \pi_{\theta_1}$ does not converge to the optimal policy nor any sub-optimal deterministic policy.
\paragraph{Wandering} (NPG with a large baseline)\textbf{.}
Consider $\gA: \theta_{t+1} \gets \theta_{t} + \eta \cdot \left( \hat{r}_t - \hat{b}_t \right)$ with $\hat{b}_t(a) = \left( \frac{ \sI\left\{ a_t = a \right\} }{ \pi_{\theta_t}(a) } - 1 \right) \cdot b$ for all $a \in [K]$. If $b > r(a^*)$, then we have $\pi_{\theta_{t+1}}(a_t) < \pi_{\theta_{t}}(a_t)$, i.e., a selected action's probability will decrease after updating, hence $\kappa(\gA, a) = 0$ for all $a \in [K]$. However, $\pi_{\theta_t}(a) \not\to 1$ as $t \to \infty$ for all $a \in [K]$, therefore $\pi_{\theta_t}$ will wander within the simplex forever.

The above examples show that not converging to a sub-optimal policy does not necessarily imply converging to an optimal policy almost surely, and a stronger condition is needed to eliminate unreasonable behaviors like $\theta_{t+1} \gets \theta_{t}$. We leave it as an open question to identify necessary and sufficient conditions for almost sure convergence to a global optimum.

\subsection{Geometry-Convergence Trade-off}
In
\cref{sec:turnover_results} we see that NPG and GNGP can use true gradients to significantly accelerate PG by better exploiting geometry. 
However, in the stochastic setting, any estimated geometry might be inaccurate, and intuitively, accelerated methods risk leveraging inaccurate information too aggressively. 
On the one hand, if progress is sufficiently fast (i.e., with a large committal rate), then an algorithm might never recover from aggressive yet inaccurate updates (\cref{thm:committal_rate_main_theorem}). 
On the other hand, large progress is necessary for fast convergence. 
The tension between these observations suggest that there might be an inherent trade-off between exploiting geometry and avoiding premature convergence in stochastic policy optimization. We formalize this intuition with the following results.
For the first result, we need to restrict to the class of policy optimization methods that do not decrease the probability of the optimal action whenever that action is chosen: In particular, a policy optimization method is said to be \emph{optimality-smart} if for any $t\ge 1$,
$\pi_{\tilde \theta_{t}}(a^*)\ge \pi_{\theta_t}(a^*)$ holds 
where $\tilde \theta_t$ is the parameter vector obtained when $a^*$ is chosen in every time step, starting at $\theta_1$, while $\theta_t$ is \emph{any} parameter vector that can be obtained with $t$ updates (regardless of the action sequence chosen), but also starting from $\theta_1$.
\begin{theorem}
\label{thm:committal_rate_optimal_action_special}
Let $\gA$ be optimality-smart and pick a bandit instance.
If $\gA$ together with on-policy sampling 
 leads to $\{\theta_t\}_{t\ge 1}$ such that $\{\pi_{\theta_t}\}_{t\ge 1}$ converges to a globally optimal policy at a rate $O(1/t^\alpha)$ with positive probability, for $\alpha > 0$, then $\kappa(\gA, a^*) \ge \alpha$.  
\end{theorem}
This theorem implies that a large committal rate for the optimal action is necessary for achieving fast convergence to the globally optimal policy, since the sub-optimality dominates how close the optimal action's probability is to $1$, i.e., $\left( \pi^* - \pi_{\theta_t} \right)^\top r \ge \left( 1 - \pi_{\theta_t}(a^*) \right) \cdot \Delta$. Therefore $\left( \pi^* - \pi_{\theta_t} \right)^\top r \in O(1/t^\alpha)$ implies $1 - \pi_{\theta_t}(a^*) \in O(1/t^\alpha)$. Combining this result with \cref{thm:committal_rate_main_theorem} formally establishes the following inherent trade-off between exploiting geometry to accelerate convergence versus achieving global optimality almost surely (aggressiveness vs. stability).
\begin{theorem}[Geometry-Convergence trade-off]
\label{thm:geometry_convergence_tradeoff}
If an algorithm $\gA$ is optimality-smart, and  $\kappa(\gA, a^*) = \kappa(\gA, a)$ for at least one $a \not= a^*$, then $\gA$ with on-policy sampling can only exhibit at most one of the following two behaviors: 
\textbf{(i)} $\gA$ converges to a globally optimal policy almost surely; 
\textbf{(ii)} $\gA$ converges to a deterministic policy at a rate faster than $O(1/t)$ with positive probability.
\end{theorem}
In other words, if $\gA$ has a chance to converge to a global optimum, then either $\gA$ converges to the globally optimal policy with probability $1$ ($\gA$ is stable) but at a rate no better than $O(1/t)$, or it achieves a faster than $O(1/t)$ convergence rate ($\gA$ is aggressive) but fails to converge to the globally optimal policy with some positive probability. 
This trade-off between the geometry and convergence is faced by any stochastic policy optimization algorithm that is not informed by external oracle information that allows it to distinguish optimal and sub-optimal actions based on on-policy samples.
\begin{remark}
\cref{thm:geometry_convergence_tradeoff} implies that an algorithm can achieve at most one of the mentioned two results. It is possible that an algorithm achieves neither (e.g., staying or wandering).
\end{remark}

\subsection{Exploiting External Information}
\label{sec:exploiting_external_information}

In \cref{thm:geometry_convergence_tradeoff}, the condition of $\kappa(\gA, a^*) = \kappa(\gA, a)$ for at least one sub-optimal action $a \in [K]$ is necessary for the trade-off to hold. 
If this condition can somehow be bypassed,
then it is possible to simultaneously achieve faster rates and almost sure convergence to a global optimum. 
For example, consider \cref{update_rule:softmax_natural_pg_special_on_policy_stochastic_gradient_oracle_baseline}. As mentioned before, we have $\kappa(\gA, a^*) = \infty$ and $\kappa(\gA, a) = 0$ for all $a \not= a^*$, breaking the mentioned condition, 
which allows $\gA$ to enjoy almost sure global convergence as well as a 
$O(e^{-c \cdot t})$
rate.
Of course, such a fortuitous outcome required a very specific baseline that is aware of both the optimal reward and the reward gap. Without introducing external mechanisms that inform an on-policy algorithm 
it appears that such information cannot be recovered sufficiently quickly from sample data alone \cite{tucker2018baseline}. 
Nevertheless, it remains an open question to prove that this is not possible, or whether some other strategy might allow an on-policy stochastic policy optimization algorithm to avoid the condition of \cref{thm:geometry_convergence_tradeoff} and achieve both fast rates and almost sure global convergence.

\begin{figure*}[ht]
\centering
\vskip -0.05in
\includegraphics[width=0.8\linewidth]{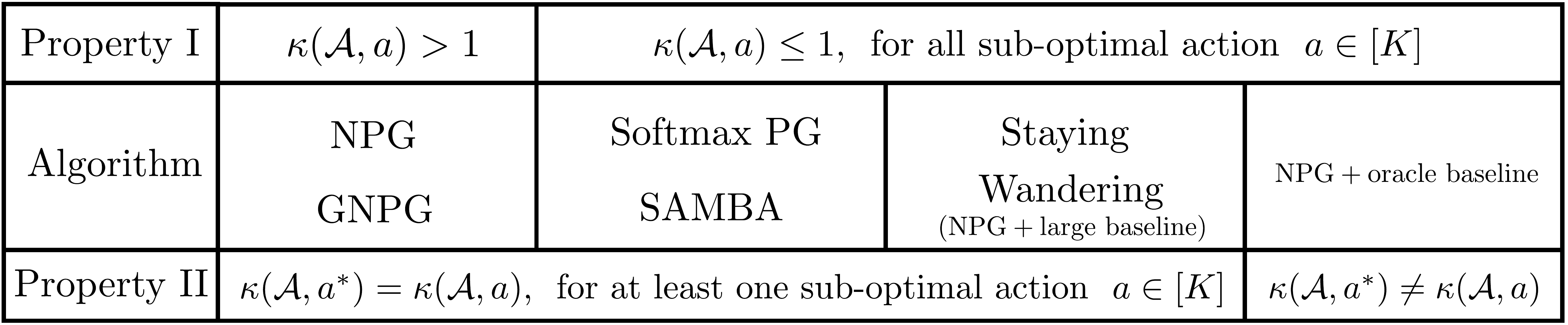}
\vskip -0.05in
\caption{
Different algorithmic behaviours subdivided by two properties of committal rate. SAMBA \citep{denisov2020regret} does not use parametric policies and is discussed in the appendix.
}
\label{fig:iteration_behaviours}
\vskip -0.05in
\end{figure*}

\cref{fig:iteration_behaviours} summarizes all the iteration behaviours we studied in this paper, organized by two properties of committal rate: \textbf{(i)} possible failure if $\kappa(\gA, a) > 1$ for at least one sub-optimal action $a$; and \textbf{(ii)} an inherent geometry-convergence trade-off if $\kappa(\gA, a^*) = \kappa(\gA, a)$ for at least one sub-optimal action $a$. It remains open to study where other algorithms suit themselves in this diagram.


\section{Initialization Sensitivity and Ensemble Methods}
\label{sec:ensemble_method}

We use the committal rate to further reveal mystery observed in practice about the initialization sensitivity \citep{henderson2018deep}. With the understanding of this unavoidable phenomenon, we introduce ensemble method and quantitatively characterize the successful rate in terms of number of trials.

\subsection{Initialization Sensitivity}

It has been observed empirically that RL algorithms are sensitive to initialization in practice: the same algorithm can produce remarkably different performance given different random seeds~\citep{henderson2018deep}. Some existing work has attempted to explain initialization sensitivity due to the softmax transform \citep{mei2020escaping}, but such results only hold for true gradients and apply to standard PG methods.

Using the committal rate theory developed above, we can provide a new explanation and additional understanding of the initialization sensitivity of practical policy optimization algorithms.  Most well-performing policy optimization algorithms in practice, such as TRPO and PPO \citep{schulman2015trust,schulman2017proximal}, are based on NPG, which exploits geometry to accelerate PG in true gradient settings. However, according to \cref{thm:geometry_convergence_tradeoff}, such fast convergence must incur a positive probability of failing to reach a global optimum, even in bandit settings. Therefore, the need to attempt multiple random seeds to achieve success is an unavoidable consequence of using these algorithms according to this theory.

\subsection{Ensemble Methods}

The committal rate theory also explains why ensemble methods \citep{wiering2008ensemble,Jung2020Population-Guided,parker2020effective}, \emph{i.e.}, running a policy optimization algorithm in multiple parallel threads and picking the best performing one, can provably work well. This is because a fast algorithm for the true gradient setting can have a positive probability of success or failure across different initializations while always converging quickly.  In which case, multiple independent runs can then be used to reduce the failure probability to any desired positive value, while retaining efficiency (if full parallelism can be maintained).
\begin{theorem}
\label{thm:ensemble_method}
With probability $1 - \delta$, the best single run among $O(\log{\left( 1/ \delta \right)})$ independent runs of NPG (GNPG) converges to a globally optimal policy at an 
$O(e^{-c \cdot t})$
rate.
\end{theorem}
It is known that softmax PG can get stuck on long plateaus for even true gradient settings \citep{mei2020escaping,li2021softmax}, which means almost sure global convergence does not necessarily imply good practical performance. Therefore, it is a reasonable choice to perform well with the compromise of small failure probability. Here we consider simply best selection, and it remains open to understand other practical training tricks, such as proximal update \citep{schulman2017proximal,lazic2021optimization} and regularization \citep{mnih2016asynchronous} under stochastic settings.

\if
Although we prove the global optimality convergence of ensemble NPG (GNPG) with simply best selection with high probability, in practice, evolution updates~\citep{khadka2018evolution}, policy distillation~\citep{pourchot2018cem} and other tricks can be exploited to further improve the performances. 
\fi

\section{Discussions}
\label{sec:discussions}

\subsection{Sufficient and Necessary Conditions for Almost Sure Global Convergence}

We make the following conjecture with some intuitions for the sufficient and necessary condition for global convergence, a question left open in \cref{sec:geometry_convergence_tradeoff}.
\begin{conjecture}
Given a stochastic policy optimization algorithm $\gA$, if $\kappa(\gA, a^*) = \kappa(\gA, a)$ for at least one sub-optimal action $a$, then $\kappa(\gA, a^*) \in (0, 1]$ is a sufficient and necessary condition for global convergence to $\pi^*$ with \textbf{polynomial convergence rate} of $O(1/t^\alpha)$, where $\alpha > 0$.
\end{conjecture}
The necessary condition is from \cref{thm:committal_rate_main_theorem}. For the sufficient condition, \cref{thm:committal_rate_optimal_action_special} can potentially be strengthened to $\kappa(\gA, a^*) \ge \alpha$ is a sufficient and necessary condition for global convergence rate $O(1/t^\alpha)$ ($\alpha > 0$). The observation here is that \cref{asmp:positive_reward} leads to $\left( \pi^* - \pi_{\theta_t} \right)^\top r \le 1 - \pi_{\theta_t}(a^*)$. This suggests that if $1 - \pi_{\theta_t}(a^*) \in O(1/t^\alpha)$, then $\left( \pi^* - \pi_{\theta_t} \right)^\top r \in O(1/t^\alpha)$. However, a gap here is $\kappa(\gA, a^*) \ge \alpha$ means ``$1 - \pi_{\theta_t}(a^*) \in O(1/t^\alpha)$ if we fix sampling $a^*$ forever'', and it is not clear if this implies ``$1 - \pi_{\theta_t}(a^*) \in O(1/t^\alpha)$ if we run the algorithm $\gA$ using on-policy sampling $a_t \sim \pi_{\theta_t}(\cdot)$''.

\subsection{Lower Bounds in Bandit Literature}

In the bandit literature \citep{lattimore2020bandit}, the $\Omega(\log{T})$ result implies that the convergence speed in terms of sub-optimality (``average regret'') cannot be faster than $O(1/t)$. However, the lower bound construction there holds for stochastic reward settings. \cref{thm:geometry_convergence_tradeoff} holds for a simpler optimization setting: the reward is fixed and deterministic, but the policy gradient is estimated by on-policy sampling. Therefore, the difficulty and trade-off are from the restriction on the action-selection scheme (balancing the aggressiveness and the stability), not from estimating or tracking the reward signal.

\subsection{General MDPs}


The one-state MDP results already show the main findings, since a large portion is about constructing counterexamples showing that the stochastic policy optimization algorithms do not perform well as in the true gradient setting. A counterexample for one-state MDPs is also a counterexample for general MDPs. Therefore, there is no loss of generality by establishing negative results using one-state MDPs. We include extensions to general finite MDPs in \cref{sec:extension_general_mdps} for completeness.

\section{Conclusion and Future Work}
\label{sec:conclusions_future_work}

This paper introduces the committal rate theory, which not only explains why faster policy optimization algorithms in the true gradient setting become dominated by slower counterparts in the on-policy stochastic setting, but also reveals an inherent geometry-convergence trade-off in stochastic policy optimization. The theory also explains empirical observations of sensitivity to random initialization for practical policy optimization algorithms as well as the effectiveness of ensemble methods.

One interesting future direction is to study necessary and sufficient conditions for almost sure global convergence, which could be weaker than the bounded variance assumption. Another important direction is to investigate whether other techniques might be used in on-policy settings to break the condition of \cref{thm:geometry_convergence_tradeoff} to achieve almost sure global convergence with a fast rate. One also expects that some generalized versions of committal rate would be meaningful in stochastic reward settings.

\begin{ack}
The authors would like to thank anonymous reviewers for their valuable comments. Jincheng Mei, Bo Dai and Dale Schuurmans would like to thank Lihong Li for helpful early discussions. Jincheng Mei and Bo Dai would like to thank Nicolas Le Roux for providing feedback on a draft of
this manuscript. Jincheng Mei would like to thank Michael Bowling for carefully checking the paper draft in a thesis chapter.
Csaba Szepesv\'ari and Dale Schuurmans gratefully acknowledge funding from the Canada CIFAR AI Chairs Program, Amii and NSERC.

\end{ack}

\bibliography{all}
\bibliographystyle{unsrtnat}

\newpage
\appendix
\newcommand{\ttheta}{\tilde\theta}

\begin{center}
\LARGE \textbf{Appendix}
\end{center}

The appendix is organized as follows.

\etocdepthtag.toc{mtappendix}
\etocsettagdepth{mtmainpaper}{none}
\etocsettagdepth{mtappendix}{subsubsection}
\begingroup
\parindent=0em
\etocsettocstyle{\rule{\linewidth}{\tocrulewidth}\vskip0.5\baselineskip}{\rule{\linewidth}{\tocrulewidth}}
\tableofcontents 
\endgroup

\section{Proofs for Algorithm Preferability (\texorpdfstring{\cref{sec:turnover_results}}{Lg}) }

\subsection{True Gradient Setting}

\subsubsection{Softmax PG}

\textbf{\cref{lem:non_uniform_lojasiewicz_softmax_special}} (Non-uniform \L{}ojasiewicz (N\L{}), \citep{mei2020global}) \textbf{.}
Let $a^*$ be the uniqe optimal action.
Denote $\pi^* = \argmax_{\pi \in \Delta}{ \pi^\top r}$. Then, 
\begin{align}
    \left\| \frac{d \pi_\theta^\top r}{d \theta} \right\|_2 \ge \pi_\theta(a^*) \cdot ( \pi^* - \pi_\theta )^\top r.
\end{align}
\begin{proof}
See the proof in \citep[Lemma 3]{mei2020global}. We include a proof for completeness.

Using the expression of the policy gradient in \cref{update_rule:softmax_pg_special_true_gradient}, we have,
\begin{align}
\label{eq:non_uniform_lojasiewicz_softmax_special_intermediate_1}
    \left\| \frac{d \pi_\theta^\top r}{d \theta} \right\|_2 &= \left[ \sum_{a \in \gA }{  \pi_\theta(a)^2 \cdot (r(a) - \pi_\theta^\top r)^2 } \right]^\frac{1}{2} \\
    &\ge  \pi_\theta(a^*) \cdot (r(a^*) - \pi_\theta^\top r). \qedhere
\end{align}
\end{proof}

\textbf{\cref{prop:upper_bound_softmax_pg_special_true_gradient}} (PG upper bound \cite{mei2020global})\textbf{.}
Using \cref{update_rule:softmax_pg_special_true_gradient} with $\eta = 2/5$, we have, for all $t \ge 1$,
\begin{align}
\label{eq:upper_bound_softmax_pg_special_true_gradient_result_1}
    ( \pi^* - \pi_{\theta_t} )^\top r \le 5 / (c^2 \cdot t),
\end{align}
such that $c = \inf_{t\ge 1} \pi_{\theta_t}(a^*) > 0$ is a constant that depends on $r$ and $\theta_1$, but it does not depend on the time $t$. 
In particular, if $\pi_{\theta_1}(a) = 1/K$, $\forall a$, then $c \ge 1/K$, i.e.,
\begin{align}
\label{eq:upper_bound_softmax_pg_special_true_gradient_result_2}
    ( \pi^* - \pi_{\theta_t} )^\top r \le 5 K^2 / t.
\end{align}
\begin{proof}
See the proof in \citep[Theorem 2]{mei2020global}. We include a proof for completeness.

\textbf{First part.} \cref{eq:upper_bound_softmax_pg_special_true_gradient_result_1}.

According to \citep[Lemma 2]{mei2020global}, for any $r \in \left[ 0, 1\right]^K$, $\theta \mapsto \pi_\theta^\top r$ is $5/2$-smooth,
\begin{align}
\label{eq:upper_bound_softmax_pg_special_true_gradient_intermediate_1}
    \left| ( \pi_{\theta^\prime} - \pi_\theta)^\top r - \Big\langle \frac{d \pi_\theta^\top r}{d \theta}, \theta^\prime - \theta \Big\rangle \right| \le \frac{5}{4} \cdot \| \theta^\prime - \theta \|_2^2.
\end{align}
Denote $\delta(\theta_t) \coloneqq ( \pi^* - \pi_{\theta_t} )^\top r$. We have, for all $t \ge 1$,
\begin{align}
\label{eq:upper_bound_softmax_pg_special_true_gradient_intermediate_2}
    \delta(\theta_{t+1}) - \delta(\theta_t) &= - \pi_{\theta_{t+1}}^\top r + \pi_{\theta_t}^\top r + \Big\langle \frac{d \pi_{\theta_t}^\top r}{d \theta_t}, \theta_{t+1} - \theta_{t} \Big\rangle - \Big\langle \frac{d \pi_{\theta_t}^\top r}{d \theta_t}, \theta_{t+1} - \theta_{t} \Big\rangle\\
    &\le \frac{5}{4} \cdot \| \theta_{t+1} - \theta_{t} \|_2^2 - \Big\langle \frac{d \pi_{\theta_t}^\top r}{d \theta_t}, \theta_{t+1} - \theta_{t} \Big\rangle \qquad \left( \text{by \cref{eq:upper_bound_softmax_pg_special_true_gradient_intermediate_1}} \right) \\
    &= - \frac{1}{5} \cdot \bigg\| \frac{d \pi_{\theta_t}^\top r}{d \theta_t} \bigg\|_2^2 
    \qquad \left( \text{using \cref{update_rule:softmax_pg_special_true_gradient} and } \eta = 2/5 \right) \\
    &\le - \frac{1}{5} \cdot \left[ \pi_{\theta_t}(a^*) \cdot ( \pi^* - \pi_{\theta_t} )^\top r \right]^2 
    \qquad \left( \text{by \cref{lem:non_uniform_lojasiewicz_softmax_special}} \right) \\
    &\le - \frac{c^2}{5} \cdot \delta(\theta_t)^2,
\end{align}
where $c = \inf_{t\ge 1} \pi_{\theta_t}(a^*) > 0$ is from \citep[Lemma 5]{mei2020global}. Then we have, for all $t \ge 1$,
\begin{align}
    \frac{1}{ \delta(\theta_t) } &= \frac{1}{\delta(\theta_1)} + \sum_{s=1}^{t-1}{ \left[ \frac{1}{\delta(\theta_{s+1})} - \frac{1}{\delta(\theta_{s})} \right] } \\
    &= \frac{1}{\delta(\theta_1)} + \sum_{s=1}^{t-1}{ \frac{1}{\delta(\theta_{s+1}) \cdot \delta(\theta_{s})} \cdot \left( \delta(\theta_{s}) - \delta(\theta_{s+1}) \right) } \\
    &\ge \frac{1}{\delta(\theta_1)} + \sum_{s=1}^{t-1}{ \frac{1}{\delta(\theta_{s+1}) \cdot \delta(\theta_{s})} \cdot \frac{c^2}{5} \cdot \delta(\theta_{s})^2 } \qquad \left( \text{by \cref{eq:upper_bound_softmax_pg_special_true_gradient_intermediate_2}} \right) \\
    &\ge \frac{1}{\delta(\theta_1)} + \frac{c^2}{5} \cdot \left( t - 1 \right) \qquad \left(0 < \delta(\theta_{t+1}) \le \delta(\theta_{t}), \text{ by \cref{eq:upper_bound_softmax_pg_special_true_gradient_intermediate_2}} \right) \\
    &\ge \frac{c^2}{5} \cdot t, \qquad \left( \delta(\theta_1) \le 1 < 5 / c^2 \right)
\end{align}
which implies \cref{eq:upper_bound_softmax_pg_special_true_gradient_result_1}.

\textbf{Second part.} \cref{eq:upper_bound_softmax_pg_special_true_gradient_result_2}.

Suppose $\pi_{\theta_1}(a) = 1/K$, $\forall a$. Using similar arguments in \citep[Proposition 2]{mei2020global}, we prove that if $\pi_{\theta_t}(a^*) \ge \pi_{\theta_t}(a)$, for all $a \not= a^*$, then $\pi_{\theta_{t+1}}(a^*) \ge \pi_{\theta_t}(a^*)$. We have, $\forall a \not= a^*$,
\begin{align}
\label{eq:upper_bound_softmax_pg_special_true_gradient_intermediate_3}
    \frac{d \pi_{\theta_t}^\top r}{d \theta_t(a^*)} &= \pi_{\theta_t}(a^*) \cdot \left( r(a^*) - \pi_{\theta_t}^\top r \right) \\
    &\ge \pi_{\theta_t}(a^*) \cdot \left( r(a) - \pi_{\theta_t}^\top r \right) \qquad  \left( r(a^*) - \pi_\theta^\top r > 0 \text{ and } r(a^*) > r(a) \right) \\
    &> \pi_{\theta_t}(a) \cdot \left( r(a) - \pi_{\theta_t}^\top r \right) \qquad  \left( \pi_{\theta_t}(a^*) \ge \pi_{\theta_t}(a), \text{ by assumption} \right) \\
    &= \frac{d \pi_{\theta_t}^\top r}{d \theta_t(a)}.
\end{align}
After one step policy gradient update, we have,
\begin{align}
    \pi_{\theta_{t+1}}(a^*) &= \frac{\exp\left\{ \theta_{t+1}(a^*) \right\}}{ \sum_{a}{ \exp\left\{ \theta_{t+1}(a) \right\}} } \\
    &= \frac{\exp\Big\{ \theta_{t}(a^*) + \eta \cdot \frac{d \pi_{\theta_t}^\top r}{d \theta_t(a^*)} \Big\}}{ \sum_{a}{ \exp\Big\{ \theta_{t}(a) + \eta \cdot \frac{d \pi_{\theta_t}^\top r}{d \theta_t(a)} \Big\}} } \\
    &\ge \frac{\exp\Big\{ \theta_{t}(a^*) + \eta \cdot \frac{d \pi_{\theta_t}^\top r}{d \theta_t(a^*)} \Big\}}{ \sum_{a}{ \exp\Big\{ \theta_{t}(a) + \eta \cdot \frac{d \pi_{\theta_t}^\top r}{d \theta_t(a^*)} \Big\}} } \qquad \left( \text{by \cref{eq:upper_bound_softmax_pg_special_true_gradient_intermediate_3}} \right) \\
    &= \frac{\exp\left\{ \theta_{t}(a^*) \right\}}{ \sum_{a}{ \exp\left\{ \theta_{t}(a) \right\}} } = \pi_{\theta_t}(a^*).
\end{align}
Note that $\pi_{\theta_1}(a^*) \ge \pi_{\theta_1}(a)$, and thus we have,
\begin{equation*}
    c = \inf_{t\ge 1} \pi_{\theta_t}(a^*) \ge \pi_{\theta_1}(a^*) = 1 / K. \qedhere
\end{equation*}
\end{proof}

\textbf{\cref{prop:lower_bound_softmax_pg_special_true_gradient}} (PG lower bound \cite{mei2020global})\textbf{.}
For sufficiently large $t \ge 1$,  \cref{update_rule:softmax_pg_special_true_gradient} with $\eta \in ( 0 , 1]$
exhibits 
\begin{align}
    (\pi^* - \pi_{\theta_t})^\top r \ge \Delta^2 / \left( 6 \cdot t \right),
\end{align}
where $\Delta = r(a^*) - \max_{a \not= a^*}{ r(a) } > 0$ is the reward gap of $r$.
\begin{proof}
See the proof in \citep[Theorem 9]{mei2020global}. We include a  proof for completeness.

According to \citep[Lemma 17]{mei2020global},
\begin{align}
\label{eq:lower_bound_softmax_pg_special_true_gradient_intermediate_1}
    \left\| \frac{d \pi_\theta^\top r}{d \theta} \right\|_2 \le \frac{\sqrt{2}}{\Delta} \cdot (\pi^* - \pi_\theta)^\top r.
\end{align}
Denote $\delta(\theta_t) \coloneqq ( \pi^* - \pi_{\theta_t} )^\top r$. We have, for all $t \ge 1$,
\begin{align}
\label{eq:lower_bound_softmax_pg_special_true_gradient_intermediate_2}
    \delta(\theta_t) - \delta(\theta_{t+1}) &= ( \pi_{\theta_{t+1}} - \pi_{\theta_t}) ^\top r - \Big\langle \frac{d \pi_{\theta_t}^\top r}{d \theta_t}, \theta_{t+1} - \theta_t \Big\rangle + \Big\langle \frac{d \pi_{\theta_t}^\top r}{d \theta_t}, \theta_{t+1} - \theta_t \Big\rangle \\
    &\le \frac{5}{4} \cdot \left\| \theta_{t+1} - \theta_t \right\|_2^2 + \Big\langle \frac{d \pi_{\theta_t}^\top r}{d \theta_t}, \theta_{t+1} - \theta_t \Big\rangle \qquad \left( \text{by \cref{eq:upper_bound_softmax_pg_special_true_gradient_intermediate_1}} \right) \\
    &\le \left( \frac{5}{4} + 1 \right) \cdot \bigg\| \frac{d \pi_{\theta_t}^\top r}{d \theta_t} \bigg\|_2^2 
    \qquad \left( \text{using \cref{update_rule:softmax_pg_special_true_gradient} and } \eta \in (0, 1] \right) \\
    &\le \frac{9}{2} \cdot \frac{1}{ \Delta^2} \cdot \delta(\theta_t)^2. 
    \qquad \left( \text{by \cref{eq:lower_bound_softmax_pg_special_true_gradient_intermediate_1}} \right)
\end{align}
According to \cref{prop:upper_bound_softmax_pg_special_true_gradient}, we have $\delta(\theta_t) \to 0$ as $t \to \infty$. We show that for all large enough $t \ge 1$, $\delta(\theta_t) \le \frac{10}{9} \cdot \delta(\theta_{t+1})$ by contradiction. Suppose $\delta(\theta_t) > \frac{10}{9} \cdot \delta(\theta_{t+1})$. We have,
\begin{align}
\label{eq:lower_bound_softmax_pg_special_true_gradient_intermediate_3}
    \delta(\theta_{t+1}) &\ge \delta(\theta_t) - \frac{9}{2} \cdot \frac{1}{ \Delta^2} \cdot \delta(\theta_t) \qquad \left( \text{by \cref{eq:lower_bound_softmax_pg_special_true_gradient_intermediate_2}} \right) \\
    &> \frac{10}{9} \cdot \delta(\theta_{t+1}) - \frac{9}{2} \cdot \frac{1}{ \Delta^2} \cdot \left( \frac{10}{9} \cdot \delta(\theta_{t+1}) \right)^2  \\
    &= \frac{10}{9} \cdot \delta(\theta_{t+1}) - \frac{50 }{9} \cdot \frac{1}{\Delta^2} \cdot \delta(\theta_{t+1})^2,
\end{align}
where the second inequality is because of the function $f: x \mapsto x - a \cdot x^2$ with $a > 0$ is monotonically increasing for all $0 < x \le \frac{1}{ 2 a }$. \cref{eq:lower_bound_softmax_pg_special_true_gradient_intermediate_3} implies that,
\begin{align}
    \delta(\theta_{t+1}) > \frac{\Delta^2}{50 },
\end{align}
for large enough $t \ge 1$, which is a contradiction with $\delta(\theta_t) \to 0$ as $t \to \infty$. Thus we have $\frac{\delta(\theta_{t+1})}{\delta(\theta_t)} \ge \frac{9}{10}$ holds for all large enough $t \ge 1$. Next, we have,
\begin{align}
    \frac{1}{ \delta(\theta_{t+1}) } - \frac{1}{ \delta(\theta_t) } &= \frac{1}{ \delta(\theta_{t+1}) \cdot \delta(\theta_t) } \cdot \left( \delta(\theta_t) - \delta(\theta_{t+1}) \right) \\
    &\le \frac{1}{ \delta(\theta_{t+1}) \cdot \delta(\theta_t) } \cdot \frac{9}{2} \cdot \frac{1}{ \Delta^2} \cdot \delta(\theta_t)^2 \qquad \left( \text{by \cref{eq:lower_bound_softmax_pg_special_true_gradient_intermediate_2}} \right) \\
    &\le \frac{5 }{ \Delta^2}. \qquad \left( \frac{\delta(\theta_t)}{\delta(\theta_{t+1})} \le \frac{10}{9} \right)
\end{align}
Summing up from $T_1$ (some large enough time) to $T_1 + t$, we have
\begin{align}
    \frac{1}{\delta(\theta_{T_1+t})} - \frac{1}{\delta(\theta_{T_1})} \le \frac{5}{\Delta^2} \cdot (t - 1) \le \frac{5}{\Delta^2} \cdot t.
\end{align}
Since $T_1$ is a finite time, $\delta(\theta_{T_1}) \ge 1/C$ for some constant $C > 0$. Rearranging, we have
\begin{align}
    (\pi^* - \pi_{\theta_{T_1+t}})^\top r  = \delta(\theta_{T_1+t}) \ge \frac{1}{ \frac{1}{\delta_{T_1}} + \frac{5}{\Delta^2} \cdot t} \ge \frac{1}{ C + \frac{5}{\Delta^2} \cdot t } \ge \frac{1}{ C + \frac{5 }{\Delta^2} \cdot (T_1 + t) }.
\end{align}
By abusing notation $t \coloneqq T_1 + t$ and $C \le \frac{t}{\Delta^2}$, we have
\begin{align}
    (\pi^* - \pi_{\theta_t})^\top r \ge \frac{1}{ C + \frac{5}{\Delta^2} \cdot t } \ge \frac{1}{  \frac{t}{\Delta^2} + \frac{5}{\Delta^2} \cdot t } = \frac{\Delta^2}{6 \cdot t },
\end{align}
for all large enough $t \ge 1$.
\end{proof}

\subsubsection{NPG}

\textbf{\cref{lem:natural_lojasiewicz_continuous_special}} (Natural Non-uniform \L{}ojasiewicz (N\L{}) inequality, continuous)\textbf{.}
Let $r \in (0,1)^K$. Denote $\Delta(a) \coloneqq r(a^*) - r(a)$, and $\Delta \coloneqq r(a^*) - \max_{a \not= a^*}{ r(a) }$ as the reward gap of $r$. We have, for any policy $\pi_\theta \coloneqq \softmax(\theta)$,
\begin{align}
    \Big\langle \frac{d \pi_\theta^\top r}{d \theta}, r \Big\rangle \ge \pi_\theta(a^*) \cdot \Delta \cdot \left( \pi^* - \pi_\theta \right)^\top r.
\end{align}
\begin{proof}
Without loss of generality, let $r(1) > r(2) > \cdots > r(K)$. We have,
\begin{align}
\MoveEqLeft
    \Big\langle \frac{d \pi_\theta^\top r}{d \theta}, r \Big\rangle = r^\top \left( \diagonalmatrix{(\pi_\theta)} - \pi_\theta \pi_\theta^\top \right) r \\
    &= \sum_{i=1}^{K}{ \pi_\theta(i) \cdot r(i)^2 } - \left[ \sum_{i=1}^{K}{ \pi_\theta(i) \cdot r(i) } \right]^2 \\
    &= \sum_{i=1}^{K}{ \pi_\theta(i) \cdot r(i)^2 } - \sum_{i=1}^{K}{ \pi_\theta(i)^2 \cdot r(i)^2 } - 2 \cdot \sum_{i=1}^{K-1}{ \pi_\theta(i) \cdot r(i) \cdot \sum_{j = i+1}^{K}{ \pi_\theta(j) \cdot r(j) } } \\
    &= \sum_{i=1}^{K}{ \pi_\theta(i) \cdot r(i)^2 \cdot \left[ 1 - \pi_\theta(i) \right] } - 2 \cdot \sum_{i=1}^{K-1}{ \pi_\theta(i) \cdot r(i) \cdot \sum_{j = i+1}^{K}{ \pi_\theta(j) \cdot r(j) } } \\
    &= \sum_{i=1}^{K}{ \pi_\theta(i) \cdot r(i)^2 \cdot \sum_{j \not= i}{\pi_\theta(j)}  } - 2 \cdot \sum_{i=1}^{K-1}{ \pi_\theta(i) \cdot r(i) \cdot \sum_{j = i+1}^{K}{ \pi_\theta(j) \cdot r(j) } } \\
    &= \sum_{i=1}^{K-1}{ \pi_\theta(i) \cdot \sum_{j = i+1}^{K}{\pi_\theta(j) \cdot \left[ r(i)^2 + r(j)^2 \right] }  } - 2 \cdot \sum_{i=1}^{K-1}{ \pi_\theta(i) \cdot r(i) \cdot \sum_{j = i+1}^{K}{ \pi_\theta(j) \cdot r(j) } } \\
    &= \sum_{i=1}^{K-1}{ \pi_\theta(i) \cdot \sum_{j = i+1}^{K}{\pi_\theta(j) \cdot \left[ r(i) - r(j) \right]^2 }  },
\end{align}
which can be lower bounded as,
\begin{align}
    \Big\langle \frac{d \pi_\theta^\top r}{d \theta}, r \Big\rangle &\ge \pi_\theta(1) \cdot \sum_{j = 2}^{K}{\pi_\theta(j) \cdot \left[ r(1) - r(j) \right]^2 } \qquad \left( \text{fewer terms} \right) \\
    &= \pi_\theta(a^*) \cdot \sum_{a \not= a^*}{\pi_\theta(a) \cdot \Delta(a)^2 } \qquad \left( a^* = 1 \right) \\
    &\ge \pi_\theta(a^*) \cdot \Delta \cdot \sum_{a \not= a^*}{\pi_\theta(a) \cdot \Delta(a) } \qquad \left( \Delta(a) \ge \Delta \right) \\
    &= \pi_\theta(a^*) \cdot \Delta \cdot \left( \pi^* - \pi_\theta \right)^\top r. \qedhere
\end{align}
\end{proof}

\begin{remark}
The natural N\L{} inequality of \cref{lem:natural_lojasiewicz_continuous_special} is tight. Consider $K = 2$, we have,
\begin{align}
\MoveEqLeft
    r^\top \left( \diagonalmatrix{(\pi_\theta)} - \pi_\theta \pi_\theta^\top \right) r = \pi_\theta(1) \cdot r(1)^2 + \pi_\theta(2) \cdot r(2)^2 - \left[ \pi_\theta(1) \cdot r(1) + \pi_\theta(2) \cdot r(2) \right]^2 \\
    &= \pi_\theta(1) \cdot r(1)^2 \cdot \left[ 1 - \pi_\theta(1) \right] + \pi_\theta(2) \cdot r(2)^2 \cdot \left[ 1 - \pi_\theta(2) \right] - 2 \cdot \pi_\theta(1) \cdot r(1) \cdot \pi_\theta(2) \cdot r(2) \\
    &= \pi_\theta(1) \cdot r(1)^2 \cdot \pi_\theta(2) + \pi_\theta(2) \cdot r(2)^2 \cdot \pi_\theta(1) - 2 \cdot \pi_\theta(1) \cdot r(1) \cdot \pi_\theta(2) \cdot r(2) \qquad \left( \pi_\theta(1) + \pi_\theta(2) = 1 \right) \\
    &= \pi_\theta(1) \cdot \pi_\theta(2) \cdot \left[ r(1) - r(2) \right]^2 \\
    &= \pi_\theta(a^*) \cdot \Delta \cdot \left( \pi^* - \pi_\theta \right)^\top r, \qquad \left( a^* = 1, \ \Delta = r(1) - r(2), \ \left( \pi^* - \pi_\theta \right)^\top r = \pi_\theta(2) \cdot \left[ r(1) - r(2) \right] \right)
\end{align}
which means the equality holds for the above problem.
\end{remark}

\begin{remark}
For the continuous natural PG flow: $\frac{d \theta_t}{d t} = \eta \cdot r$, and $\pi_{\theta_{t}} = \softmax(\theta_{t})$, \cref{lem:natural_lojasiewicz_continuous_special} can be used to characterize the progress at each time step. We have, for all $t \ge 1$,
\begin{align}
\MoveEqLeft
    \frac{d \left( \pi^* - \pi_{\theta_t} \right)^\top r }{d t} = - \frac{d \pi_{\theta_t}^\top r }{d t} \\
    &= - \left( \frac{d \theta_t}{ d t} \right)^\top \left( \frac{d \pi_{\theta_t}^\top r }{d \theta_t} \right) \\
    &= - \eta \cdot r^\top \left( \diagonalmatrix{(\pi_{\theta_t})} - \pi_{\theta_t} \pi_{\theta_t}^\top \right) r \qquad \left( \text{NPG flow} \right) \\
    &\le - \eta \cdot \pi_{\theta_t}(a^*) \cdot \Delta \cdot \left( \pi^* - \pi_{\theta_t} \right)^\top r, \qquad \left( \text{by \cref{lem:natural_lojasiewicz_continuous_special}} \right)
\end{align}
which means the progress at time $t$ is proportional to the sub-optimality gap $\left( \pi^* - \pi_{\theta_t} \right)^\top r$, leading to a linear convergence rate.
\end{remark}

\textbf{\cref{lem:natural_lojasiewicz_discrete_special}} (Natural N\L{} inequality, discrete)\textbf{.}
Given any policy $\pi$, define $\pi^\prime$ as
\begin{align}
    \pi^\prime(a) \coloneqq \frac{ \pi(a) \cdot e^{\eta \cdot r(a)} }{ \sum_{a^\prime}{ \pi(a^\prime) \cdot e^{\eta \cdot r(a^\prime)} } }, \quad \text{ for all } a \in [K],
\end{align}
where $\eta > 0$ is the learning rate. We have,
\begin{align}
    \left( \pi^\prime - \pi \right)^\top r &\ge \left[ 1 - \frac{1}{ \pi(a^*) \cdot \left( e^{ \eta \cdot \Delta } - 1 \right) + 1} \right] \cdot \left( \pi^* - \pi \right)^\top r.
\end{align}
\begin{proof}
Without loss of generality, let $r(1) > r(2) > \cdots > r(K)$. We have,
\begin{align}
\label{eq:natural_lojasiewicz_discrete_special_intermediate_1}
\MoveEqLeft
    \left( \pi^\prime - \pi \right)^\top r = \sum_{i=1}^{K}{ \left[ \pi^\prime(i) \cdot r(i) - \pi(i) \cdot r(i) \right] } \\
    &= \sum_{i=1}^{K}{ \left[ \frac{ \pi(i) \cdot e^{\eta \cdot r(i)} \cdot r(i) }{ \sum_{j=1}^{K}{ \pi(j) \cdot e^{\eta \cdot r(j)} } }  - \pi(i) \cdot r(i) \right] } \qquad \left( \text{by definition of } \pi^\prime \right) \\
    &= \frac{1}{ \sum_{j=1}^{K}{ \pi(j) \cdot e^{\eta \cdot r(j)} } } \cdot \left[ \sum_{i=1}^{K}{  \pi(i) \cdot e^{\eta \cdot r(i)} \cdot r(i)  } - \sum_{i=1}^{K}{ \pi(i) \cdot r(i) } \cdot \sum_{j=1}^{K}{ \pi(j) \cdot e^{\eta \cdot r(j)} } \right].
\end{align}
Next, we have,
\begin{align}
\label{eq:natural_lojasiewicz_discrete_special_intermediate_2}
\MoveEqLeft
    \sum_{i=1}^{K}{  \pi(i) \cdot e^{\eta \cdot r(i)} \cdot r(i)  } - \sum_{i=1}^{K}{ \pi(i) \cdot r(i) } \cdot \sum_{j=1}^{K}{ \pi(j) \cdot e^{\eta \cdot r(j)} } \\
    &= \sum_{i=1}^{K}{  \pi(i) \cdot e^{\eta \cdot r(i)} \cdot r(i)  } - \sum_{i=1}^{K}{ \pi(i)^2 \cdot e^{\eta \cdot r(i)} \cdot r(i)  } - \sum_{i=1}^{K}{ \pi(i) \cdot r(i) } \cdot \sum_{j \not= i}{ \pi(j) \cdot e^{\eta \cdot r(j)} } \\
    &= \sum_{i=1}^{K}{  \pi(i) \cdot e^{\eta \cdot r(i)} \cdot r(i) \cdot \left[ 1 - \pi(i) \right] } - \sum_{i=1}^{K}{ \pi(i) \cdot r(i) } \cdot \sum_{j \not= i}{ \pi(j) \cdot e^{\eta \cdot r(j)} } \\
    &= \sum_{i=1}^{K}{  \pi(i) \cdot e^{\eta \cdot r(i)} \cdot r(i)  } \cdot \sum_{j \not= i}{ \pi(j) } - \sum_{i=1}^{K}{ \pi(i) \cdot r(i) } \cdot \sum_{j \not= i}{ \pi(j) \cdot e^{\eta \cdot r(j)} } \\
    &= \sum_{i=1}^{K-1}{ \pi(i) \cdot \sum_{j = i+1}^{K}{\pi(j) \cdot \left[ e^{\eta \cdot r(i)} \cdot r(i) + e^{\eta \cdot r(j)} \cdot r(j) \right] }  } - \sum_{i=1}^{K-1}{ \pi(i) \cdot \sum_{j = i+1}^{K}{\pi(j) \cdot \left[ e^{\eta \cdot r(j)} \cdot r(i) + e^{\eta \cdot r(i)} \cdot r(j) \right] }  } \\
    &= \sum_{i=1}^{K-1}{ \pi(i) \cdot \sum_{j = i+1}^{K}{\pi(j) \cdot \left[ e^{\eta \cdot r(i) } - e^{ \eta \cdot r(j) } \right] \cdot \left[ r(i) - r(j) \right] }  },
\end{align}
which can be lower bounded as,
\begin{align}
\MoveEqLeft
    \sum_{i=1}^{K}{  \pi(i) \cdot e^{\eta \cdot r(i)} \cdot r(i)  } - \sum_{i=1}^{K}{ \pi(i) \cdot r(i) } \cdot \sum_{j=1}^{K}{ \pi(j) \cdot e^{\eta \cdot r(j)} } \\
    &\ge \pi(1) \cdot \sum_{j = 2}^{K}{\pi(j) \cdot \left[ e^{\eta \cdot r(1) } - e^{ \eta \cdot r(j) } \right] \cdot \left[ r(1) - r(j) \right] } \qquad \left( \text{fewer terms} \right) \\
    &\ge \pi(1) \cdot \sum_{j = 2}^{K}{\pi(j) \cdot \left[ e^{\eta \cdot r(1) } - e^{ \eta \cdot r(2) } \right] \cdot \left[ r(1) - r(j) \right] } \qquad \left( r(j) \le r(2), \text{ for all } j \ge 2 \right) \\
    &= \pi(1) \cdot e^{ \eta \cdot r(2) } \cdot \left( e^{ \eta \cdot \Delta } - 1 \right) \cdot \sum_{a \not= a^*}{\pi(a) \cdot \Delta(a) } \qquad \left( \Delta = r(1) - r(2) \right) \\
    &= \pi(1) \cdot e^{ \eta \cdot r(2) } \cdot \left( e^{ \eta \cdot \Delta } - 1 \right) \cdot \left( \pi^* - \pi \right)^\top r.
\end{align}
Combining \cref{eq:natural_lojasiewicz_discrete_special_intermediate_1,eq:natural_lojasiewicz_discrete_special_intermediate_2}, we have,
\begin{align}
\MoveEqLeft
    \left( \pi^\prime - \pi \right)^\top r \ge \frac{\pi(1) \cdot e^{ \eta \cdot r(2) } \cdot \left( e^{ \eta \cdot \Delta } - 1 \right) }{\pi(1) \cdot e^{ \eta \cdot r(1) } + \sum_{j=2}^{K}{ \pi(j) \cdot e^{ \eta \cdot r(j) } } } \cdot \left( \pi^* - \pi \right)^\top r \\
    &= \frac{\pi(1) \cdot \left( e^{ \eta \cdot \Delta } - 1 \right) }{\pi(1) \cdot e^{ \eta \cdot \Delta } + \sum_{j=2}^{K}{ \pi(j) \cdot e^{ \eta \cdot \left[ r(j) - r(2) \right] } } } \cdot \left( \pi^* - \pi \right)^\top r \\
    &\ge \frac{\pi(1) \cdot \left( e^{ \eta \cdot \Delta } - 1 \right) }{\pi(1) \cdot e^{ \eta \cdot \Delta } + \sum_{j=2}^{K}{ \pi(j) } } \cdot \left( \pi^* - \pi \right)^\top r \qquad \left( r(j) - r(2) \le 0, \text{ for all } j \ge 2 \right) \\
    &= \frac{\pi(1) \cdot \left( e^{ \eta \cdot \Delta } - 1 \right) }{\pi(1) \cdot e^{ \eta \cdot \Delta } + 1 - \pi(1) } \cdot \left( \pi^* - \pi \right)^\top r \\
    &= \left[ 1 - \frac{1}{ \pi(a^*) \cdot \left( e^{ \eta \cdot \Delta } - 1 \right) + 1} \right] \cdot \left( \pi^* - \pi \right)^\top r. \qquad \left( a^* = 1 \right) \qedhere
\end{align}
\end{proof}

\begin{remark}
The natural N\L{} inequality of \cref{lem:natural_lojasiewicz_discrete_special} is tight. Consider $K = 2$, we have,
\begin{align}
\MoveEqLeft
    \left( \pi^\prime - \pi \right)^\top r  = \frac{\pi(1) \cdot e^{\eta \cdot r(1)} \cdot r(1) + \pi(2) \cdot e^{\eta \cdot r(2)} \cdot r(2) }{ \pi(1) \cdot e^{\eta \cdot r(1)} + \pi(2) \cdot e^{\eta \cdot r(2)} } - \left[ \pi(1) \cdot r(1) + \pi(2) \cdot r(2) \right] \\
    &= \frac{\pi(1) \cdot \pi(2) \cdot \left[ r(1) - r(2) \right] \cdot \left[ e^{\eta \cdot r(1)} - e^{\eta \cdot r(2)} \right] }{ \pi(1) \cdot e^{\eta \cdot r(1)} + \pi(2) \cdot e^{\eta \cdot r(2)} } \\
    &= \frac{\pi(1) \cdot \left( e^{\eta \cdot \left[ r(1) - r(2) \right] } - 1 \right) }{ \pi(1) \cdot e^{\eta \cdot \left[ r(1) - r(2) \right] } + \pi(2) } \cdot \pi(2) \cdot \left[ r(1) - r(2) \right] \\
    &= \frac{\pi(a^*) \cdot \left( e^{\eta \cdot \Delta } - 1 \right) }{ \pi(a^*) \cdot e^{\eta \cdot \Delta } + 1 - \pi(a^*) } \cdot \left( \pi^* - \pi \right)^\top r, \qquad \left( a^* = 1, \ \Delta = r(1) - r(2) \right)
\end{align}
which means the equality holds for the above problem.
\end{remark}

\textbf{\cref{thm:npg_rate_discrete_special}} (NPG upper bound)\textbf{.}
Using \cref{update_rule:softmax_natural_pg_special_true_gradient} with any $\eta > 0$, i.e., $\forall t \ge 1$,
\begin{align}
    \theta_{t+1} \gets \theta_{t} + \eta \cdot r, \ \text{ and } \pi_{\theta_{t+1}} = \softmax(\theta_{t+1}),
\end{align}
where $\eta > 0$ is the learning rate. We have, for all $t \ge 1$,
\begin{align}
    \left( \pi^* - \pi_{\theta_t} \right)^\top r \le \left( \pi^* - \pi_{\theta_{1}} \right)^\top r \cdot e^{ - c \cdot (t - 1) },
\end{align}
where $c \coloneqq \log{ \left( \pi_{\theta_{1}}(a^*) \cdot \left( e^{ \eta \cdot \Delta } - 1 \right) + 1 \right) } > 0$ for any $\eta > 0$, and $\Delta = r(a^*) - \max_{a \not= a^*}{ r(a) } > 0$.
\begin{proof}
We have, for all $t \ge 1$,
\begin{align}
\MoveEqLeft
    \left( \pi^* - \pi_{\theta_{t+1}} \right)^\top r = \left( \pi^* - \pi_{\theta_{t}} \right)^\top r
    - \left( \pi_{\theta_{t+1}} - \pi_{\theta_{t}} \right)^\top r \\
    &\le \frac{1}{ \pi_{\theta_{t}}(a^*) \cdot \left( e^{ \eta \cdot \Delta } - 1 \right) + 1} \cdot \left( \pi^* - \pi_{\theta_{t}} \right)^\top r \qquad \left( \text{by \cref{lem:natural_lojasiewicz_discrete_special}} \right) \\
    &\le \frac{1}{ \pi_{\theta_{1}}(a^*) \cdot \left( e^{ \eta \cdot \Delta } - 1 \right) + 1} \cdot \left( \pi^* - \pi_{\theta_{t}} \right)^\top r \qquad \left( \text{see below} \right) \\
    &\le  \frac{ 1 }{ \left[ \pi_{\theta_{1}}(a^*) \cdot \left( e^{ \eta \cdot \Delta } - 1 \right) + 1 \right]^{t}} \cdot \left( \pi^* - \pi_{\theta_{1}} \right)^\top r \\
    &= \frac{ \left( \pi^* - \pi_{\theta_{1}} \right)^\top r }{ e^{ c \cdot t} },
\end{align}
where the second inequality is because of for all $t \ge 1$,
\begin{align}
    \pi_{\theta_{t+1}}(a^*) &= \frac{ \pi_{\theta_{t}}(a^*) \cdot e^{\eta \cdot r(a^*)} }{ \sum_{a}{ \pi_{\theta_{t}}(a) \cdot e^{\eta \cdot r(a)} } } \\
    &= \frac{ \pi_{\theta_{t}}(a^*) }{ \sum_{a}{ \pi_{\theta_{t}}(a) \cdot e^{ - \eta \cdot \Delta(a)} } } \\
    &\ge \pi_{\theta_{t}}(a^*). \qquad \left( \Delta(a) \ge 0 \right) \qedhere
\end{align}
\end{proof}

\subsubsection{GNPG}

\textbf{\cref{lem:non_uniform_smoothness_softmax_special}} (Non-uniform Smoothness (NS), \citep{mei2021leveraging})\textbf{.}
The spectral radius (largest absolute eigenvalue) of Hessian matrix $\frac{d^2 \pi_\theta^\top r}{d \theta^2}$ is upper bounded by $3 \cdot \Big\| \frac{d \pi_\theta^\top r}{d \theta} \Big\|_2$.
\begin{proof}
See the proof in \citep[Lemma 2]{mei2021leveraging}. We include a  proof for completeness.

Let  $S \coloneqq S(r,\theta)\in \R^{K\times K}$ be 
the second derivative of the value map $\theta \mapsto \pi_\theta^\top r$. Denote $H(\pi_\theta) \coloneqq \diagonalmatrix(\pi_\theta) - \pi_\theta \pi_\theta^\top$ as the Jacobian of $\theta \mapsto \softmax(\theta)$.
Now, by its definition we have
\begin{align}
    S &= \frac{d }{d \theta } \left\{ \frac{d \pi_\theta^\top r}{d \theta} \right\} \\
    &= \frac{d }{d \theta } \left\{ H(\pi_\theta) r \right\}  \\
    &= \frac{d }{d \theta } \left\{ ( \diagonalmatrix(\pi_\theta) - \pi_\theta \pi_\theta^\top) r \right\}.
\end{align}
Continuing with our calculation fix $i, j \in [K]$. Then, 
\begin{align}
    S_{(i, j)} &= \frac{d \{ \pi_\theta(i) \cdot  ( r(i) - \pi_\theta^\top r ) \} }{d \theta(j)} \\
    &= \frac{d \pi_\theta(i) }{d \theta(j)} \cdot ( r(i) - \pi_\theta^\top r ) + \pi_\theta(i) \cdot \frac{d \{ r(i) - \pi_\theta^\top r \} }{d \theta(j)} \\
    &= (\delta_{ij} \pi_\theta(j) -  \pi_\theta(i) \pi_\theta(j) ) \cdot ( r(i) - \pi_\theta^\top r ) - \pi_\theta(i) \cdot ( \pi_\theta(j) r(j) - \pi_\theta(j) \pi_\theta^\top r ) \\
    &= \delta_{ij} \pi_\theta(j) \cdot ( r(i) - \pi_\theta^\top r ) -  \pi_\theta(i) \pi_\theta(j) \cdot ( r(i) - \pi_\theta^\top r ) - \pi_\theta(i) \pi_\theta(j) \cdot ( r(j) -  \pi_\theta^\top r ),
\end{align}
where
\begin{align}
    \delta_{ij} = \begin{cases}
		1, & \text{if } i = j, \\
		0, & \text{otherwise}
	\end{cases}
\end{align}
is Kronecker's $\delta$-function. To show the bound on 
the spectral radius of $S$, pick $y \in \sR^K$. Then,
\begin{align}
\label{eq:non_uniform_smoothness_softmax_special_Hessian_spectral_radius}
    \left| y^\top S y \right| &= \left| \sum\limits_{i=1}^{K}{ \sum\limits_{j=1}^{K}{ S_{(i,j)} \cdot y(i) \cdot y(j)} } \right| \\
    &= \left| \sum_{i}{ \pi_\theta(i) ( r(i) - \pi_\theta^\top r ) y(i)^2 } - 2 \sum_{i} \pi_\theta(i) ( r(i) - \pi_\theta^\top r ) y(i) \sum_{j} \pi_\theta(j) y(j) \right| \\
    &= \left| \left( H(\pi_\theta) r \right)^\top \left( y \odot y \right) - 2 \cdot \left( H(\pi_\theta) r \right)^\top y \cdot \left( \pi_\theta^\top y \right) \right| \\
    &\le \left\| H(\pi_\theta) r \right\|_\infty \cdot \left\| y \odot y \right\|_1 + 2 \cdot \left\| H(\pi_\theta) r \right\|_2 \cdot \left\| y \right\|_2 \cdot \left\| \pi_\theta \right\|_1 \cdot \left\| y \right\|_\infty \\
    &\le 3 \cdot \left\| H(\pi_\theta) r \right\|_2 \cdot \left\| y \right\|_2^2 \\
    &= 3 \cdot \bigg\| \frac{d \pi_\theta^\top r}{d \theta} \bigg\|_2 \cdot \left\| y \right\|_2^2,
\end{align}
where $\odot$ is Hadamard (component-wise) product, and the third last inequality uses H{\" o}lder's inequality together with the triangle inequality.
\end{proof}

\textbf{\cref{prop:upper_bound_softmax_gnpg_special_true_gradient}} (GNPG upper bound \citep{mei2021leveraging})\textbf{.}
Using \cref{update_rule:softmax_gnpg_special_true_gradient} with $\eta = 1/6$, we have, for all $t \ge 1$,  
\begin{align}
    ( \pi^* - \pi_{\theta_t} )^\top r \le \left( \pi^* - \pi_{\theta_{1}}\right)^\top r \cdot e^{ - \frac{  c \cdot (t-1) }{12} },
\end{align}
where $c = \inf_{t\ge 1} \pi_{\theta_t}(a^*) > 0$ does not depend on $t$.
If $\pi_{\theta_1}(a) = 1/K$, $\forall a$, then $c \ge 1/K$.
\begin{proof}
See the proof in \citep[Theorem 2]{mei2021leveraging}. We include a  proof for completeness.

Denote $\theta_{\zeta_t} \coloneqq \theta_t + \zeta_t \cdot (\theta_{t+1} - \theta_t)$ with some $\zeta_t \in [0,1]$. According to Taylor's theorem,
\begin{align}
\label{eq:upper_bound_softmax_gnpg_special_true_gradient_intermediate_1}
\MoveEqLeft
    \left| ( \pi_{\theta_{t+1}} - \pi_{\theta_t})^\top r - \Big\langle \frac{d \pi_{\theta_t}^\top r}{d \theta_t}, \theta_{t+1} - \theta_t \Big\rangle \right| = \frac{1}{2} \cdot \left| \left( \theta_{t+1} - \theta_t \right)^\top S(r, \theta_{\zeta_t} ) \left( \theta_{t+1} - \theta_t \right) \right| \\
    &\le \frac{3 }{2} \cdot \bigg\| \frac{d \pi_{\theta_{\zeta_t}}^\top r}{d {\theta_{\zeta_t}}} \bigg\|_2 \cdot \| \theta_{t+1} - \theta_t \|_2^2 \qquad \left( \text{by \cref{lem:non_uniform_smoothness_softmax_special}} \right) \\
    &\le 3 \cdot \bigg\| \frac{d \pi_{\theta_t}^\top r}{d {\theta_t}} \bigg\|_2 \cdot \| \theta_{t+1} - \theta_t \|_2^2,
\end{align}
and the last inequality is from \citep[Lemma 3]{mei2021leveraging}. Denote $\delta(\theta_t) \coloneqq ( \pi^* - \pi_{\theta_t} )^\top r$. We have, for all $t \ge 1$,
\begin{align}
\label{eq:upper_bound_softmax_gnpg_special_true_gradient_intermediate_2}
    \delta(\theta_{t+1}) - \delta(\theta_t) &= - \pi_{\theta_{t+1}}^\top r + \pi_{\theta_t}^\top r + \Big\langle \frac{d \pi_{\theta_t}^\top r}{d \theta_t}, \theta_{t+1} - \theta_{t} \Big\rangle - \Big\langle \frac{d \pi_{\theta_t}^\top r}{d \theta_t}, \theta_{t+1} - \theta_{t} \Big\rangle \\
    &\le 3 \cdot \bigg\| \frac{d \pi_{\theta_t}^\top r}{d {\theta_t}} \bigg\|_2 \cdot \| \theta_{t+1} - \theta_t \|_2^2 - \Big\langle \frac{d \pi_{\theta_t}^\top r}{d \theta_t}, \theta_{t+1} - \theta_{t} \Big\rangle \qquad \left( \text{by \cref{eq:upper_bound_softmax_gnpg_special_true_gradient_intermediate_1}} \right) \\
    &= - \frac{1}{12} \cdot \bigg\| \frac{d \pi_{\theta_t}^\top r}{d \theta_t} \bigg\|_2 
    \qquad \left( \text{using \cref{update_rule:softmax_gnpg_special_true_gradient} and } \eta = 1/6 \right) \\
    &\le - \frac{1}{12} \cdot \pi_{\theta_t}(a^*) \cdot ( \pi^* - \pi_{\theta_t} )^\top r
    \qquad \left( \text{by \cref{lem:non_uniform_lojasiewicz_softmax_special}} \right) \\
    &\le - \frac{c}{12} \cdot \delta(\theta_t),
\end{align}
where $c = \inf_{t\ge 1} \pi_{\theta_t}(a^*) > 0$ is from \citep[Lemma 4]{mei2021leveraging}. Then we have, for all $t \ge 1$,
\begin{align}
    \delta(\theta_t) &\le \delta(\theta_{t-1}) \cdot \left( 1 - \frac{  c  }{12} \right) \\
    &\le \delta(\theta_{t-1}) \cdot e^{ - \frac{c}{12} }  \\
    &\le \delta(\theta_1) \cdot e^{ - \frac{c \cdot (t - 1)}{12} }. 
\end{align}
If $\pi_{\theta_1}(a) = 1/K$, $\forall a$, then similar arguments to \cref{eq:upper_bound_softmax_pg_special_true_gradient_intermediate_3} give $c \ge 1/K$.
\end{proof}

\subsection{On-policy Stochastic Gradient Setting}

\subsubsection{Softmax PG}

\textbf{\cref{lem:unbiased_bounded_variance_softmax_pg_special_on_policy_stochastic_gradient}.}
Let $\hat{r}$ be the IS estimator using on-policy sampling $a \sim \pi_{\theta}(\cdot)$. The stochastic softmax PG estimator is unbiased and bounded, i.e.,
\begin{align}
    \expectation_{a \sim \pi_\theta(\cdot)}{ \left[ \frac{d \pi_\theta^\top \hat{r} }{d \theta} \right] } &= \frac{d \pi_\theta^\top r}{d \theta}, \text{ and} \\
    \expectation_{a \sim \pi_\theta(\cdot)}{ \left\| \frac{d \pi_\theta^\top \hat{r} }{d \theta} \right\|_2^2 } &\le 2.
\end{align}
\begin{proof}
\textbf{First part.} $\expectation_{a \sim \pi_\theta(\cdot)}{ \left[ \frac{d \pi_\theta^\top \hat{r} }{d \theta} \right] } = \frac{d \pi_\theta^\top r}{d \theta}$.

We have, for all $i \in [K]$, the true softmax PG is,
\begin{align}
    \frac{d \pi_{\theta}^\top r}{d \theta(i)} = \pi_{\theta}(i) \cdot \left( r(i) - \pi_{\theta}^\top r \right).
\end{align}
On the other hand, we have, for all $i \in [K]$,
\begin{align}
    \frac{d \pi_{\theta}^\top \hat{r}}{d \theta(i)} &= \pi_{\theta}(i) \cdot \left( \hat{r}(i) - \pi_{\theta}^\top \hat{r} \right) \\
    &= \pi_{\theta}(i) \cdot \left( \frac{ \sI\left\{ a = i \right\} }{ \pi_{\theta}(i) } \cdot r(i) - \sum_{j}{ \sI\left\{ a = j \right\} \cdot r(j) } \right) \qquad \left( \text{by \cref{def:on_policy_importance_sampling}} \right) \\
    &= \sI\left\{ a = i \right\} \cdot r(i) - \pi_{\theta}(i) \cdot r(a).
\end{align}
The expectation of stochastic softmax PG is,
\begin{align}
    \expectation_{a \sim \pi_{\theta}(\cdot) }{ \left[ \frac{d \pi_{\theta}^\top \hat{r}}{d \theta(i)} \right] } &= \sum_{a \in [K]}{ \pi_\theta(a) \cdot \left( \sI\left\{ a = i \right\} \cdot r(i) - \pi_{\theta}(i) \cdot r(a) \right) } \\
    &= \pi_{\theta}(i) \cdot r(i) - \pi_{\theta}(i) \cdot \pi_{\theta}^\top r \\
    &= \frac{d \pi_{\theta}^\top r}{d \theta(i)}.
\end{align}

\textbf{Second part.} $\expectation_{a \sim \pi_\theta(\cdot)}{ \left\| \frac{d \pi_\theta^\top \hat{r} }{d \theta} \right\|_2^2 } \le 2$.

The squared stochastic PG norm is,
\begin{align}
\label{eq:unbiased_bounded_variance_softmax_pg_special_on_policy_stochastic_gradient_intermediate_1}
\MoveEqLeft
    \left\| \frac{d \pi_{\theta}^\top \hat{r}}{d \theta} \right\|_2^2 = \sum_{i=1}^{K}{ \left( \frac{d \pi_{\theta}^\top \hat{r}}{d \theta(i)} \right)^2 } = \sum_{i=1}^{K}{ \pi_{\theta}(i)^2 \cdot \left( \hat{r}(i) - \pi_{\theta}^\top \hat{r} \right)^2 } \\
    &= \sum_{i=1}^{K}{ \pi_{\theta}(i)^2 \cdot \Bigg[ \frac{ \left( \sI\left\{ a = i \right\} \right)^2 }{ \pi_{\theta}(i)^2 } \cdot r(i)^2 - 2 \cdot \frac{ \sI\left\{ a = i \right\} }{ \pi_{\theta}(i) } \cdot r(i) \cdot \sum_{j=1}^{K}{ \sI\left\{ a = j \right\} \cdot r(j) } }  + \bigg( \sum_{j=1}^{K}{ \sI\left\{ a = j \right\} \cdot r(j) } \bigg)^2 \Bigg] \\
    &= r(a)^2 - 2 \cdot \pi_\theta(a) \cdot r(a)^2 + \sum_{i=1}^{K}{ \pi_{\theta}(i)^2 \cdot r(a)^2 } \\
    &= \left( 1 - \pi_\theta(a) \right) \cdot r(a)^2 - \pi_\theta(a) \cdot r(a)^2 + \pi_\theta(a)^2 \cdot r(a)^2 + \sum_{a^\prime \not= a}{ \pi_{\theta}(a^\prime)^2 \cdot r(a)^2 } \\
    &= \left( 1 - \pi_\theta(a) \right)^2 \cdot r(a)^2 + \sum_{a^\prime \not= a}{ \pi_{\theta}(a^\prime)^2 \cdot r(a)^2 }.
\end{align}
Taking expectation of $a \sim \pi_\theta(\cdot)$, the expected squared stochastic PG norm is,
\begin{align}
\MoveEqLeft
    \expectation_{a \sim \pi_{\theta}(\cdot)}{ \left\| \frac{d \pi_{\theta}^\top \hat{r}}{d \theta} \right\|_2^2 } = \sum_{a \in [K]}{ \pi_{\theta}(a) \cdot \left( 1 - \pi_\theta(a) \right)^2 \cdot r(a)^2 } + \sum_{a \in [K]}{ \pi_{\theta}(a) \cdot \sum_{a^\prime \not= a}{ \pi_{\theta}(a^\prime)^2 \cdot r(a)^2 } } \\
    &\le \sum_{a \in [K]}{ \pi_{\theta}(a) \cdot \left( 1 - \pi_\theta(a) \right)^2 \cdot r(a)^2 } + \sum_{a \in [K]}{ \pi_{\theta}(a) \cdot \left[ \sum_{a^\prime \not= a}{ \pi_{\theta}(a^\prime) } \right]^2  \cdot r(a)^2  } \qquad \left( \left\| x \right\|_2 \le \left\| x \right\|_1 \right) \\
    &= 2 \cdot \sum_{a \in [K]}{ \pi_{\theta}(a) \cdot \left( 1 - \pi_{\theta}(a) \right)^2 \cdot r(a)^2  } \\
    &\le 2 \cdot \sum_{a \in [K]}{ \pi_{\theta}(a)} \qquad \left( r \in (0, 1]^K, \text{ and } \pi_{\theta}(a) \in (0, 1) \text{ for all } a \in [K] \right) \\
    &= 2. \qedhere
\end{align}
\end{proof}

\begin{lemma}[Non-uniform Smoothness (NS) between two iterations]
\label{lem:non_uniform_smoothness_special_two_iterations}
Let $\theta^\prime = \theta + \eta \cdot \frac{d \pi_{\theta}^\top \hat{r}}{d {\theta}}$ (using stochastic PG update). We have, for $\eta = \frac{1}{12} \cdot \Big\| \frac{d \pi_{\theta}^\top r}{d {\theta}} \Big\|_2$ (using true PG norm in learning rate),
\begin{align}
    \left| ( \pi_{\theta^\prime} - \pi_\theta)^\top r - \Big\langle \frac{d \pi_\theta^\top r}{d \theta}, \theta^\prime - \theta \Big\rangle \right| \le 3 \cdot \left\| \frac{d \pi_{\theta}^\top r}{d {\theta}} \right\|_2 \cdot \| \theta^\prime - \theta \|_2^2.
\end{align}
\end{lemma}
\begin{proof}
Denote $\theta_\zeta \coloneqq \theta + \zeta \cdot (\theta^\prime - \theta)$ with some $\zeta \in [0,1]$. According to Taylor's theorem, $\forall \theta, \ \theta^\prime$,
\begin{align}
\label{eq:non_uniform_smoothness_special_two_iterations_intermediate_1}
    \left| ( \pi_{\theta^\prime} - \pi_\theta)^\top r - \Big\langle \frac{d \pi_\theta^\top r}{d \theta}, \theta^\prime - \theta \Big\rangle \right| &= \frac{1}{2} \cdot \left| \left( \theta^\prime - \theta \right)^\top \frac{d^2 \pi_{\theta_\zeta}^\top r}{d {\theta_\zeta}^2} \left( \theta^\prime - \theta \right) \right| \\
    &\le \frac{3 }{2} \cdot \bigg\| \frac{d \pi_{\theta_\zeta}^\top r}{d {\theta_\zeta}} \bigg\|_2 \cdot \| \theta^\prime - \theta \|_2^2. \qquad \left( \text{by \cref{lem:non_uniform_smoothness_softmax_special}} \right)
\end{align}
Denote $\zeta_1 \coloneqq \zeta$. Also denote $\theta_{\zeta_2} \coloneqq \theta + \zeta_2 \cdot (\theta_{\zeta_1} - \theta)$ with some $\zeta_2 \in [0,1]$. We have,
\begin{align}
\label{eq:non_uniform_smoothness_special_two_iterations_intermediate_2}
\MoveEqLeft
    \left\| \frac{d \pi_{\theta_{\zeta_1}}^\top r}{d {\theta_{\zeta_1}}} - \frac{d \pi_{\theta}^\top r}{d {\theta}} \right\|_2 = \left\| \int_{0}^{1} \bigg\langle \frac{d^2 \{ \pi_{\theta_{\zeta_2}}^\top r \} }{d {\theta_{\zeta_2}^2 }}, \theta_{\zeta_1} - \theta \bigg\rangle d \zeta_2 \right\|_2 \\
    &\le \int_{0}^{1} \left\| \frac{d^2 \{ \pi_{\theta_{\zeta_2}}^\top r \} }{d {\theta_{\zeta_2}^2 }} \right\|_2 \cdot \left\| \theta_{\zeta_1} - \theta \right\|_2 d \zeta_2 \qquad \left( \text{by Cauchy–Schwarz} \right) \\
    &\le \int_{0}^{1} 3 \cdot \left\| \frac{d  \pi_{\theta_{\zeta_2}}^\top r }{d {\theta_{\zeta_2} }} \right\|_2 \cdot \zeta_1 \cdot \left\| \theta^\prime - \theta \right\|_2 d \zeta_2 \qquad \left( \text{by \cref{lem:non_uniform_smoothness_softmax_special}} \right)  \\
    &\le \int_{0}^{1} 3 \cdot \left\| \frac{d  \pi_{\theta_{\zeta_2}}^\top r }{d {\theta_{\zeta_2} }} \right\|_2 \cdot \eta \cdot \left\| \frac{d \pi_{\theta}^\top \hat{r}}{d {\theta}} \right\|_2 \ d \zeta_2, \qquad \left( \zeta_1 \in [0, 1], \text{ using } \theta^\prime = \theta + \eta \cdot \frac{d \pi_{\theta}^\top \hat{r}}{d {\theta}} \right)
\end{align}
where the second inequality is because of the Hessian is symmetric, and its operator norm is equal to its spectral radius. Therefore we have,
\begin{align}
\label{eq:non_uniform_smoothness_special_two_iterations_intermediate_3}
    \left\| \frac{d \pi_{\theta_{\zeta_1}}^\top r}{d {\theta_{\zeta_1}}} \right\|_2 &\le \left\| \frac{d \pi_{\theta}^\top r}{d {\theta}} \right\|_2 + \left\| \frac{d \pi_{\theta_{\zeta_1}}^\top r}{d {\theta_{\zeta_1}}} - \frac{d \pi_{\theta}^\top r}{d {\theta}} \right\|_2 \qquad \left( \text{by triangle inequality} \right) \\
    &\le \left\| \frac{d \pi_{\theta}^\top r}{d {\theta}} \right\|_2 + 3 \eta \cdot \left\| \frac{d \pi_{\theta}^\top \hat{r}}{d {\theta}} \right\|_2 \cdot \int_{0}^{1} \left\| \frac{d  \pi_{\theta_{\zeta_2}}^\top r }{d {\theta_{\zeta_2} }} \right\|_2 d \zeta_2. \qquad \left( \text{by \cref{eq:non_uniform_smoothness_special_two_iterations_intermediate_2}} \right)
\end{align}
Denote $\theta_{\zeta_3} \coloneqq \theta + \zeta_3 \cdot (\theta_{\zeta_2} - \theta)$ with $\zeta_3 \in [0,1]$. Using similar calculation in \cref{eq:non_uniform_smoothness_special_two_iterations_intermediate_2}, we have,
\begin{align}
\label{eq:non_uniform_smoothness_special_two_iterations_intermediate_4}
    \left\| \frac{d  \pi_{\theta_{\zeta_2}}^\top r }{d {\theta_{\zeta_2} }} \right\|_2 &\le \left\| \frac{d \pi_{\theta}^\top r}{d {\theta}} \right\|_2 + \left\| \frac{d \pi_{\theta_{\zeta_2}}^\top r}{d {\theta_{\zeta_2}}} - \frac{d \pi_{\theta}^\top r}{d {\theta}} \right\|_2 \\
    &\le \left\| \frac{d \pi_{\theta}^\top r}{d {\theta}} \right\|_2 + 3 \eta \cdot \left\| \frac{d \pi_{\theta}^\top \hat{r}}{d {\theta}} \right\|_2 \cdot \int_{0}^{1}  \left\| \frac{d  \pi_{\theta_{\zeta_3}}^\top r }{d {\theta_{\zeta_3} }} \right\|_2 d \zeta_3.
\end{align}
Combining \cref{eq:non_uniform_smoothness_special_two_iterations_intermediate_3,eq:non_uniform_smoothness_special_two_iterations_intermediate_4}, we have,
\begin{align}
\label{eq:non_uniform_smoothness_special_two_iterations_intermediate_5}
    \left\| \frac{d \pi_{\theta_{\zeta_1}}^\top r}{d {\theta_{\zeta_1}}} \right\|_2 \le \left( 1 + 3 \eta \cdot \left\| \frac{d \pi_{\theta}^\top \hat{r}}{d {\theta}} \right\|_2 \right) \cdot \left\| \frac{d \pi_{\theta}^\top r}{d {\theta}} \right\|_2 + \left( 3 \eta \cdot \left\| \frac{d \pi_{\theta}^\top \hat{r}}{d {\theta}} \right\|_2 \right)^2 \cdot \int_{0}^{1} \int_{0}^{1} \left\| \frac{d  \pi_{\theta_{\zeta_3}}^\top r }{d {\theta_{\zeta_3} }} \right\|_2 d \zeta_3 d \zeta_2,
\end{align}
which implies,
\begin{align}
\label{eq:non_uniform_smoothness_special_two_iterations_intermediate_6}
\MoveEqLeft
    \left\| \frac{d \pi_{\theta_{\zeta_1}}^\top r}{d {\theta_{\zeta_1}}} \right\|_2 \le \sum_{i = 0}^{\infty}{ \left( 3 \eta \cdot \left\| \frac{d \pi_{\theta}^\top \hat{r}}{d {\theta}} \right\|_2 \right)^i } \cdot \left\| \frac{d \pi_{\theta}^\top r}{d {\theta}} \right\|_2 \\
    &= \frac{1}{1 - 3 \eta \cdot \Big\| \frac{d \pi_{\theta}^\top \hat{r}}{d {\theta}} \Big\|_2} \cdot \left\| \frac{d \pi_{\theta}^\top r}{d {\theta}} \right\|_2 \qquad \left( 3 \eta \cdot \left\| \frac{d \pi_{\theta}^\top \hat{r}}{d {\theta}} \right\|_2 \in ( 0 , 1), \text{ see below} \right) \\
    &= \frac{1}{1 - \frac{1}{4} \cdot \Big\| \frac{d \pi_{\theta}^\top \hat{r}}{d {\theta}} \Big\|_2 \cdot \Big\| \frac{d \pi_{\theta}^\top r}{d {\theta}} \Big\|_2 } \cdot \left\| \frac{d \pi_{\theta}^\top r}{d {\theta}} \right\|_2 \qquad \left( \eta = \frac{1}{12} \cdot \left\| \frac{d \pi_{\theta}^\top r}{d {\theta}} \right\|_2 \right) \\
    &\le \frac{1}{1 - \frac{1}{4} \cdot \Big\| \frac{d \pi_{\theta}^\top \hat{r}}{d {\theta}} \Big\|_2} \cdot \left\| \frac{d \pi_{\theta}^\top r}{d {\theta}} \right\|_2 \qquad \left( \left\| \frac{d \pi_{\theta}^\top r}{d {\theta}} \right\|_2 \le 1, \text{ see below} \right) \\
    &\le \frac{1}{1 - \frac{1}{2}} \cdot \left\| \frac{d \pi_{\theta}^\top r}{d {\theta}} \right\|_2 \qquad \left( \left\| \frac{d \pi_{\theta}^\top \hat{r}}{d {\theta}} \right\|_2 \le 2, \text{ see below} \right) \\
    &= 2 \cdot \left\| \frac{d \pi_{\theta}^\top r}{d {\theta}} \right\|_2,
\end{align}
where the last inequality is from,
\begin{align}
\label{eq:non_uniform_smoothness_special_two_iterations_intermediate_7}
    \left\| \frac{d \pi_{\theta}^\top \hat{r}}{d \theta} \right\|_2^2 &= \left( 1 - \pi_\theta(a) \right)^2 \cdot r(a)^2 + \sum_{a^\prime \not= a}{ \pi_{\theta}(a^\prime)^2 \cdot r(a)^2 } \qquad \left( \text{by \cref{eq:unbiased_bounded_variance_softmax_pg_special_on_policy_stochastic_gradient_intermediate_1}} \right) \\
    &\le 2 \cdot \left( 1 - \pi_\theta(a) \right)^2 \cdot r(a)^2 \qquad \left( \left\| x \right\|_2 \le \left\| x \right\|_1 \right) \\
    &\le 2, \qquad \left( r \in (0, 1]^K, \text{ and } \pi_{\theta}(a) \in (0, 1) \text{ for all } a \in [K] \right)  
\end{align}
which implies that,
\begin{align}
    \left\| \frac{d \pi_{\theta}^\top \hat{r}}{d \theta} \right\|_2 &\le \sqrt{2} \le 2,
\end{align}
and the second last inequality is from,
\begin{align}
    \left\| \frac{d \pi_{\theta}^\top r}{d {\theta}} \right\|_2^2 &= \sum_{i=1}^{K}{ \pi_{\theta}(i)^2 \cdot \left( r(i) - \pi_{\theta}^\top r \right)^2 } \\
    &\le \sum_{i=1}^{K}{ \pi_{\theta}(i)^2} \qquad \left( r \in (0, 1]^K \right) \\
    &\le \left[ \sum_{i=1}^{K}{ \pi_{\theta}(i)} \right]^2 \qquad \left( \left\| x \right\|_2 \le \left\| x \right\|_1 \right) \\
    &= 1.
\end{align}

Combining \cref{eq:non_uniform_smoothness_special_two_iterations_intermediate_1,eq:non_uniform_smoothness_special_two_iterations_intermediate_6} finishes the proof.
\end{proof}

\textbf{\cref{thm:almost_sure_global_convergence_softmax_pg_special_on_policy_stochastic_gradient}.}
Using \cref{update_rule:softmax_pg_special_on_policy_stochastic_gradient} with learning rate 
\begin{align}
    \eta = \frac{1}{12} \cdot \bigg\| \frac{d \pi_{\theta_t}^\top r}{d {\theta_t}} \bigg\|_2,
\end{align}
for all $t \ge 1$, we have $\left( \pi^* - \pi_{\theta_t} \right)^\top r \to 0$ as $t \to \infty$ in probability.
\begin{proof}
See \citep[Proposition 4]{chung2020beyond}. We provide a different proof using the non-uniform smoothness.

Denote $\delta(\theta_t) \coloneqq ( \pi^* - \pi_{\theta_t} )^\top r$.  We have, for all $t \ge 1$,
\begin{align}
\label{eq:almost_sure_global_convergence_softmax_pg_special_on_policy_stochastic_gradient_intermediate_1}
\MoveEqLeft
    \delta(\theta_{t+1}) - \delta(\theta_t) = - \pi_{\theta_{t+1}}^\top r + \pi_{\theta_t}^\top r + \Big\langle \frac{d \pi_{\theta_t}^\top r}{d \theta_t}, \theta_{t+1} - \theta_{t} \Big\rangle - \Big\langle \frac{d \pi_{\theta_t}^\top r}{d \theta_t}, \theta_{t+1} - \theta_{t} \Big\rangle\\
    &\le  3 \cdot \bigg\| \frac{d \pi_{\theta_t}^\top r}{d {\theta_t}} \bigg\|_2 \cdot \| \theta_{t+1} - \theta_{t} \|_2^2 - \Big\langle \frac{d \pi_{\theta_t}^\top r}{d \theta_t}, \theta_{t+1} - \theta_{t} \Big\rangle \qquad \left( \text{by \cref{lem:non_uniform_smoothness_special_two_iterations}} \right) \\
    &= 3 \cdot \eta^2 \cdot \bigg\| \frac{d \pi_{\theta_t}^\top r}{d {\theta_t}} \bigg\|_2 \cdot \bigg\| \frac{d \pi_{\theta_t}^\top \hat{r}_t}{d \theta_t} \bigg\|_2^2 - \eta \cdot \Big\langle \frac{d \pi_{\theta_t}^\top r}{d \theta_t}, \frac{d \pi_{\theta_t}^\top \hat{r}_t}{d \theta_t} \Big\rangle. \qquad \left( \text{using \cref{update_rule:softmax_pg_special_on_policy_stochastic_gradient}} \right)
\end{align}
Next, taking expectation over the random sampling on \cref{eq:almost_sure_global_convergence_softmax_pg_special_on_policy_stochastic_gradient_intermediate_1}, we have, 
\begin{align}
\label{eq:almost_sure_global_convergence_softmax_pg_special_on_policy_stochastic_gradient_intermediate_2}
\MoveEqLeft
    \expectation{ \left[ \delta(\theta_{t+1}) \right] } - \expectation{ \left[ \delta(\theta_{t}) \right] } \le 3 \cdot \eta^2 \cdot \bigg\| \frac{d \pi_{\theta_t}^\top r}{d {\theta_t}} \bigg\|_2 \cdot \expectation{ \left[ \bigg\| \frac{d \pi_{\theta_t}^\top \hat{r}_t}{d \theta_t} \bigg\|_2^2 \right] } - \eta \cdot \Big\langle \frac{d \pi_{\theta_t}^\top r}{d \theta_t}, \expectation{ \left[ \frac{d \pi_{\theta_t}^\top \hat{r}_t}{d \theta_t} \right] } \Big\rangle \\
    &= 3 \cdot \eta^2 \cdot \bigg\| \frac{d \pi_{\theta_t}^\top r}{d {\theta_t}} \bigg\|_2 \cdot \expectation{ \left[ \bigg\| \frac{d \pi_{\theta_t}^\top \hat{r}_t}{d \theta_t} \bigg\|_2^2 \right] } - \eta \cdot \bigg\| \frac{d \pi_{\theta_t}^\top r}{d \theta_t} \bigg\|_2^2 \qquad \left( \text{unbiased PG, by \cref{lem:unbiased_bounded_variance_softmax_pg_special_on_policy_stochastic_gradient}} \right) \\
    &\le 3 \cdot 2 \cdot \eta^2 \cdot \bigg\| \frac{d \pi_{\theta_t}^\top r}{d {\theta_t}} \bigg\|_2 - \eta \cdot \bigg\| \frac{d \pi_{\theta_t}^\top r}{d \theta_t} \bigg\|_2^2 \qquad \left( \expectation{ \left[ \bigg\| \frac{d \pi_{\theta_t}^\top \hat{r}_t}{d \theta_t} \bigg\|_2^2 \right] } \le 2, \text{ by \cref{lem:unbiased_bounded_variance_softmax_pg_special_on_policy_stochastic_gradient}} \right) \\
    &= - \frac{1}{ 24} \cdot \bigg\| \frac{d \pi_{\theta_t}^\top r}{d \theta_t} \bigg\|_2^3 \qquad \left( \text{using } \eta = \frac{1}{12} \cdot \bigg\| \frac{d \pi_{\theta_t}^\top r}{d {\theta_t}} \bigg\|_2 \right) \\
    &\le - \frac{1}{ 24} \cdot \expectation{ \left[ \pi_{\theta_t}(a^*)^3 \right] } \cdot \expectation{ \left[ \delta(\theta_{t})^3 \right] } 
    \qquad \left( \text{by \cref{lem:non_uniform_lojasiewicz_softmax_special}} \right) \\
    &\le - \frac{c}{ 24 } \cdot \left( \expectation{ \left[ \delta(\theta_{t}) \right] } \right)^3, \qquad \left( \text{by Jensen's inequality}\right)
\end{align}
where
\begin{align}
\label{eq:almost_sure_global_convergence_softmax_pg_special_on_policy_stochastic_gradient_intermediate_3}
    c &\coloneqq \inf_{t\ge 1} \expectation{ \left[ \pi_{\theta_t}(a^*)^3 \right] } \\
    &\ge \inf_{t\ge 1} \left( \expectation{ \left[ \pi_{\theta_t}(a^*) \right] } \right)^3 \qquad \left( \text{by Jensen's inequality}\right) \\
    &> 0,
\end{align}
and the last inequality is from \citep[Lemma 5]{mei2020global}, since the expected iteration equals the true gradient update, which converges to global optimal policy. Denote $\tilde{\delta}(\theta_t) \coloneqq \expectation{ \left[ \delta(\theta_{t}) \right] } $. We have, for all $t \ge 1$,
\begin{align}
\label{eq:almost_sure_global_convergence_softmax_pg_special_on_policy_stochastic_gradient_intermediate_4}
\MoveEqLeft
    \frac{1}{ \tilde{\delta}(\theta_t)^{2} } = \frac{1}{\tilde{\delta}(\theta_1)^{2}} + \sum_{s=1}^{t-1}{ \left[ \frac{1}{\tilde{\delta}(\theta_{s+1})^{2}} - \frac{1}{\tilde{\delta}(\theta_{s})^{2}} \right] } \\
    &= \frac{1}{\tilde{\delta}(\theta_1)^{2}} + \sum_{s=1}^{t-1}{ \frac{1}{\tilde{\delta}(\theta_{s+1})^{2} } \cdot \left[ 1 - \frac{\tilde{\delta}(\theta_{s+1})^{2}}{ \tilde{\delta}(\theta_{s})^{2} } \right] } \\
    &\ge \frac{1}{\tilde{\delta}(\theta_1)^{2}} + \sum_{s=1}^{t-1}{ \frac{2}{ \bcancel{ \tilde{\delta}(\theta_{s+1})^{2}} } \cdot \frac{ \bcancel{ \tilde{\delta}(\theta_{s+1})^{2} } }{ \tilde{\delta}(\theta_{s})^{2} } \cdot \left[ 1 - \frac{\tilde{\delta}(\theta_{s+1})}{ \tilde{\delta}(\theta_{s}) } \right] } \qquad \left( 1 - x^2 \ge 2 \cdot x^2 \cdot (1-x) \text{ for all } x \in (0, 1] \right) \\
    &= \frac{1}{\tilde{\delta}(\theta_1)^{2}} + 2 \cdot \sum_{s=1}^{t-1}{ \frac{1}{ \tilde{\delta}(\theta_{s})^{3} } \cdot \left( \tilde{\delta}(\theta_{s}) - \tilde{\delta}(\theta_{s+1}) \right) } \\
    &\ge \frac{1}{\tilde{\delta}(\theta_1)^{2}} + 2 \cdot \sum_{s=1}^{t-1}{ \frac{1}{ \bcancel{ \tilde{\delta}(\theta_{s})^{3} } } \cdot \frac{c}{24} \cdot \bcancel{\tilde{\delta}(\theta_{s})^{3}} } \qquad \left( \text{by \cref{eq:almost_sure_global_convergence_softmax_pg_special_on_policy_stochastic_gradient_intermediate_2}} \right) \\
    &= \frac{1}{\tilde{\delta}(\theta_1)^{2}} + \frac{c}{ 12} \cdot \left( t - 1 \right) \\
    &\ge \frac{c \cdot t}{ 12}, \qquad \left( \tilde{\delta}(\theta_1)^2 \le 1 < \frac{12}{c} \right)
\end{align}
which implies that,
\begin{align}
    \expectation_{a_t \sim \pi_{\theta_t}(\cdot)}{ \left[ ( \pi^* - \pi_{\theta_t} )^\top r \right] } \le \frac{ \sqrt{12} }{ \sqrt{c}}\cdot \frac{1}{ \sqrt{t} },
\end{align}
where $c$ is from \cref{eq:almost_sure_global_convergence_softmax_pg_special_on_policy_stochastic_gradient_intermediate_3}. This implies $\left( \pi^* - \pi_{\theta_t} \right)^\top r \to 0$ as $t \to \infty$ in probability, i.e.,
\begin{align}
    \lim_{t \to \infty}{ \probability{ \left( \left( \pi^* - \pi_{\theta_t} \right)^\top r > \epsilon \right) } } = 0,
\end{align}
for all $\epsilon > 0$.
\end{proof}


\subsubsection{NPG}

\textbf{\cref{lem:bias_variance_softmax_natural_pg_special_on_policy_stochastic_gradient}.}
For NPG, we have, $\expectation_{a \sim \pi_\theta(\cdot)}{ \left[ \hat{r} \right] } = r$, and $\expectation_{a \sim \pi_\theta(\cdot)}{ \left\| \hat{r} \right\|_2^2 } = \sum_{a \in [K]}{ \frac{ r(a)^2 }{ \pi_{\theta}(a) }  }$.
\begin{proof}
\textbf{First part.} $\expectation_{a \sim \pi_\theta(\cdot)}{ \left[ \hat{r} \right] } = r$.

We have, for all $i \in [K]$,
\begin{align}
    \expectation_{a \sim \pi_\theta(\cdot)}{ \left[ \hat{r}(i) \right] } = \sum_{a \in [K]}{ \pi_\theta(a) \cdot \frac{ \sI\left\{ a = i \right\} }{ \pi_{\theta}(i) } \cdot r(i) } = r(i).
\end{align}

\textbf{Second part.} $\expectation_{a \sim \pi_\theta(\cdot)}{ \left\| \hat{r} \right\|_2^2 } = \sum_{a \in [K]}{ \frac{ r(a)^2 }{ \pi_{\theta}(a) }  }$.

The squared $\ell_2$ norm of natural policy gradient is,
\begin{align}
    \left\| \hat{r} \right\|_2^2 = \sum_{i}{ \hat{r}(i)^2 } = \sum_{i}{ \frac{ \left( \sI\left\{ a = i \right\} \right)^2 }{ \pi_{\theta}(i)^2 } \cdot r(i)^2 } = \sum_{i}{ \frac{ \sI\left\{ a = i \right\} }{ \pi_{\theta}(i)^2 } \cdot r(i)^2 }.
\end{align}
The expected squared norm is,
\begin{align}
    \expectation_{a \sim \pi_\theta(\cdot)}{ \left\| \hat{r} \right\|_2^2 } &= \sum_{a \in [K]}{ \pi_\theta(a) \cdot \sum_{i}{ \frac{ \sI\left\{ a = i \right\} }{ \pi_{\theta}(i)^2 } \cdot r(i)^2 } } \\
    &= \sum_{a \in [K]}{ \pi_\theta(a) \cdot \frac{1}{ \pi_{\theta}(a)^2 } \cdot r(a)^2 } \\
    &= \sum_{a \in [K]}{ \frac{ r(a)^2 }{ \pi_{\theta}(a) }  }. \qedhere
\end{align}
\end{proof}

\textbf{\cref{thm:failure_probability_softmax_natural_pg_special_on_policy_stochastic_gradient}.}
Using \cref{update_rule:softmax_natural_pg_special_on_policy_stochastic_gradient}, we have: \textbf{(i)} with positive probability, $\sum_{a \not= a^*}{ \pi_{\theta_t}(a)} \to 1$ as $t \to \infty$; \textbf{(ii)} $\forall a \in [K]$, with positive probability, $\pi_{\theta_t}(a) \to 1$, as $t \to \infty$.
\begin{proof}
\textbf{First part.} With positive probability, $\sum_{a \not= a^*}{ \pi_{\theta_t}(a)} \to 1$ as $t \to \infty$.

Let $\probability$ denote the probability measure that over the probability space $(\Omega,\gF)$ that holds all of our random variables.
Let $\gB = \{ a\in [K]: a\ne a^* \}$.
By abusing notation, for any $\pi:[K] \to [0,1]$ map
we let $\pi_{\theta_t}(\gB)$ to stand for $\sum_{a\in \gB}\pi_{\theta_t}(a)$.
Define for $t\ge 1$ the event
$\gB_t = \{ a_t \ne a^* \} (= \{ a_t\in \gB \})$ and let $\gE_t = \gB_1 \cap \dots \cap \gB_t$.
Thus, $\gE_t$ is the event that $a^*$ was not chosen in the first $t$ time steps.
Note that $\{\gE_t\}_{t\ge 1}$ is a nested sequence and thus, by the monotone convergence theorem, 
\begin{align}
\label{eq:failure_probability_softmax_natural_pg_special_on_policy_stochastic_gradient_intermediate_1}
    \lim_{t\to\infty}\probability{(\gE_t)} = \probability{(\gE)}\,,
\end{align}
where $\gE = \cap_{t\ge 1}\gB_t$.
We start with a lower bound on the probability of $\gE_t$.
The lower bound is stated in a generic form:
In particular, let  $(b_t)_{t\ge 1}$ be a deterministic sequence which satisfies that for any $t\ge 1$,
\begin{align}
\label{eq:failure_probability_softmax_natural_pg_special_on_policy_stochastic_gradient_intermediate_2}
    \mathbb{I}_{\gE_{t-1}} \cdot  \pi_{\theta_t}(\gB) \ge \mathbb{I}_{\gE_{t-1}} \cdot b_t \qquad \text{  holds $\probability$-almost surely},
\end{align}
where we let $\gE_0=\Omega$ and for an event $\gE$, $\mathbb{I}_{\gE}$ stands for the characteristic function of $\gE$ (i.e., $\mathbb{I}_{\gE}(\omega)=1$ if $\omega\in \gE$ and $\mathbb{I}_{\gE}(\omega)=0$, otherwise).
We make the following claim:

\noindent \underline{Claim 1}: Under the above assumption, for any $t\ge 1$
it holds that 
\begin{align}
\label{eq:failure_probability_softmax_natural_pg_special_on_policy_stochastic_gradient_intermediate_3}
    \probability(\gE_t) \ge \prod_{s=1}^t b_s\,.
\end{align}

For the proof of this claim let $\gH_t$ denote the sequence formed of the first $t$ actions:
\begin{align}
    \gH_t \coloneqq \left( a_1, a_2, \cdots, a_t \right).
\end{align}
By definition, 
\begin{align}
    \theta_t = \gA\left( \theta_1, a_1, r(a_1), \theta_2, a_2, r(a_2), \cdots, \theta_{t-1}, a_{t-1}, r(a_{t-1}) \right).
\end{align}
By our assumption that the $t$th action is chosen from $\pi_{\theta_t}$, it follows that 
$\probability$ satisfies that
for all $a$ and $t \ge 1$, 
\begin{align}
\label{eq:failure_probability_softmax_natural_pg_special_on_policy_stochastic_gradient_intermediate_4}
    \probability{ \left( a_t = a \ | \ \gH_{t-1} \right) } = \pi_{\theta_t}(a) \qquad \text{ $\probability$-almost surely.} 
\end{align}

We prove the claim by induction on $t$.
For $t=1$, from \cref{eq:failure_probability_softmax_natural_pg_special_on_policy_stochastic_gradient_intermediate_2,eq:failure_probability_softmax_natural_pg_special_on_policy_stochastic_gradient_intermediate_4}, using that $\gE_0 = \Omega$ and $H_0=()$, we have that $\probability$-almost surely,
\begin{align}
    \probability{ ( \gE_1 ) } = \pi_{\theta_1}(\gB).
\end{align}
Suppose the claim holds up to $t-1$. We have,
\begin{align}
\label{eq:failure_probability_softmax_natural_pg_special_on_policy_stochastic_gradient_intermediate_5}
\MoveEqLeft
    \probability{ ( \gE_t ) } = \expectation{ \left[ \probability{ (  \gE_t  \ | \ \gH_{t-1} ) } \right] }  \tag{by the tower rule}\\
    &= \expectation{  \left[ \sI_{\gE_{t-1}} \cdot \probability{ ( \gB_t \ | \ \gH_{t-1} ) } \right] } \tag{$\gE_{t-1}$ is $\gH_{t-1}$-measurable}\\
    &= \expectation{  \left[ \sI_{\gE_{t-1}} \cdot \pi_{\theta_t}(\gB) \right] }
    \tag{by \cref{eq:failure_probability_softmax_natural_pg_special_on_policy_stochastic_gradient_intermediate_4}}  \\
    &\ge \expectation{  \left[ \sI_{\gE_{t-1}} \cdot b_t \right] } \tag{by \cref{eq:failure_probability_softmax_natural_pg_special_on_policy_stochastic_gradient_intermediate_2}}\\
    &= b_t \cdot \probability{ ( \gE_{t-1} ) } \tag{$b_t$ is deterministic}\\
    &= \prod_{s=1}^{t} b_s\,. \tag{induction hypothesis}
\end{align}

Now, we claim the following:

\noindent \underline{Claim 2}: A suitable choice for $b_t$ is
\begin{align}
\label{eq:failure_probability_softmax_natural_pg_special_on_policy_stochastic_gradient_intermediate_6}
    b_t = \exp\left\{ \frac{ - \exp\left\{ \theta_1(a^*) \right\} }{ (K-1) \cdot \exp{ \Big\{ \frac{ \sum_{a \not= a^*}{ \theta_1(a) } + \eta \cdot r_{\min} \cdot (t-1) }{K-1} \Big\} } } \right\}.
\end{align}

\underline{Proof of Claim 2}:
Clearly, it suffices to show that for any sequence $(a_1,\dots,a_{t-1})$ such that $a_s\ne a^*$, $\theta_t := \gA(\theta_1,a_1,r(a_1),\dots,a_{t-1},r(a_{t-1}))$ is such that
$\pi_{\theta_t}(\gB) \ge b_t$ with $b_t$ as defined in \cref{eq:failure_probability_softmax_natural_pg_special_on_policy_stochastic_gradient_intermediate_6}. 

We have, for each sub-optimal action $a \not= a^*$,
\begin{align}
\label{eq:failure_probability_softmax_natural_pg_special_on_policy_stochastic_gradient_intermediate_7}
\MoveEqLeft
    \theta_t(a) = \theta_1(a) + \eta \cdot \sum_{s=1}^{t-1}{ \hat{r}_s(a) } \qquad \left( \text{by \cref{update_rule:softmax_natural_pg_special_on_policy_stochastic_gradient}} \right) \\
    &= \theta_1(a) + \eta \cdot \sum_{s=1}^{t-1}{ \frac{ \sI\left\{ a_s = a \right\} }{ \pi_{\theta_s}(a) } \cdot r(a) } \qquad \left( \text{by \cref{def:on_policy_importance_sampling}} \right)  \\
    &\ge \theta_1(a) + \eta \cdot \sum_{s=1}^{t-1}{ \sI\left\{ a_s = a \right\} \cdot r(a) } \qquad \left( \pi_{\theta_s}(a) \in (0, 1), \ r(a) \in (0, 1] \right) \\
    &\ge \theta_1(a) + \eta \cdot r_{\min} \cdot \sum_{s=1}^{t-1}{ \sI\left\{ a_s = a \right\} }, \qquad \left( r_{\min} \coloneqq \min_{a \not= a^*}{ r(a) } \right)
\end{align}
where $r_{\min} \in (0, 1]$ according to \cref{asmp:positive_reward}, i.e., $r(a) \in (0, 1]$ for all $a \in [K]$. Then we have,
\begin{align}
\label{eq:failure_probability_softmax_natural_pg_special_on_policy_stochastic_gradient_intermediate_8}
\MoveEqLeft
    \sum_{a \not= a^*}{ \exp\left\{ \theta_t(a) \right\} } \ge (K-1) \cdot \exp{ \left\{ \frac{ \sum_{a \not= a^*}{ \theta_t(a) } }{K-1} \right\} } \qquad \left( \text{by Jensen's inequality} \right) \\
    &\ge (K-1) \cdot \exp{ \left\{ \frac{ \sum_{a \not= a^*}{ \theta_1(a) } + \eta \cdot r_{\min} \cdot \sum_{a \not= a^*}{ \sum_{s=1}^{t-1}{ \sI\left\{ a_s = a \right\} } } }{K-1} \right\} } \qquad \left( \text{by \cref{eq:failure_probability_softmax_natural_pg_special_on_policy_stochastic_gradient_intermediate_7}} \right) \\
    &= (K-1) \cdot \exp{ \left\{ \frac{ \sum_{a \not= a^*}{ \theta_1(a) } + \eta \cdot r_{\min} \cdot (t-1) }{K-1} \right\} }. \qquad \left( a_1 \not= a^*, a_2 \not= a^*, \cdots, a_{t-1} \not= a^* \right)
\end{align}
On the other hand, we have,
\begin{align}
\label{eq:failure_probability_softmax_natural_pg_special_on_policy_stochastic_gradient_intermediate_9}
    \theta_t(a^*) &= \theta_1(a^*) + \eta \cdot \sum_{s=1}^{t-1}{ \frac{ \sI\left\{ a_s = a^* \right\} }{ \pi_{\theta_s}(a^*) } \cdot r(a^*) } \qquad \left( \text{by \cref{update_rule:softmax_natural_pg_special_on_policy_stochastic_gradient,def:on_policy_importance_sampling}} \right) \\
    &= \theta_1(a^*). \qquad \left( a_s \not= a^* \text{ for all } s \in \left\{ 1, 2, \dots, t-1 \right\} \right)
\end{align}
Next, we have,
\begin{align}
\label{eq:failure_probability_softmax_natural_pg_special_on_policy_stochastic_gradient_intermediate_10}
\MoveEqLeft
    \sum_{a \not= a^*}{ \pi_{\theta_t}(a)} = 1 - \pi_{\theta_t}(a^*) \\
    &= 1 - \frac{ \exp\left\{ \theta_t(a^*) \right\} }{ \sum_{a \not= a^*}{ \exp\left\{ \theta_t(a) \right\} } + \exp\left\{ \theta_t(a^*) \right\} } \\
    &\ge 1 - \frac{ \exp\left\{ \theta_1(a^*) \right\} }{ (K-1) \cdot \exp{ \left\{ \frac{ \sum_{a \not= a^*}{ \theta_1(a) } + \eta \cdot r_{\min} \cdot (t-1) }{K-1} \right\} } + \exp\left\{ \theta_1(a^*) \right\} }. \qquad \left( \text{by \cref{eq:failure_probability_softmax_natural_pg_special_on_policy_stochastic_gradient_intermediate_8,eq:failure_probability_softmax_natural_pg_special_on_policy_stochastic_gradient_intermediate_9}} \right)
\end{align}
According to \cref{lem:auxiliary_lemma_1}, for all $x \in (0, 1)$,
\begin{align}
\label{eq:failure_probability_softmax_natural_pg_special_on_policy_stochastic_gradient_intermediate_11}
    1 - x \ge \exp\Big\{ \frac{-1 }{ 1/x - 1 } \Big\}.
\end{align}
Let 
\begin{align}
    x = \frac{ \exp\left\{ \theta_1(a^*) \right\} }{ (K-1) \cdot \exp{ \Big\{ \frac{ \sum_{a \not= a^*}{ \theta_1(a) } + \eta \cdot r_{\min} \cdot (t-1) }{K-1} \Big\} } + \exp\left\{ \theta_1(a^*) \right\} } \in (0, 1).
\end{align}
We have,
\begin{align}
\label{eq:failure_probability_softmax_natural_pg_special_on_policy_stochastic_gradient_intermediate_12}
\MoveEqLeft
    \sum_{a \not= a^*}{ \pi_{\theta_t}(a)} \ge \exp\left\{ \frac{-1 }{ \frac{ (K-1) \cdot \exp{ \big\{ \frac{ \sum_{a \not= a^*}{ \theta_1(a) } + \eta \cdot r_{\min} \cdot (t-1) }{K-1} \big\} } + \exp\left\{ \theta_1(a^*) \right\} }{ \exp\left\{ \theta_1(a^*) \right\} } - 1 } \right\} \qquad \left( \text{by \cref{eq:failure_probability_softmax_natural_pg_special_on_policy_stochastic_gradient_intermediate_10,eq:failure_probability_softmax_natural_pg_special_on_policy_stochastic_gradient_intermediate_11}} \right) \\
    &= \exp\left\{ \frac{ - \exp\left\{ \theta_1(a^*) \right\} }{ (K-1) \cdot \exp{ \Big\{ \frac{ \sum_{a \not= a^*}{ \theta_1(a) } + \eta \cdot r_{\min} \cdot (t-1) }{K-1} \Big\} } } \right\} = b_t,
\end{align}
finishing the proof of the claim.

Combining \cref{eq:failure_probability_softmax_natural_pg_special_on_policy_stochastic_gradient_intermediate_1} with 
the conclusions of Claim 1 and 2 together, we get
\begin{align}
\label{eq:failure_probability_softmax_natural_pg_special_on_policy_stochastic_gradient_intermediate_13}
\MoveEqLeft
\probability{(\gE)}
    \ge \prod_{t=1}^{\infty}{ \exp\left\{ \frac{ - \exp\left\{ \theta_1(a^*) \right\} }{ (K-1) \cdot \exp{ \Big\{ \frac{ \sum_{a \not= a^*}{ \theta_1(a) } + \eta \cdot r_{\min} \cdot (t-1) }{K-1} \Big\} } } \right\} } \qquad \left( \text{by \cref{eq:failure_probability_softmax_natural_pg_special_on_policy_stochastic_gradient_intermediate_12}} \right) \\
    &= \exp{ \left\{ - \frac{ \exp\left\{ \theta_1(a^*) \right\} }{ \exp{ \Big\{ \frac{ \sum_{a \not= a^*}{ \theta_1(a) } }{K-1} \Big\} } } \cdot \frac{ \exp{ \big\{ \frac{ \eta \cdot r_{\min} }{K-1} \big\} } }{ K - 1 } \cdot \sum_{t=1}^{\infty}{ \frac{1}{ \exp{ \big\{ \frac{ \eta \cdot r_{\min} \cdot t }{K-1} \big\} } } } \right\} } \\
    &\ge \exp{ \left\{ - \frac{ \exp\left\{ \theta_1(a^*) \right\} }{ \exp{ \Big\{ \frac{ \sum_{a \not= a^*}{ \theta_1(a) } }{K-1} \Big\} } } \cdot \frac{ \exp{ \big\{ \frac{ \eta \cdot r_{\min} }{K-1} \big\} } }{ K - 1} \cdot \int_{t=0}^{\infty}{ \frac{1}{ \exp{ \big\{ \frac{ \eta \cdot r_{\min} \cdot t }{K-1} \big\} } } dt } \right\} } \\
    &= \exp{ \left\{ - \frac{ \exp\left\{ \theta_1(a^*) \right\} }{ \exp{ \Big\{ \frac{ \sum_{a \not= a^*}{ \theta_1(a) } }{K-1} \Big\} } } \cdot \frac{ \exp{ \big\{ \frac{ \eta \cdot r_{\min} }{K-1} \big\} } }{ K - 1 } \cdot \frac{K-1}{ \eta \cdot r_{\min} } \right\} } \\
    &= \exp{ \left\{ - \frac{ \exp\left\{ \theta_1(a^*) \right\} }{ \exp{ \Big\{ \frac{ \sum_{a \not= a^*}{ \theta_1(a) } }{K-1} \Big\} } } \cdot \frac{ \exp{ \big\{ \frac{ \eta \cdot r_{\min} }{K-1} \big\} } }{ \eta \cdot r_{\min} } \right\} }.
\end{align}
Note that $r_{\min} \in \Theta(1)$, $\exp\left\{ \theta_1(a^*) \right\} \in \Theta(1)$, $\eta \in \Theta(1)$, $\exp{ \big\{ \frac{ \eta \cdot r_{\min} }{K-1} \big\} } \in \Theta(1)$ and,
\begin{align}
    \exp{ \bigg\{ \frac{ \sum_{a \not= a^*}{ \theta_1(a) } }{K-1} \bigg\} } \in \Theta(1).
\end{align}
Therefore, we have ``the probability of sampling sub-optimal actions forever using on-policy sampling $a_t \sim \pi_{\theta_t}(\cdot)$'' is lower bounded by a constant of $ \frac{1}{ \exp\left\{ \Theta(1) \right\} } \in \Theta(1)$, which implies that with positive probability $\Theta(1)$, we have $\sum_{a \not= a^*}{ \pi_{\theta_t}(a)} \to 1$ as $t \to \infty$.

\textbf{Second part.} $\forall a \in [K]$, with positive probability, $\pi_{\theta_t}(a) \to 1$, as $t \to \infty$.

The proof is similar to the first part. Let $\gB = \{ a \}$.
For any $\pi:[K] \to [0,1]$ map
we let $\pi_{\theta_t}(\gB)$ to stand for $\pi_{\theta_t}(a)$.
Define for $t\ge 1$ the event
$\gB_t = \{ a_t = a \} (= \{ a_t\in \gB \})$ and let $\gE_t = \gB_1 \cap \dots \cap \gB_t$.
Thus, $\gE_t$ is the event that $a$ was chosen in the first $t$ time steps.
Note that $\{\gE_t\}_{t\ge 1}$ is a nested sequence and thus, by the monotone convergence theorem, $\lim_{t\to\infty}\probability{(\gE_t)} = \probability{(\gE)}$, where $\gE = \cap_{t\ge 1}\gB_t$. We show that by letting
\begin{align}
\label{eq:failure_probability_softmax_natural_pg_special_on_policy_stochastic_gradient_intermediate_14}
    b_t = \exp\bigg\{ \frac{ - \sum_{a^\prime \not= a}{ \exp\{ \theta_1(a^\prime) \} } }{ \exp\{ \theta_1(a) + \eta \cdot r(a) \cdot \left( t - 1 \right) \} } \bigg\},
\end{align}
we have \cref{eq:failure_probability_softmax_natural_pg_special_on_policy_stochastic_gradient_intermediate_2,eq:failure_probability_softmax_natural_pg_special_on_policy_stochastic_gradient_intermediate_3} hold using the arguments in the first part.

It suffices to show that for any sequence $(a_1,\dots,a_{t-1})$ such that $a_s = a$, for all $s \in \{ 1, 2, \dots, t-1 \}$, $\theta_t := \gA(\theta_1,a_1,r(a_1),\dots,a_{t-1},r(a_{t-1}))$ is such that
$\pi_{\theta_t}(\gB) \ge b_t$ with $b_t$ as defined in \cref{eq:failure_probability_softmax_natural_pg_special_on_policy_stochastic_gradient_intermediate_14}. Now suppose $a_1 = a, a_2 = a, \cdots, a_{t-1} = a$. We have,
\begin{align}
\label{eq:failure_probability_softmax_natural_pg_special_on_policy_stochastic_gradient_intermediate_15}
    \theta_{t}(a) &= \theta_1(a) + \eta \cdot \sum_{s=1}^{t-1}{ \hat{r}_s(a) } \qquad \left( \text{by \cref{update_rule:softmax_natural_pg_special_on_policy_stochastic_gradient}} \right) \\
    &= \theta_1(a) + \eta \cdot \sum_{s=1}^{t-1}{ \frac{ \sI\left\{ a_s = a \right\} }{ \pi_{\theta_s}(a) } \cdot r(a) } \qquad \left( \text{by \cref{def:on_policy_importance_sampling}} \right) \\
    &= \theta_1(a) + \eta \cdot \sum_{s=1}^{t-1}{ \frac{ r(a) }{ \pi_{\theta_s}(a) }  } \qquad \left( a_s = a \text{ for all } s \in \left\{ 1, 2, \dots, t-1 \right\} \right) \\
    &\ge \theta_1(a) + \eta \cdot \sum_{s=1}^{t-1}{ r(a) } \qquad \left( \pi_{\theta_s}(a) \in (0, 1) \right) \\
    &= \theta_1(a) +  \eta \cdot r(a) \cdot \left( t - 1 \right).
\end{align}
On the other hand, we have, for any other action $a^\prime \not= a$,
\begin{align}
\label{eq:failure_probability_softmax_natural_pg_special_on_policy_stochastic_gradient_intermediate_16}
    \theta_{t}(a^\prime) &= \theta_1(a^\prime) + \eta \cdot \sum_{s=1}^{t-1}{ \frac{ \sI\left\{ a_s = a^\prime \right\} }{ \pi_{\theta_s}(a^\prime) } \cdot r(a^\prime) } \qquad \left( \text{by \cref{update_rule:softmax_natural_pg_special_on_policy_stochastic_gradient,def:on_policy_importance_sampling}} \right) \\
    &= \theta_1(a^\prime). \qquad \left( a_s \not= a^\prime \text{ for all } s \in \left\{ 1, 2, \dots, t-1 \right\} \right)
\end{align}
Therefore, we have,
\begin{align}
\label{eq:failure_probability_softmax_natural_pg_special_on_policy_stochastic_gradient_intermediate_17}
\MoveEqLeft
    \pi_{\theta_t}(a) = 1 - \sum_{a^\prime \not= a}{ \pi_{\theta_t}(a^\prime) } \\
    &= 1 - \frac{ \sum_{a^\prime \not= a}{ \exp\{ \theta_t(a^\prime) \} } }{ \exp\{ \theta_t(a) \} + \sum_{a^\prime \not= a}{ \exp\{ \theta_t(a^\prime) \} } } \\
    &\ge 1 - \frac{ \sum_{a^\prime \not= a}{ \exp\{ \theta_1(a^\prime) \} } }{ \exp\{ \theta_1(a) + \eta \cdot r(a) \cdot \left( t - 1 \right) \} + \sum_{a^\prime \not= a}{ \exp\{ \theta_1(a^\prime) \} } }. \qquad \left( \text{by \cref{eq:failure_probability_softmax_natural_pg_special_on_policy_stochastic_gradient_intermediate_15,eq:failure_probability_softmax_natural_pg_special_on_policy_stochastic_gradient_intermediate_16}} \right)
\end{align}
Let 
\begin{align}
    x = \frac{ \sum_{a^\prime \not= a}{ \exp\{ \theta_1(a^\prime) \} } }{ \exp\{ \theta_1(a) + \eta \cdot r(a) \cdot \left( t - 1 \right) \} + \sum_{a^\prime \not= a}{ \exp\{ \theta_1(a^\prime) \} } } \in (0, 1).
\end{align}
We have,
\begin{align}
\label{eq:failure_probability_softmax_natural_pg_special_on_policy_stochastic_gradient_intermediate_18}
\MoveEqLeft
    \pi_{\theta_t}(a) \ge 1 - x \qquad \left( \text{by \cref{eq:failure_probability_softmax_natural_pg_special_on_policy_stochastic_gradient_intermediate_17}} \right) \\
    &\ge \exp\left\{ \frac{-1 }{ \frac{ \exp\{ \theta_1(a) + \eta \cdot r(a) \cdot \left( t - 1 \right) \} + \sum_{a^\prime \not= a}{ \exp\{ \theta_1(a^\prime) \} } }{ \sum_{a^\prime \not= a}{ \exp\{ \theta_1(a^\prime) \} } } - 1 } \right\} \qquad \left( \text{by \cref{lem:auxiliary_lemma_1}} \right) \\
    &= \exp\bigg\{ \frac{ - \sum_{a^\prime \not= a}{ \exp\{ \theta_1(a^\prime) \} } }{ \exp\{ \theta_1(a) + \eta \cdot r(a) \cdot \left( t - 1 \right) \} } \bigg\} \\
    &= b_t.
\end{align}
Therefore we have,
\begin{align}
\label{eq:failure_probability_softmax_natural_pg_special_on_policy_stochastic_gradient_intermediate_19}
\MoveEqLeft
    \prod_{t=1}^{\infty}{ \pi_{\theta_t}(a) } \ge \prod_{t=1}^{\infty}{ \exp\bigg\{ \frac{ - \sum_{a^\prime \not= a}{ \exp\{ \theta_1(a^\prime) \} } }{ \exp\{ \theta_1(a) + \eta \cdot r(a) \cdot \left( t - 1 \right) \} } \bigg\} } \qquad \left( \text{by \cref{eq:failure_probability_softmax_natural_pg_special_on_policy_stochastic_gradient_intermediate_18}} \right) \\
    &= \exp\bigg\{ - \sum_{a^\prime \not= a}{ \exp\{ \theta_1(a^\prime) \} } \cdot \frac{ \exp\{ \eta \cdot r(a) \} }{ \exp\{ \theta_1(a) \} } \cdot  \sum_{t=1}^{\infty}{  \frac{ 1 }{ \exp\{ \eta \cdot r(a) \cdot t \} } } \bigg\} \\
    &\ge \exp\bigg\{ - \sum_{a^\prime \not= a}{ \exp\{ \theta_1(a^\prime) \} } \cdot \frac{ \exp\{ \eta \cdot r(a) \} }{ \exp\{ \theta_1(a) \} } \cdot \int_{t=0}^{\infty}{  \frac{ 1 }{ \exp\{ \eta \cdot r(a) \cdot t \} } dt } \bigg\} \\
    &= \exp\bigg\{ - \frac{ \exp\{ \eta \cdot r(a) \} }{ \eta \cdot r(a)} \cdot \frac{ \sum_{a^\prime \not= a}{ \exp\{ \theta_1(a^\prime) \} } }{ \exp\{ \theta_1(a) \} } \bigg\} \\
    &\in \Omega(1),
\end{align}
where the last line is due to $r(a) \in \Theta(1)$, $\exp\left\{ \theta_1(a) \right\} \in \Theta(1)$ for all $a \in [K]$, and $\eta \in \Theta(1)$. With \cref{eq:failure_probability_softmax_natural_pg_special_on_policy_stochastic_gradient_intermediate_19}, we have ``the probability of sampling action $a$ forever using on-policy sampling $a_t \sim \pi_{\theta_t}(\cdot)$'' is lower bounded by a constant of $\Omega(1)$. Therefore, for all $a \in [K]$, with positive probability $\Omega(1)$, $\pi_{\theta_t}(a) \to 1$, as $t \to \infty$. 
\end{proof}

\subsubsection{GNPG}

\begin{lemma}
Using on-policy IS estimator of \cref{def:on_policy_importance_sampling}, the stochastic GNPG is biased, i.e.,
\begin{align}
    \expectation_{a \sim \pi_{\theta}(\cdot)} \left[ \frac{d \pi_{\theta}^\top \hat{r}}{d \theta} \bigg/ \left\| \frac{d \pi_{\theta}^\top \hat{r}}{d \theta} \right\|_2 \right] \not= \frac{d \pi_{\theta}^\top r}{d \theta} \bigg/ \left\| \frac{d \pi_{\theta}^\top r}{d \theta} \right\|_2.
\end{align}
\end{lemma}
\begin{proof}
Consider a two-action example with $r(1) > r(2)$. The true normalized PG of $a^* = 1$ is,
\begin{align}
    g(1) &\coloneqq \frac{d \pi_{\theta}^\top r}{d \theta(1)} \bigg/ \left\| \frac{d \pi_{\theta}^\top r}{d \theta} \right\|_2 \\
    &= \frac{ \pi_{\theta}(1) \cdot \left( r(1) - \pi_\theta^\top r \right) }{ \sqrt{ \pi_{\theta}(1)^2 \cdot \left( r(1) - \pi_\theta^\top r \right)^2 + \pi_{\theta}(2)^2 \cdot \left( r(2) - \pi_\theta^\top r \right)^2 } } \\
    &= \frac{ \pi_{\theta}(1) \cdot \pi_{\theta}(2) \cdot \left( r(1) - r(2) \right) }{ \sqrt{ \pi_{\theta}(1)^2 \cdot \pi_{\theta}(2)^2 \cdot \left( r(1) - r(2) \right)^2 + \pi_{\theta}(1)^2 \cdot \pi_{\theta}(2)^2 \cdot \left( r(1) - r(2) \right)^2 } } \\
    &= \frac{1}{ \sqrt{2} }.
\end{align}
On the other hand, the stochastic normalized PG of $a^* = 1$ is,
\begin{align}
    \hat{g}(1) &\coloneqq \expectation_{a \sim \pi_{\theta}(\cdot)} \left[ \frac{d \pi_{\theta}^\top \hat{r}}{d \theta(1)} \bigg/ \left\| \frac{d \pi_{\theta}^\top \hat{r}}{d \theta} \right\|_2 \right] \\
    &= \pi_{\theta}(1) \cdot \frac{ \pi_{\theta}(1) \cdot \left( \frac{r(1)}{ \pi_{\theta}(1) } - \pi_{\theta}(1) \cdot \frac{r(1)}{ \pi_{\theta}(1) } \right) }{ \sqrt{ \pi_{\theta}(1)^2 \cdot \left( \frac{r(1)}{ \pi_{\theta}(1) } - \pi_{\theta}(1) \cdot \frac{r(1)}{ \pi_{\theta}(1) } \right)^2 + \pi_{\theta}(2)^2 \cdot \left( 0 - \pi_{\theta}(1) \cdot \frac{r(1)}{ \pi_{\theta}(1) } \right)^2 } } \\
    &\qquad + \pi_{\theta}(2) \cdot \frac{ \pi_{\theta}(1) \cdot \left( 0 - \pi_{\theta}(2) \cdot \frac{r(2)}{ \pi_{\theta}(2) } \right) }{ \sqrt{ \pi_{\theta}(1)^2 \cdot \left( 0 - \pi_{\theta}(2) \cdot \frac{r(2)}{ \pi_{\theta}(2) } \right)^2 + \pi_{\theta}(2)^2 \cdot \left( \frac{r(2)}{ \pi_{\theta}(2) } - \pi_{\theta}(2) \cdot \frac{r(2)}{ \pi_{\theta}(2) } \right)^2 } } \\
    &= \pi_{\theta}(1) \cdot \frac{ \pi_{\theta}(2) \cdot r(1) }{ \sqrt{ \pi_{\theta}(2)^2 \cdot r(1)^2 + \pi_{\theta}(2)^2 \cdot r(1)^2 } } - \pi_{\theta}(2) \cdot \frac{ \pi_{\theta}(1) \cdot r(2) }{ \sqrt{ \pi_{\theta}(1)^2 \cdot r(2)^2 + \pi_{\theta}(1)^2 \cdot r(2)^2 } } \\
    &= \frac{1}{ \sqrt{2} } \cdot \left( \pi_{\theta}(1) - \pi_{\theta}(2) \right).
\end{align}
It is clear that the true normalized PG of $a^* = 1$ is always positive $g(1) > 0$, while the expectation of the stochastic normalized PG estimator of $a^* = 1$ is negative when $\pi_{\theta}(1) <  \pi_{\theta}(2)$.
\end{proof}

\textbf{\cref{thm:failure_probability_softmax_gnpg_special_on_policy_stochastic_gradient}.}
Using \cref{update_rule:softmax_gnpg_special_on_policy_stochastic_gradient}, 
we have, $\forall a \in [K]$, with positive probability, $\pi_{\theta_t}(a) \to 1$, as $t \to \infty$. 
\begin{proof}
The proof is similar to the second part of \cref{thm:failure_probability_softmax_natural_pg_special_on_policy_stochastic_gradient}. We first calculate the stochastic normalized PG in each iteration. Denote $a_t$ as the action sampled at $t$-th iteration. We have,
\begin{align}
\label{eq:failure_probability_softmax_gnpg_special_on_policy_stochastic_gradient_intermediate_1}
    \frac{d \pi_{\theta_t}^\top \hat{r}_t}{d \theta_t(a_t)} &= \pi_{\theta_t}(a_t) \cdot ( \hat{r}_t(a_t) - \pi_{\theta_t}^\top \hat{r}_t ) \\
    &= \pi_{\theta_t}(a_t) \cdot \left( \frac{ r(a_t) }{ \pi_{\theta_t}(a_t) } - \pi_{\theta_t}(a_t) \cdot \frac{ r(a_t) }{ \pi_{\theta_t}(a_t) } \right) \qquad \left( \text{by \cref{def:on_policy_importance_sampling}} \right) \\
    &= \left( 1 - \pi_{\theta_t}(a_t) \right) \cdot r(a_t).
\end{align}
On the other hand, for all $a^\prime \not= a_t$,
\begin{align}
\label{eq:failure_probability_softmax_gnpg_special_on_policy_stochastic_gradient_intermediate_2}
    \frac{d \pi_{\theta_t}^\top \hat{r}_t}{d \theta_t(a^\prime)} &= \pi_{\theta_t}(a^\prime) \cdot ( \hat{r}_s(a^\prime) - \pi_{\theta_t}^\top \hat{r}_t ) \\
    &= \pi_{\theta_t}(a^\prime) \cdot \left( 0 - \pi_{\theta_t}(a_t) \cdot \frac{ r(a_t) }{ \pi_{\theta_t}(a_t) } \right) \qquad \left( \text{by \cref{def:on_policy_importance_sampling}} \right) \\
    &= - \pi_{\theta_t}(a^\prime) \cdot r(a_t).
\end{align}
Therefore, the stochastic PG norm is,
\begin{align}
\MoveEqLeft
\label{eq:failure_probability_softmax_gnpg_special_on_policy_stochastic_gradient_intermediate_3}
    \bigg\| \frac{d \pi_{\theta_t}^\top \hat{r}_t}{d \theta_t} \bigg\|_2 = \left[ \left( \frac{d \pi_{\theta_t}^\top \hat{r}_t}{d \theta_t(a_t)} \right)^2 + \sum_{a^\prime \not= a_t}{ \left( \frac{d \pi_{\theta_t}^\top \hat{r}_t}{d \theta_t(a^\prime)} \right)^2} \right]^{\frac{1}{2}} \\
    &= \left[ \left( 1 - \pi_{\theta_t}(a_t) \right)^2 \cdot r(a_t)^2 + \sum_{a^\prime \not= a_t}{  \pi_{\theta_t}(a^\prime)^2 \cdot r(a_t)^2 } \right]^{\frac{1}{2}} \qquad \left( \text{by \cref{eq:failure_probability_softmax_gnpg_special_on_policy_stochastic_gradient_intermediate_1,eq:failure_probability_softmax_gnpg_special_on_policy_stochastic_gradient_intermediate_2}} \right) \\
    &\le \left[ \left( 1 - \pi_{\theta_t}(a_t) \right)^2 \cdot r(a_t)^2 +  \bigg( \sum_{a^\prime \not= a_t}{  \pi_{\theta_t}(a^\prime)} \bigg)^2 \cdot r(a_t)^2  \right]^{\frac{1}{2}} \qquad \left( \left\| x \right\|_2 \le \left\| x \right\|_1 \right) \\
    &= \sqrt{2} \cdot \left( 1 - \pi_{\theta_t}(a_t) \right) \cdot r(a_t).
\end{align}
The proof is then similar to the second part of \cref{thm:failure_probability_softmax_natural_pg_special_on_policy_stochastic_gradient}. Let $\gB = \{ a \}$.
For any $\pi:[K] \to [0,1]$ map
we let $\pi_{\theta_t}(\gB)$ to stand for $\pi_{\theta_t}(a)$.
Define for $t\ge 1$ the event
$\gB_t = \{ a_t = a \} (= \{ a_t\in \gB \})$ and let $\gE_t = \gB_1 \cap \dots \cap \gB_t$.
Thus, $\gE_t$ is the event that $a$ was chosen in the first $t$ time steps.
Note that $\{\gE_t\}_{t\ge 1}$ is a nested sequence and thus, by the monotone convergence theorem, $\lim_{t\to\infty}\probability{(\gE_t)} = \probability{(\gE)}$, where $\gE = \cap_{t\ge 1}\gB_t$. We show that by letting
\begin{align}
\label{eq:failure_probability_softmax_gnpg_special_on_policy_stochastic_gradient_intermediate_4}
    b_t = \exp\bigg\{ \frac{ - \sum_{a^\prime \not= a}{ \exp\{ \theta_1(a^\prime) \} } }{ \exp\big\{ \theta_1(a) + \frac{\eta}{\sqrt{2}} \cdot \left( t - 1 \right) \big\} } \bigg\},
\end{align}
we have \cref{eq:failure_probability_softmax_natural_pg_special_on_policy_stochastic_gradient_intermediate_2,eq:failure_probability_softmax_natural_pg_special_on_policy_stochastic_gradient_intermediate_3} hold using the arguments in the first part of \cref{thm:failure_probability_softmax_natural_pg_special_on_policy_stochastic_gradient}.

It suffices to show that for any sequence $(a_1,\dots,a_{t-1})$ such that $a_s = a$, for all $s \in \{ 1, 2, \dots, t-1 \}$, $\theta_t := \gA(\theta_1,a_1,r(a_1),\dots,a_{t-1},r(a_{t-1}))$ is such that
$\pi_{\theta_t}(\gB) \ge b_t$ with $b_t$ as defined in \cref{eq:failure_probability_softmax_gnpg_special_on_policy_stochastic_gradient_intermediate_4}. Now suppose $a_1 = a, a_2 = a, \cdots, a_{t-1} = a$. We have,
\begin{align}
\label{eq:failure_probability_softmax_gnpg_special_on_policy_stochastic_gradient_intermediate_5}
    \theta_{t}(a) &= \theta_1(a) + \eta \cdot \sum_{s=1}^{t-1}{ \frac{d \pi_{\theta_s}^\top \hat{r}_s}{d {\theta_s}(a)} \bigg/ \bigg\| \frac{d \pi_{\theta_s}^\top \hat{r}_s}{d {\theta_s}} \bigg\|_2 } \qquad \left( \text{by \cref{update_rule:softmax_gnpg_special_on_policy_stochastic_gradient}} \right) \\
    &\ge \theta_1(a) + \eta \cdot \sum_{s=1}^{t-1}{ \frac{ \left( 1 - \pi_{\theta_s}(a) \right) \cdot r(a) }{ \sqrt{2} \cdot \left( 1 - \pi_{\theta_s}(a) \right) \cdot r(a) } } \qquad \left( \text{by \cref{eq:failure_probability_softmax_gnpg_special_on_policy_stochastic_gradient_intermediate_1,eq:failure_probability_softmax_gnpg_special_on_policy_stochastic_gradient_intermediate_3}} \right) \\
    &= \theta_1(a) + \frac{\eta}{\sqrt{2}} \cdot \left( t - 1 \right).
\end{align}
On the other hand, for all $a^\prime \not= a$, we have,
\begin{align}
\label{eq:failure_probability_softmax_gnpg_special_on_policy_stochastic_gradient_intermediate_6}
    \theta_{t}(a^\prime) &= \theta_1(a^\prime) - \eta \cdot \sum_{s=1}^{t-1}{ \left( \pi_{\theta_s}(a^\prime) \cdot r(a) \right) \bigg/ \bigg\| \frac{d \pi_{\theta_s}^\top \hat{r}_s}{d {\theta_s}} \bigg\|_2 } \qquad \left( \text{by \cref{update_rule:softmax_gnpg_special_on_policy_stochastic_gradient,eq:failure_probability_softmax_gnpg_special_on_policy_stochastic_gradient_intermediate_2}} \right) \\ 
    &\le \theta_1(a^\prime).
\end{align}
Then we have,
\begin{align}
\label{eq:failure_probability_softmax_gnpg_special_on_policy_stochastic_gradient_intermediate_7}
\MoveEqLeft
    \pi_{\theta_t}(a) = 1 - \frac{ \sum_{a^\prime \not= a}{ \exp\{ \theta_t(a^\prime) \} } }{ \exp\{ \theta_t(a) \} + \sum_{a^\prime \not= a}{ \exp\{ \theta_t(a^\prime) \} } } \\
    &\ge 1 - \frac{ \sum_{a^\prime \not= a}{ \exp\{ \theta_1(a^\prime) \} } }{ \exp\big\{ \theta_1(a) + \frac{\eta}{\sqrt{2}} \cdot \left( t - 1 \right) \big\} + \sum_{a^\prime \not= a}{ \exp\{ \theta_1(a^\prime) \} } } \qquad \left( \text{by \cref{eq:failure_probability_softmax_gnpg_special_on_policy_stochastic_gradient_intermediate_5,eq:failure_probability_softmax_gnpg_special_on_policy_stochastic_gradient_intermediate_6}} \right) \\
    &\ge \exp\bigg\{ \frac{ - \sum_{a^\prime \not= a}{ \exp\{ \theta_1(a^\prime) \} } }{ \exp\big\{ \theta_1(a) + \frac{\eta}{\sqrt{2}} \cdot \left( t - 1 \right) \big\} } \bigg\} \qquad \left( \text{by \cref{lem:auxiliary_lemma_1}} \right) \\
    &= b_t.
\end{align}
Using similar calculation to \cref{eq:failure_probability_softmax_natural_pg_special_on_policy_stochastic_gradient_intermediate_18}, we have,
\begin{align}
\label{eq:failure_probability_softmax_gnpg_special_on_policy_stochastic_gradient_intermediate_8} 
\MoveEqLeft
    \prod_{t=1}^{\infty}{ \pi_{\theta_t}(a) } \ge \prod_{t=1}^{\infty}{ \exp\bigg\{ \frac{ - \sum_{a^\prime \not= a}{ \exp\{ \theta_1(a^\prime) \} } }{ \exp\big\{ \theta_1(a) + \frac{\eta}{\sqrt{2}} \cdot \left( t - 1 \right) \big\} } \bigg\} } \qquad \left( \text{by \cref{eq:failure_probability_softmax_gnpg_special_on_policy_stochastic_gradient_intermediate_7}} \right) \\
    &= \exp\bigg\{ - \frac{  \sum_{a^\prime \not= a}{ \exp\{ \theta_1(a^\prime) \} }}{ \exp\{ \theta_1(a) \} } \cdot \exp\Big\{ \frac{\eta}{\sqrt{2}} \Big\} \cdot \sum_{t=1}^{\infty}{  \frac{ 1 }{ \exp\big\{ \frac{\eta}{\sqrt{2}} \cdot t \big\} } } \bigg\} \\
    &\ge \exp\bigg\{ - \frac{  \sum_{a^\prime \not= a}{ \exp\{ \theta_1(a^\prime) \} }}{ \exp\{ \theta_1(a) \} } \cdot \exp\Big\{ \frac{\eta}{\sqrt{2}} \Big\} \cdot \int_{t=0}^{\infty}{  \frac{ 1 }{ \exp\big\{ \frac{\eta}{\sqrt{2}} \cdot t \big\} } dt } \bigg\} \\
    &= \exp\bigg\{ - \frac{  \sum_{a^\prime \not= a}{ \exp\{ \theta_1(a^\prime) \} }}{ \exp\{ \theta_1(a) \} } \cdot \frac{\sqrt{2} \cdot \exp\big\{ \frac{\eta}{\sqrt{2}} \big\} }{\eta} \bigg\} \\
    &\in \Omega(1),
\end{align}
where the last line is due to, $\exp\left\{ \theta_1(a) \right\} \in \Theta(1)$ for all $a \in [K]$, and $\eta \in \Theta(1)$. With \cref{eq:failure_probability_softmax_gnpg_special_on_policy_stochastic_gradient_intermediate_8}, we have ``the probability of sampling action $a$ forever using on-policy sampling $a_t \sim \pi_{\theta_t}(\cdot)$'' is lower bounded by a constant of $\Omega(1)$. Therefore, for all $a \in [K]$, with positive probability $\Omega(1)$, $\pi_{\theta_t}(a) \to 1$, as $t \to \infty$.
\end{proof}

\section{Proofs for Committal Rate (\texorpdfstring{\cref{sec:committal_rate}}{Lg}) }

\textbf{\cref{thm:committal_rate_main_theorem}} (Committal rate main theorem)\textbf{.} 
Consider a policy optimization method $\gA$, together with $r\in (0,1]^K$ and an initial parameter vector $\theta_1\in \sR^K$.
Then,
\begin{align}
\max_{a:r(a)<r(a^*),\pi_{\theta_1}(a)>0} \kappa(\gA, a) \le 1
\end{align}
is a necessary condition for ensuring the almost sure convergence of the policies obtained using $\gA$ and online sampling to the global optimum starting from $\theta_1$.
\begin{proof}
It suffices to prove that if $\kappa(\gA, a) > 1$ happens for a suboptimal action $a \in [K]$ while $\pi_{\theta_1}(a)>0$, then if we let $\{\theta_t\}_{t\ge 1}$ be the parameter sequence obtained by using $\gA$ with online sampling, i.e., when $a_t \sim \pi_{\theta_t}(\cdot)$,  then 
the event $\gE = \{ a_t=a \text{ holds for all  } t\ge 1\}$
happens with positive probability, and it also holds that
$\pi_{\theta_t}$ converges to a sub-optimal deterministic policy with positive probability.

For convenience, denote $\alpha \coloneqq \kappa(\gA, a)$. Define the history of actions for the first $t$ iterations,
\begin{align}
    \gH_t \coloneqq \left( a_1, a_2, \cdots, a_t \right).
\end{align}
Given the historical iterations, sampled actions and rewards, the next iteration is a deterministic result of the algorithm,
\begin{align}
    \theta_t = \gA\left( \theta_1, a_1, r(a_1), \theta_2, a_2, r(a_2), \cdots, \theta_{t-1}, a_{t-1}, r(a_{t-1}) \right).
\end{align}
Let $\probability$ denote the probability measure that over the probability space $(\Omega,\gF)$ that holds all of our random variables.
By construction, 
$\probability$ satisfies that
for all $a$ and $t \ge 1$, 
\begin{align}
\label{eq:committal_rate_main_theorem_intermediate_1}
    \probability{ \left( a_t = a \ | \ \gH_{t-1} \right) } = \pi_{\theta_t}(a)  \qquad \text{ $\probability$ almost surely.} 
\end{align}
Define the following event, for all $t \ge 1$,
\begin{align}
    \gE_t \coloneqq \left\{ a_s = a, \text{ for all } 1 \le s \le t \right\}.
\end{align}
We have $\gE_t \supseteq \gE_{t+1}$, and $\gE_t$ approaches the limit event,
\begin{align}
    \gE \coloneqq \left\{ a_t = a, \text{ for all } t \ge 1 \right\}.
\end{align}
We have $\probability{ ( \gE_t ) }$ is monotonically decreasing and lower bounded by zero. According to monotone convergence theorem, 
\begin{align}
\label{eq:committal_rate_main_theorem_intermediate_2}
    \probability{ ( \gE ) } = \lim_{t \to \infty}{ \probability{ ( \gE_t ) } }.
\end{align}
Next, we prove by induction on $t$ the following holds
\begin{align}
    \probability{ ( \gE_t ) } &= \probability{ ( a_t = a  \ | \ \gE_{t-1} ) } \cdot  \probability{ ( \gE_{t-1} ) } \\
    &= \prod_{s=1}^{t}{ \pi_{\ttheta_s}(a) },
\end{align}
where $\ttheta_1 = \theta_1$, and,
\begin{align}
\label{eq:committal_rate_main_theorem_intermediate_3}
    \ttheta_t = \gA\big( \theta_1, \underbrace{a, r(a)}_{s=1}, \cdots, \underbrace{a, r(a)}_{s=t-1} \big),
\end{align}
which means $a$ is used for the first $t-1$ iterations.

First, by definition of $\ttheta_1$, we have,
\begin{align}
    \probability{ ( \gE_1 ) } = \pi_{\theta_1}(a) = \pi_{\ttheta_1}(a),
\end{align}
where the first equation is from \cref{eq:committal_rate_main_theorem_intermediate_1}. Suppose the equation holds up to $t-1$. We have,
\begin{align}
\label{eq:committal_rate_main_theorem_intermediate_4}
\MoveEqLeft
    \probability{ ( \gE_t ) } = \expectation{ \left[ \probability{ ( a_t = a, \cdots, a_1 = a \ | \ \gH_{t-1} ) } \right] } \qquad \left( \text{by the tower rule} \right) \\
    &= \expectation{  \left[ \sI\left\{ a_{t-1} = a, \cdots, a_1 = a \right\} \cdot \probability{ ( a_t = a \ | \ \gH_{t-1} ) } \right] } \qquad \left( \{ a_1, \cdots, a_{t-1} \} \text{ is deterministic given } \gH_{t-1} \right) \\
    &= \expectation{  \left[ \sI\left\{ a_{t-1} = a, \cdots, a_1 = a \right\} \cdot \pi_{\theta_t}(a) \right] } \qquad \left( \text{by \cref{eq:committal_rate_main_theorem_intermediate_1}} \right)  \\
    &= \expectation{  \left[ \sI\left\{ a_{t-1} = a, \cdots, a_1 = a \right\} \cdot \pi_{\ttheta_t}(a) \right] } \qquad \left( \text{by \cref{eq:committal_rate_main_theorem_intermediate_3}} \right) \\
    &= \pi_{\ttheta_t}(a) \cdot \probability{ ( \gE_{t-1} ) } \qquad \left( \text{from calculating the expectation} \right) \\
    &= \prod_{s=1}^{t}{ \pi_{\ttheta_s}(a) }. \qquad \left( \text{by induction hypothesis} \right)
\end{align}
Next, we show that $\prod_{t=1}^{\infty}{ \pi_{\ttheta_t}(a) } > 0$,
where $\pi_{\ttheta_t}(a)$ is the probability at $t$th iteration given $\gA$ is used when in the first $t-1$ iterations action $a$ is used. This is the sequence used in the definition of committal rate $\kappa$.
Further, for simplicity, assume that in the definition of $\kappa$, the supremum is achieved. It follows that there exists a universal constant $C > 0$
such that on $\gE$,
for all $t \ge 1$, 
\begin{align}
\label{eq:committal_rate_main_theorem_intermediate_5}
    1 - \pi_{\ttheta_t}(a) &= t^\alpha \cdot  \left[ 1 - \pi_{\ttheta_t}(a) \right] \cdot \frac{1}{t^\alpha} \\
    &\le \frac{ C }{ t^\alpha }. \qquad \left( \text{by \cref{def:committal_rate}} \right)
\end{align}
Let $u_t \coloneqq 1 - \pi_{\ttheta_t}(a) \in (0, 1)$ for all $t \ge 1$. We have,
\begin{align}
\label{eq:committal_rate_main_theorem_intermediate_6}
    \sum_{t=1}^{\infty}{ u_t } &\le \sum_{t=1}^{\infty}{ \frac{ C }{ t^\alpha } } \qquad \left( \text{by \cref{eq:committal_rate_main_theorem_intermediate_5}} \right) \\
    &< \infty. \qquad \left( \text{by \cref{lem:positive_infinite_series}, }\alpha \coloneqq \kappa(\gA, a) > 1 \right)
\end{align}
Therefore we have,
\begin{align}
\label{eq:committal_rate_main_theorem_intermediate_7}
    \prod_{t=1}^{\infty}{ \pi_{\ttheta_t}(a) } &= \prod_{t=1}^{\infty}{ \left( 1 - u_t \right) } \\
    &> 0. \qquad \left( \text{by \cref{lem:positive_infinite_product_2,eq:committal_rate_main_theorem_intermediate_6}} \right)
\end{align}
Hence, we have,
\begin{align}
    \probability(\gE) &= \lim_{T\to\infty} \probability(\gE_T) \qquad \left( \text{by \cref{eq:committal_rate_main_theorem_intermediate_2}} \right) \\
    &= \lim_{T\to\infty} \prod_{t=1}^{T}{ \pi_{\ttheta_t}(a) } \qquad \left( \text{by \cref{eq:committal_rate_main_theorem_intermediate_4}} \right) \\
    &= \prod_{t=1}^{\infty}{ \pi_{\ttheta_t}(a) }>0, \qquad \left( \text{by \cref{eq:committal_rate_main_theorem_intermediate_7}} \right)
\end{align}
and thus $\pi_{\theta_t}(a) \to 1$ as $t\to\infty$.
\end{proof}

\textbf{\cref{thm:committal_rate_stochastic_npg_gnpg}.} Let \cref{asmp:positive_reward} holds. For the stochastic updates NPG and GNPG from \cref{update_rule:softmax_natural_pg_special_on_policy_stochastic_gradient,update_rule:softmax_gnpg_special_on_policy_stochastic_gradient}  
we obtain $\kappa(\text{NPG}, a) = \infty$ and $\kappa(\text{GNPG}, a) = \infty$ for all $a \in [K]$ respectively.
\begin{proof}
\textbf{First part (NPG).} We first show that $\kappa(\text{NPG}, a) = \infty$ for all $a \in [K]$. According to \cref{def:committal_rate}, let action $a$ be sampled forever after initialization. We have, for stochastic NPG update,
\begin{align}
\label{eq:committal_rate_stochastic_npg_gnpg_intermediate_1}
\MoveEqLeft
    1 - \pi_{\theta_t}(a) = \sum_{a^\prime \not= a}{ \pi_{\theta_t}(a^\prime) } \\
    &\le \frac{ \sum_{a^\prime \not= a}{ \exp\{ \theta_1(a^\prime) \} } }{ \exp\{ \theta_1(a) + \eta \cdot r(a) \cdot \left( t - 1 \right) \} + \sum_{a^\prime \not= a}{ \exp\{ \theta_1(a^\prime) \} } }. \qquad \left( \text{by \cref{eq:failure_probability_softmax_natural_pg_special_on_policy_stochastic_gradient_intermediate_17}} \right)
\end{align}
Since $\exp\{ \theta_1(i) \} \in \Theta(1)$ for all $i \in [K]$, we have, for any finite $\alpha \in ( 0, \infty)$,
\begin{align}
\label{eq:committal_rate_stochastic_npg_gnpg_intermediate_2}
    \lim_{t \to \infty}{ t^\alpha \cdot  \left[ 1 - \pi_{\theta_t}(a) \right] } &\le \lim_{t \to \infty}{ \frac{ t^\alpha \cdot \sum_{a^\prime \not= a}{ \exp\{ \theta_1(a^\prime) \} } }{ \exp\{ \theta_1(a) + \eta \cdot r(a) \cdot \left( t - 1 \right) \} + \sum_{a^\prime \not= a}{ \exp\{ \theta_1(a^\prime) \} } } } \qquad \left( \text{by \cref{eq:committal_rate_stochastic_npg_gnpg_intermediate_1}} \right) \\
    &= \lim_{t \to \infty}{ \frac{ \Theta(t^\alpha) }{ \Theta( \exp\{ \eta \cdot r(a) \cdot \left( t - 1 \right)  \} ) } } = 0,
\end{align}
which means $\kappa(\text{NPG}, a) = \infty$ for all $a \in [K]$.

\textbf{Second part (GNPG).} We next show that $\kappa(\text{GNPG}, a) = \infty$ for all $a \in [K]$. Let action $a$ be sampled forever after initialization. We have, for stochastic GNPG update,
\begin{align}
\MoveEqLeft
    1 - \pi_{\theta_t}(a) = \sum_{a^\prime \not= a}{ \pi_{\theta_t}(a^\prime) } \\
    &\le \frac{ \sum_{a^\prime \not= a}{ \exp\{ \theta_1(a^\prime) \} } }{ \exp\big\{ \theta_1(a) + \frac{\eta}{\sqrt{2}} \cdot \left( t - 1 \right) \big\} + \sum_{a^\prime \not= a}{ \exp\{ \theta_1(a^\prime) \} } }. \qquad \left( \text{by \cref{eq:failure_probability_softmax_gnpg_special_on_policy_stochastic_gradient_intermediate_7}} \right)
\end{align}
Using similar arguments to \cref{eq:committal_rate_stochastic_npg_gnpg_intermediate_2}, we have $\kappa(\text{GNPG}, a) = \infty$ for all $a \in [K]$.
\end{proof}

\textbf{\cref{thm:committal_rate_softmax_pg}.}
Softmax PG obtains
$\kappa(\text{PG}, a) = 1$ for all $a \in [K]$.
\begin{proof}
\textbf{First part.} $\kappa(\textnormal{PG}, a) \ge 1$.

According to \cref{def:committal_rate}, let action $a$ be sampled forever  after initialization. We have, for stochastic PG update,
\begin{align}
\label{eq:committal_rate_softmax_pg_intermediate_1}
\MoveEqLeft
    \left( 1 - \pi_{\theta_{t+1}}(a) \right) - \left( 1 - \pi_{\theta_t}(a) \right) = \pi_{\theta_t}(a) - \pi_{\theta_{t+1}}(a) + \Big\langle \frac{d \pi_{\theta_t}(a)}{d \theta_t}, \theta_{t+1} - \theta_{t} \Big\rangle - \Big\langle \frac{d \pi_{\theta_t}(a)}{d \theta_t}, \theta_{t+1} - \theta_{t} \Big\rangle \\
    &\le \frac{5}{4} \cdot \| \theta_{t+1} - \theta_{t} \|_2^2 - \Big\langle \frac{d \pi_{\theta_t}(a)}{d \theta_t}, \theta_{t+1} - \theta_{t} \Big\rangle \qquad \left( \text{by \cref{eq:upper_bound_softmax_pg_special_true_gradient_intermediate_1}, smoothness} \right) \\
    &= \frac{5 \cdot \eta^2}{4} \cdot \bigg\| \frac{d \pi_{\theta_t}^\top \hat{r}_t}{d \theta_t}  \bigg\|_2^2 - \eta \cdot \bigg\langle \frac{d \pi_{\theta_t}(a)}{d \theta_t}, \frac{d \pi_{\theta_t}^\top \hat{r}_t}{d \theta_t} \bigg\rangle \qquad \left( \text{using \cref{update_rule:softmax_pg_special_on_policy_stochastic_gradient}} \right) \\
    &= \frac{5 \cdot \eta^2}{4} \cdot \bigg( \sum_{a^\prime \not= a}{ \pi_{\theta_t}(a^\prime)^2 \cdot r(a)^2 } + \left( 1 - \pi_{\theta_t}(a) \right)^2 \cdot r(a)^2 \bigg) - \eta \cdot \bigg\langle \frac{d \pi_{\theta_t}(a)}{d \theta_t}, \frac{d \pi_{\theta_t}^\top \hat{r}_t}{d \theta_t} \bigg\rangle \quad \left( \text{by \cref{eq:failure_probability_softmax_gnpg_special_on_policy_stochastic_gradient_intermediate_1,eq:failure_probability_softmax_gnpg_special_on_policy_stochastic_gradient_intermediate_2}} \right) \\
    &\le \frac{5 \cdot \eta^2}{2} \cdot \left( 1 - \pi_{\theta_t}(a) \right)^2 \cdot r(a)^2 - \eta \cdot \bigg\langle \frac{d \pi_{\theta_t}(a)}{d \theta_t}, \frac{d \pi_{\theta_t}^\top \hat{r}_t}{d \theta_t} \bigg\rangle \qquad \left( \left\| x \right\|_2 \le \left\| x \right\|_1 \right) \\
    &= \frac{5 \cdot \eta^2}{2} \cdot \left( 1 - \pi_{\theta_t}(a) \right)^2 \cdot r(a)^2 - \eta \cdot \pi_{\theta_t}(a) \cdot r(a) \cdot \bigg( \sum_{a^\prime \not= a}{ \pi_{\theta_t}(a^\prime)^2 } + \left( 1 - \pi_{\theta_t}(a) \right)^2 \bigg) \quad \left( \text{see below} \right) \\
    &\le \frac{5 \cdot \eta^2}{2} \cdot \left( 1 - \pi_{\theta_t}(a) \right)^2 \cdot r(a)^2 - \eta \cdot \pi_{\theta_t}(a) \cdot r(a) \cdot \left( 1 - \pi_{\theta_t}(a) \right)^2,
\end{align}
where the first inequality is because $\pi_{\theta}(a) = \pi_{\theta}^\top e_a$, where $e_a \in \{ 0, 1 \}^K$ with $e_a(a) = 1$ and $e_a(a^\prime) = 0$ for all $a^\prime \not= a$, and the second last equality is because of
\begin{align}
    \frac{d \pi_{\theta_t}(a)}{d \theta_t(i)} = \begin{cases}
		\pi_{\theta_t}(i) \cdot \left( 1 - \pi_{\theta_t}(i) \right), & \text{if } i = a, \\
		- \pi_{\theta_t}(i) \cdot \pi_{\theta_t}(a). & \text{otherwise}
	\end{cases}
\end{align}
Using $\eta = \frac{ \pi_{\theta_t}(a)}{5 \cdot r(a)} $, for all $t \ge 1$, we have,
\begin{align}
\label{eq:committal_rate_softmax_pg_intermediate_2}
    \left( 1 - \pi_{\theta_{t+1}}(a) \right) - \left( 1 - \pi_{\theta_t}(a) \right) \le - \frac{1}{10} \cdot \pi_{\theta_t}(a)^2 \cdot \left( 1 - \pi_{\theta_t}(a) \right)^2,
\end{align}
which means $\pi_{\theta_{t+1}}(a) \ge \pi_{\theta_t}(a)$ for all $t \ge 1$. Therefore, we have $\eta \ge \frac{ \pi_{\theta_1}(a)}{5 \cdot r(a)} \in \Theta(1)$ and,
\begin{align}
\label{eq:committal_rate_softmax_pg_intermediate_3}
    \left( 1 - \pi_{\theta_{t+1}}(a) \right) - \left( 1 - \pi_{\theta_t}(a) \right) \le - \frac{1}{10} \cdot \pi_{\theta_1}(a)^2 \cdot \left( 1 - \pi_{\theta_t}(a) \right)^2.
\end{align}
Then we have,
\begin{align}
\label{eq:committal_rate_softmax_pg_intermediate_4}
\MoveEqLeft
    \frac{1}{ 1 - \pi_{\theta_t}(a) } = \frac{1}{ 1 - \pi_{\theta_1}(a)} + \sum_{s=1}^{t-1}{ \left[ \frac{1}{1 - \pi_{\theta_{s+1}}(a)} - \frac{1}{1 - \pi_{\theta_s}(a)} \right] } \\
    &= \frac{1}{ 1 - \pi_{\theta_1}(a)} + \sum_{s=1}^{t-1}{ \frac{1}{ \left( 1 - \pi_{\theta_{s+1}}(a) \right) \cdot \left( 1 - \pi_{\theta_{s}}(a) \right) } \cdot \left[ \left( 1 - \pi_{\theta_{s}}(a) \right) - \left( 1 - \pi_{\theta_{s+1}}(a) \right) \right] } \\
    &\ge \frac{1}{ 1 - \pi_{\theta_1}(a)} + \sum_{s=1}^{t-1}{ \frac{1}{ \left( 1 - \pi_{\theta_{s+1}}(a) \right) \cdot \left( 1 - \pi_{\theta_{s}}(a) \right) } \cdot \frac{ \pi_{\theta_1}(a)^2 }{10} \cdot \left( 1 - \pi_{\theta_s}(a) \right)^2 } \qquad \left( \text{by \cref{eq:committal_rate_softmax_pg_intermediate_3}} \right) \\
    &\ge \frac{1}{ 1 - \pi_{\theta_1}(a)} + \frac{ \pi_{\theta_1}(a)^2 }{10} \cdot (t-1) \qquad \left( \pi_{\theta_{t+1}}(a) \ge \pi_{\theta_t}(a) \right) \\
    &\ge \frac{\pi_{\theta_1}(a)^2 }{10} \cdot t, \qquad \left( \frac{1}{ 1 - \pi_{\theta_1}(a)} \ge 1 \ge \frac{\pi_{\theta_1}(a)^2 }{10} \right)
\end{align}
which implies for all $t \ge 1$,
\begin{align}
    t \cdot \left[ 1 - \pi_{\theta_t}(a) \right] &\le t \cdot \left[ \frac{10}{ \pi_{\theta_1}(a)^2 } \cdot \frac{1}{t} \right] \qquad \left( \text{by \cref{eq:committal_rate_softmax_pg_intermediate_4}} \right)  \\
    &= \frac{10}{ \pi_{\theta_1}(a)^2 },
\end{align}
which means $\kappa(\text{PG}, a) \ge 1$ for all $a \in [K]$ according to \cref{def:committal_rate}.

\textbf{Second part.} $\kappa(\textnormal{PG}, a) \le 1$.

Let action $a$ be sampled forever after initialization. We show that $1 - \pi_{\theta_t}(a)$ cannot decrease faster than $O(1/t)$. Similar to \cref{eq:committal_rate_softmax_pg_intermediate_1}, we have, 
\begin{align}
\label{eq:committal_rate_softmax_pg_intermediate_5}
\MoveEqLeft
     \left( 1 - \pi_{\theta_t}(a) \right) - \left( 1 - \pi_{\theta_{t+1}}(a) \right) =  \pi_{\theta_{t+1}}(a) - \pi_{\theta_t}(a) - \Big\langle \frac{d \pi_{\theta_t}(a)}{d \theta_t}, \theta_{t+1} - \theta_{t} \Big\rangle + \Big\langle \frac{d \pi_{\theta_t}(a)}{d \theta_t}, \theta_{t+1} - \theta_{t} \Big\rangle \\
    &\le \frac{5}{4} \cdot \| \theta_{t+1} - \theta_{t} \|_2^2 + \Big\langle \frac{d \pi_{\theta_t}(a)}{d \theta_t}, \theta_{t+1} - \theta_{t} \Big\rangle \qquad \left( \text{by \cref{eq:upper_bound_softmax_pg_special_true_gradient_intermediate_1}, smoothness} \right) \\
    &= \frac{5 \cdot \eta^2}{4} \cdot \bigg\| \frac{d \pi_{\theta_t}^\top \hat{r}_t}{d \theta_t}  \bigg\|_2^2 + \eta \cdot \bigg\langle \frac{d \pi_{\theta_t}(a)}{d \theta_t}, \frac{d \pi_{\theta_t}^\top \hat{r}_t}{d \theta_t} \bigg\rangle \qquad \left( \text{using \cref{update_rule:softmax_pg_special_on_policy_stochastic_gradient}} \right) \\
    &\le \frac{5 \cdot \eta^2}{2} \cdot \left( 1 - \pi_{\theta_t}(a) \right)^2 \cdot r(a)^2 + \eta \cdot \bigg\langle \frac{d \pi_{\theta_t}(a)}{d \theta_t}, \frac{d \pi_{\theta_t}^\top \hat{r}_t}{d \theta_t} \bigg\rangle \qquad \left( \text{by \cref{eq:failure_probability_softmax_gnpg_special_on_policy_stochastic_gradient_intermediate_1,eq:failure_probability_softmax_gnpg_special_on_policy_stochastic_gradient_intermediate_2}} \right) \\
    &= \frac{5 \cdot \eta^2}{2} \cdot \left( 1 - \pi_{\theta_t}(a) \right)^2 \cdot r(a)^2 + \eta \cdot \pi_{\theta_t}(a) \cdot r(a) \cdot \bigg( \sum_{a^\prime \not= a}{ \pi_{\theta_t}(a^\prime)^2 } + \left( 1 - \pi_{\theta_t}(a) \right)^2 \bigg) \\
    &\le \frac{5 \cdot \eta^2}{2} \cdot \left( 1 - \pi_{\theta_t}(a) \right)^2 \cdot r(a)^2 + 2 \cdot \eta \cdot \pi_{\theta_t}(a) \cdot r(a) \cdot \left( 1 - \pi_{\theta_t}(a) \right)^2 \qquad \left( \left\| x \right\|_2 \le \left\| x \right\|_1 \right) \\
    &\le \frac{9}{2} \cdot \left( 1 - \pi_{\theta_t}(a) \right)^2 \cdot r(a),
\end{align}
where the last inequality is due to $\pi_{\theta_t}(a) \in (0, 1)$, $r(a) \in (0, 1]$, and $\eta \in (0, 1]$. Denote $\delta(\theta_t) \coloneqq 1 - \pi_{\theta_t}(a)$. We have, for all $t \ge 1$,
\begin{align}
    \delta(\theta_t) - \delta(\theta_{t+1}) \le \frac{9}{2} \cdot r(a) \cdot \delta(\theta_t)^2, 
\end{align}
which is similar to \cref{eq:lower_bound_softmax_pg_special_true_gradient_intermediate_2}. Therefore, using similar calculations in the proofs for \cref{prop:lower_bound_softmax_pg_special_true_gradient}, we have, for all large enough $t \ge 1$,
\begin{align}
    t \cdot \left[ 1 - \pi_{\theta_t}(a) \right] &\ge t \cdot \left[ \frac{1}{6 \cdot r(a)} \cdot \frac{1}{ t } \right] \\
    &= \frac{1}{6 \cdot r(a)},
\end{align}
which means $\kappa(\text{PG}, a) \le 1$ for all $a \in [K]$ according to \cref{def:committal_rate}.
\end{proof}

\textbf{\cref{thm:almost_sure_global_convergence_softmax_natural_pg_special_on_policy_stochastic_gradient_oracle_baseline}.}
Using \cref{update_rule:softmax_natural_pg_special_on_policy_stochastic_gradient_oracle_baseline}, $\left( \pi^* - \pi_{\theta_t} \right)^\top r \to 0$ as $t \to \infty$ with probability $1$.
\begin{proof}
Consider the sequence $\left\{ \pi_{\theta_t}(a^*) \right\}_{t \ge 1}$ produced by \cref{update_rule:softmax_natural_pg_special_on_policy_stochastic_gradient_oracle_baseline} using on-policy sampling $a_t \sim \pi_{\theta_t}(\cdot)$. We show that $\pi_{\theta_t}(a^*) \to 1$ as $t \to \infty$ with probability $1$.

First, for convenience, we duplicate \cref{update_rule:softmax_natural_pg_special_on_policy_stochastic_gradient_oracle_baseline} here.

\textbf{\cref{update_rule:softmax_natural_pg_special_on_policy_stochastic_gradient_oracle_baseline}} \text{(NPG with oracle baseline)}\textbf{.}
$\theta_{t+1} \gets \theta_{t} + \eta \cdot \big( \hat{r}_t - \hat{b}_t \big)$, where $\hat{b}_t(a) = \left( \frac{ \sI\left\{ a_t = a \right\} }{ \pi_{\theta_t}(a) } - 1 \right) \cdot b$ for all $a \in [K]$, and $b \in (r(a^*) - \Delta, r(a^*)) $.

Note that \cref{update_rule:softmax_natural_pg_special_on_policy_stochastic_gradient_oracle_baseline} is equivalent to the following update,
\begin{align}
\label{eq:almost_sure_global_convergence_softmax_natural_pg_special_on_policy_stochastic_gradient_oracle_baseline_intermediate_1}
    \theta_{t+1}(a) = \begin{cases}
		\theta_t(a) + \frac{ \eta }{ \pi_{\theta_t}(a) } \cdot \left( r(a) - b \right), & \text{if } a = a_t, \\
		\theta_{t}(a), & \text{otherwise}
	\end{cases}
\end{align}
Next, we show that $\pi_{\theta_{t+1}}(a^*) \ge \pi_{\theta_t}(a^*)$ using on-policy sampling $a_t \sim \pi_{\theta_t}(\cdot)$. There are two cases.

Case (a): If $a_t = a^*$, then we have,
\begin{align}
\label{eq:almost_sure_global_convergence_softmax_natural_pg_special_on_policy_stochastic_gradient_oracle_baseline_intermediate_2}
    \theta_{t+1}(a^*) &= \theta_t(a^*) + \frac{ \eta }{ \pi_{\theta_t}(a^*) } \cdot \left( r(a^*) - b \right) \qquad \left( \text{by \cref{eq:almost_sure_global_convergence_softmax_natural_pg_special_on_policy_stochastic_gradient_oracle_baseline_intermediate_1}} \right) \\
    &> \theta_t(a^*), \qquad \left( r(a^*) > b \right)
\end{align}
while $\theta_{t+1}(a) = \theta_{t}(a)$ for all sub-optimal actions $a \not= a^*$. Then we have,
\begin{align}
\label{eq:almost_sure_global_convergence_softmax_natural_pg_special_on_policy_stochastic_gradient_oracle_baseline_intermediate_3}
\MoveEqLeft
    \pi_{\theta_{t+1}}(a^*) = \frac{ \exp\{ \theta_{t+1}(a^*) \}}{ \exp\{ \theta_{t+1}(a^*) \} + \sum_{a \not= a^*}{ \exp\{ \theta_{t+1}(a) \} } } \\
    &> \frac{ \exp\{ \theta_{t}(a^*) \}}{ \exp\{ \theta_{t}(a^*) \} + \sum_{a \not= a^*}{ \exp\{ \theta_{t+1}(a) \} } } \qquad \left( \text{by \cref{eq:almost_sure_global_convergence_softmax_natural_pg_special_on_policy_stochastic_gradient_oracle_baseline_intermediate_2}} \right) \\
    &= \frac{ \exp\{ \theta_{t}(a^*) \}}{ \exp\{ \theta_{t}(a^*) \} + \sum_{a \not= a^*}{ \exp\{ \theta_{t}(a) \} } } \qquad \left( \theta_{t+1}(a) = \theta_{t}(a), \text{ for all } a \not= a^* \right) \\
    &= \pi_{\theta_{t}}(a^*).
\end{align}

Case (b): If $a_t = a \not= a^*$, then we have,
\begin{align}
\label{eq:almost_sure_global_convergence_softmax_natural_pg_special_on_policy_stochastic_gradient_oracle_baseline_intermediate_4}
    \theta_{t+1}(a) &= \theta_t(a) + \frac{ \eta }{ \pi_{\theta_t}(a) } \cdot \left( r(a) - b \right) \qquad \left( \text{by \cref{eq:almost_sure_global_convergence_softmax_natural_pg_special_on_policy_stochastic_gradient_oracle_baseline_intermediate_1}} \right) \\
    &< \theta_t(a), \qquad \left( r(a) \le r(a^*) - \Delta < b \right)
\end{align}
where $\Delta = r(a^*) - \max_{a \not= a^*}{ r(a) } > 0$ is the reward gap. Also $\theta_{t+1}(a^\prime) = \theta_{t}(a^\prime)$ for all the other actions $a^\prime \not= a$. Then we have,
\begin{align}
\label{eq:almost_sure_global_convergence_softmax_natural_pg_special_on_policy_stochastic_gradient_oracle_baseline_intermediate_5}
\MoveEqLeft
    \pi_{\theta_{t+1}}(a^*) = \frac{ \exp\{ \theta_{t+1}(a^*) \}}{ \exp\{ \theta_{t+1}(a) \} + \sum_{a^\prime \not= a}{ \exp\{ \theta_{t+1}(a^\prime) \} } } \\
    &> \frac{ \exp\{ \theta_{t+1}(a^*) \}}{ \exp\{ \theta_{t}(a) \} + \sum_{a^\prime  \not= a}{ \exp\{ \theta_{t+1}(a^\prime ) \} } } \qquad \left( \text{by \cref{eq:almost_sure_global_convergence_softmax_natural_pg_special_on_policy_stochastic_gradient_oracle_baseline_intermediate_4}} \right) \\
    &= \frac{ \exp\{ \theta_{t}(a^*) \}}{ \exp\{ \theta_{t}(a) \} + \sum_{a^\prime  \not= a}{ \exp\{ \theta_{t}(a^\prime ) \} } } \qquad \left( \theta_{t+1}(a^\prime) = \theta_{t}(a^\prime), \text{ for all } a^\prime \not= a \right) \\
    &= \pi_{\theta_{t}}(a^*).
\end{align}

Therefore, we have $\pi_{\theta_{t+1}}(a^*) \ge \pi_{\theta_t}(a^*)$, for all $t \ge 1$. Note that $\pi_{\theta_t}(a^*) \le 1$. According to monotone convergence theorem, we have $\pi_{\theta_{t+1}}(a^*)$ approaches to some finite value as $t \to \infty$.

Suppose $\pi_{\theta_{t}}(a^*) \to \pi_{\theta_{\infty}}(a^*)$ as $t \to \infty$. We show that $\pi_{\theta_{\infty}}(a^*) = 1$ by contradiction. Suppose $\pi_{\theta_{\infty}}(a^*) < 1$. Then at the convergent point, according to \cref{eq:almost_sure_global_convergence_softmax_natural_pg_special_on_policy_stochastic_gradient_oracle_baseline_intermediate_3,eq:almost_sure_global_convergence_softmax_natural_pg_special_on_policy_stochastic_gradient_oracle_baseline_intermediate_5}, we can further improve the probability of $a^*$ by online sampling and updating once, which is a contradiction with convergence.

Thus we have $\pi_{\theta_{t}}(a^*) \to 1$ as $t \to \infty$ with probability $1$, which implies that $\left( \pi^* - \pi_{\theta_t} \right)^\top r \to 0$ as $t \to \infty$ with probability $1$.
\end{proof}

The Stochastic Approximation Markov Bandit Algorithm (SAMBA) \citep{denisov2020regret} algorithm is mentioned in \cref{sec:geometry_convergence_tradeoff,fig:iteration_behaviours}.
\begin{update_rule}[SAMBA]
\label{update_rule:samba}
At iteration $t \ge 1$, denote the greedy action $\bar{a}_t \coloneqq \argmax_{a \in [K]}{ \pi_t(a) } $. Sample action $a_t \sim \pi_t(\cdot)$. \textit{(i)} If $a_t = \bar{a}_t$, then perform update $\pi_{t+1}(a^\prime) \gets \pi_{t}(a^\prime) - \eta \cdot \pi_t(a^\prime)^2 \cdot \frac{r(a_t)}{\pi_t(a_t)}$ for all non-greedy action $a^\prime \not= a_t$; \textit{(ii)} If $a_t \not= \bar{a}_t$, then perform update $\pi_{t+1}(a_t) \gets \pi_{t}(a_t) + \eta \cdot \pi_t(a_t)^2 \cdot \frac{r(a_t)}{\pi_t(a_t)}$. After doing \textit{(i)} or \textit{(ii)}, calculate $\pi_{t+1}(\bar{a}_t) = 1 - \sum_{a^\prime \not= \bar{a}_t}{ \pi_{t+1}(a^\prime) }$.
\end{update_rule}
The SAMBA algorithm does not maintain parameters $\theta$, and the last step $\pi_{t+1}(\bar{a}_t) = 1 - \sum_{a^\prime \not= \bar{a}_t}{ \pi_{t+1}(a^\prime) }$ in \cref{update_rule:samba} is a necessary projection to the probability simplex, such that $\pi_t$ is a valid probability distribution over $[K]$. As shown in \citep{denisov2020regret}, if the learning rate has the knowledge of the optimal action's reward and reward gap, i.e.,
\begin{align}
    \eta < \frac{\Delta}{ r(a^*) - \Delta },
\end{align}
then \cref{update_rule:samba} converges to $\pi^*$ almost surely with a $O(1/t)$ rate, i.e.,
\begin{align}
    \left( \pi^* - \pi_t \right)^\top r \le C / t.
\end{align}
We calculate the committal rate of SAMBA.
\begin{proposition}
\label{thm:committal_rate_samba}
For SAMBA from \cref{update_rule:samba}, we have $\kappa(\text{SAMBA}, a) = 1$ for all $a \in [K]$.
\end{proposition}
\begin{proof}
\textbf{First part.} $\kappa(\textnormal{SAMBA}, a) \ge 1$.

According to \cref{def:committal_rate}, let action $a$ be the greedy action and be sampled forever. According to \textit{(i)} in \cref{update_rule:samba}, we have, for all $a^\prime \not= a$,
\begin{align}
\label{eq:committal_rate_samba_intermediate_1}
    \pi_{t+1}(a^\prime) &= \pi_{t}(a^\prime) - \eta \cdot \pi_t(a^\prime)^2 \cdot \frac{r(a_t)}{\pi_t(a_t)} \\
    &= \pi_{t}(a^\prime) - \eta \cdot \pi_t(a^\prime)^2 \cdot \frac{r(a)}{\pi_t(a)} \qquad \left( a_t = a \text{ by fixed sampling} \right) \\
    &\le \pi_{t}(a^\prime) - \eta \cdot \pi_t(a^\prime)^2 \cdot r(a). \qquad \left( \pi_t(a) \in (0, 1) \right)
\end{align}
Using similar calculations in \cref{eq:upper_bound_softmax_pg_special_true_gradient_intermediate_2}, we have, for all $a^\prime \not= a$,
\begin{align}
\label{eq:committal_rate_samba_intermediate_2}
\MoveEqLeft
    \frac{1}{ \pi_{t}(a^\prime) } = \frac{1}{\pi_{1}(a^\prime)} + \sum_{s=1}^{t-1}{ \left[ \frac{1}{\pi_{s+1}(a^\prime)} - \frac{1}{\pi_{s}(a^\prime)} \right] } \\
    &= \frac{1}{\pi_{1}(a^\prime)} + \sum_{s=1}^{t-1}{ \frac{1}{\pi_{s+1}(a^\prime) \cdot \pi_{s}(a^\prime)} \cdot \left( \pi_{s}(a^\prime)) - \pi_{s+1}(a^\prime) \right) } \\
    &\ge \frac{1}{\pi_{1}(a^\prime)} + \sum_{s=1}^{t-1}{ \frac{1}{\pi_{s+1}(a^\prime) \cdot \pi_{s}(a^\prime)} \cdot \eta \cdot \pi_s(a^\prime)^2 \cdot r(a) } \qquad \left( \text{by \cref{eq:committal_rate_samba_intermediate_1}} \right) \\
    &\ge \frac{1}{\pi_{1}(a^\prime)} + \eta \cdot r(a) \cdot (t - 1) \qquad \left( \pi_{t+1}(a^\prime) \le \pi_{t}(a^\prime), \text{ by \cref{eq:committal_rate_samba_intermediate_1}} \right) \\
    &\ge \eta \cdot r(a) \cdot t, \qquad \left( \frac{1}{\pi_{1}(a^\prime)} \ge 1 \ge \eta \cdot r(a) \right)
\end{align}
which implies, for all large enough $t \ge 1$,
\begin{align}
    t \cdot \left[ 1 - \pi_t(a) \right] &= t \cdot \sum_{a^\prime \not= a}{ \pi_t(a^\prime) } \\
    &\le t \cdot \sum_{a^\prime \not= a}{ \frac{1}{\eta \cdot r(a) \cdot t} } \qquad \left( \text{by \cref{eq:committal_rate_samba_intermediate_2}} \right) \\
    &= \sum_{a^\prime \not= a}{ \frac{1}{\eta \cdot r(a) } },
\end{align}
which means $\kappa(\text{SAMBA}, a) \ge 1$ for all $a \in [K]$ according to \cref{def:committal_rate}.

\textbf{Second part.} $\kappa(\textnormal{SAMBA}, a) \le 1$. 

Let action $a$ be the greedy action and be sampled forever. According to \textit{(i)} in \cref{update_rule:samba}, we have, for all $a^\prime \not= a$,
\begin{align}
\label{eq:committal_rate_samba_intermediate_3}
    \pi_{t+1}(a^\prime) &= \pi_{t}(a^\prime) - \eta \cdot \pi_t(a^\prime)^2 \cdot \frac{r(a_t)}{\pi_t(a_t)} \\
    &= \pi_{t}(a^\prime) - \eta \cdot \pi_t(a^\prime)^2 \cdot \frac{r(a)}{\pi_t(a)} \qquad \left( a_t = a \text{ by fixed sampling} \right) \\
    &\ge \pi_{t}(a^\prime) - \eta \cdot K \cdot  \pi_t(a^\prime)^2 \cdot r(a), \qquad \left( \pi_t(a) \ge 1/K, \ a \text{ is greedy action} \right)
\end{align}
which is similar to \cref{eq:lower_bound_softmax_pg_special_true_gradient_intermediate_2}. Therefore, using similar calculations in the proofs for \cref{prop:lower_bound_softmax_pg_special_true_gradient}, we have, for all large enough $t \ge 1$, we have,
\begin{align}
\label{eq:committal_rate_samba_intermediate_4}
    \frac{ \pi_{t+1}(a^\prime)}{ \pi_{t}(a^\prime) } \ge \frac{1}{2}.
\end{align}
Denote
\begin{align}
    t_0 \coloneqq \min\Big\{t \ge 1: \frac{ \pi_{t+1}(a^\prime)}{ \pi_{t}(a^\prime) } \ge \frac{1}{2} , \text{ for all } s \ge t \Big\}.
\end{align}
On the other hand, since $t_0 \in O(1)$, we have, for all $t < t_0$,
\begin{align}
\label{eq:committal_rate_samba_intermediate_5}
    \pi_{t+1}(a^\prime) \ge c_0 > 0.
\end{align}
Next, we have, for all $t \ge t_0$,
\begin{align}
\label{eq:committal_rate_samba_intermediate_6}
\MoveEqLeft
    \frac{1}{ \pi_{t}(a^\prime) } = \frac{1}{\pi_{1}(a^\prime)} + \sum_{s=1}^{t_0 -1}{ \frac{1}{\pi_{s+1}(a^\prime) } \cdot \left( 1 - \frac{ \pi_{s+1}(a^\prime)}{ \pi_{s}(a^\prime) } \right) } + \sum_{s=t_0}^{t -1}{ \frac{1}{\pi_{s+1}(a^\prime) \cdot \pi_{s}(a^\prime)} \cdot \left( \pi_{s}(a^\prime)) - \pi_{s+1}(a^\prime) \right) }\\
    &\le \frac{1}{c_0} + \sum_{s=1}^{t_0 -1}{ \frac{1}{c_0} \cdot 1 } + \sum_{s=t_0}^{t-1}{ \frac{1}{\pi_{s+1}(a^\prime) \cdot \pi_{s}(a^\prime)} \cdot \eta \cdot K \cdot \pi_s(a^\prime)^2 \cdot r(a) } \qquad \left( \text{by \cref{eq:committal_rate_samba_intermediate_5,eq:committal_rate_samba_intermediate_3}} \right) \\
    &\le \frac{t_0 }{c_0} + 2 \cdot \eta \cdot K \cdot r(a) \cdot (t - t_0), \qquad \left( \text{by \cref{eq:committal_rate_samba_intermediate_4}} \right)
\end{align}
which implies, for all large enough $t \ge 1$,
\begin{align}
    t \cdot \left[ 1 - \pi_t(a) \right] &= t \cdot \sum_{a^\prime \not= a}{ \pi_t(a^\prime) } \\
    &\ge t \cdot \sum_{a^\prime \not= a}{ \frac{1}{ t_0 /c_0 + 2 \cdot \eta \cdot K \cdot r(a) \cdot (t - t_0) } } \qquad \left( \text{by \cref{eq:committal_rate_samba_intermediate_6}} \right) \\
    &\ge \sum_{a^\prime \not= a}{ \frac{1}{3 \cdot \eta \cdot K \cdot r(a) } }, \qquad \left( t_0 /c_0 \le \eta \cdot K \cdot r(a) \cdot t \right)
\end{align}
which means $\kappa(\text{SAMBA}, a) \le 1$ for all $a \in [K]$ according to \cref{def:committal_rate}.
\end{proof}

\section{Proofs for Geometry-Convergence Trade-off (\texorpdfstring{\cref{sec:geometry_convergence_tradeoff}}{Lg}) }

First, we show that the algorithms we study in this paper, i.e., softmax PG, NPG, and GNPG, are optimality-smart. Recall from the main paper that, a policy optimization method is said to be \emph{optimality-smart} if for any $t\ge 1$,
$\pi_{\tilde \theta_{t}}(a^*)\ge \pi_{\theta_t}(a^*)$ holds 
where $\tilde \theta_t$ is the parameter vector obtained when $a^*$ is chosen in every time step, starting at $\theta_1$, while $\theta_t$ is \emph{any} parameter vector that can be obtained with $t$ updates (regardless of the action sequence chosen), but also starting from $\theta_1$.

\begin{proposition}
\label{prop:softmax_pg_npg_gnpg_optimality_smart}
Softmax PG, NPG, and GNPG are optimality-smart.
\end{proposition}
\begin{proof}
We show that for softmax PG, NPG, and GNPG, if $a_t = a^*$, then $\pi_{\theta_{t+1}}(a^*) \ge \pi_{\theta_{t}}(a^*)$; if $a_t = a \not= a^*$, then $\pi_{\theta_{t+1}}(a^*) \le \pi_{\theta_{t}}(a^*)$. For softmax PG and GNPG the later claim holds when $a^*$ is the dominating action at $t$th iteration, i.e., $\pi_{\theta_t}(a^*) \ge \pi_{\theta_t}(a^\prime)$ for all $a^\prime \not= a^*$. From existing results (\cref{prop:upper_bound_softmax_pg_special_true_gradient,prop:upper_bound_softmax_gnpg_special_true_gradient}) we know that softmax PG and GNPG converge to $\pi^*$ as $t \to \infty$ (using true policy gradients; also holds for using fixed sampling $a_t = a^*$ for all $t \ge 1$), thus we have for all large enough $t \ge 1$, $\pi_{\theta_t}(a^*) \ge \pi_{\theta_t}(a^\prime)$ for all $a^\prime \not= a^*$.

\textbf{First part.} Softmax PG and GNPG are optimality-smart.

If $a_t = a^*$, then we have,
\begin{align}
\label{eq:softmax_pg_npg_gnpg_optimality_smart_intermediate_1}
    \theta_{t+1}(a^*) &= \theta_t(a^*) + \eta \cdot \frac{d \pi_{\theta_t}^\top \hat{r}_t}{d \theta_t(a^*)} \\
    &= \theta_t(a^*) + \eta \cdot \left( 1 - \pi_{\theta_t}(a^*) \right) \cdot r(a^*) \qquad \left( \text{by \cref{eq:failure_probability_softmax_gnpg_special_on_policy_stochastic_gradient_intermediate_1}} \right) \\
    &\ge \theta_t(a^*). \qquad \left( r \in (0, 1]^K \right)
\end{align}
And for any $a \not= a^*$, we have,
\begin{align}
\label{eq:softmax_pg_npg_gnpg_optimality_smart_intermediate_2}
    \theta_{t+1}(a) &= \theta_t(a) + \eta \cdot \frac{d \pi_{\theta_t}^\top \hat{r}_t}{d \theta_t(a)} \\
    &= \theta_t(a) - \eta \cdot \pi_{\theta_t}(a) \cdot r(a^*) \qquad \left( \text{by \cref{eq:failure_probability_softmax_gnpg_special_on_policy_stochastic_gradient_intermediate_2}} \right) \\
    &\le \theta_t(a). \qquad \left( r \in (0, 1]^K \right)
\end{align}
Therefore, we have,
\begin{align}
\MoveEqLeft
\label{eq:softmax_pg_npg_gnpg_optimality_smart_intermediate_3}
    \pi_{\theta_{t+1}}(a^*) = \frac{ \exp\{ \theta_{t+1}(a^*) \}}{ \exp\{ \theta_{t+1}(a^*) \} + \sum_{a \not= a^*}{ \exp\{ \theta_{t+1}(a) \} } } \\
    &\ge \frac{ \exp\{ \theta_{t}(a^*) \}}{ \exp\{ \theta_{t}(a^*) \} + \sum_{a \not= a^*}{ \exp\{ \theta_{t}(a) \} } } \qquad \left( \text{by \cref{eq:softmax_pg_npg_gnpg_optimality_smart_intermediate_1,eq:softmax_pg_npg_gnpg_optimality_smart_intermediate_2}} \right) \\
    &= \pi_{\theta_{t}}(a^*).
\end{align}
On the other hand, given $a_t = a \not= a^*$, we show that if $\pi_{\theta_t}(a^*) \ge \pi_{\theta_t}(a^\prime)$ for all $a^\prime \not= a^*$, then $\pi_{\theta_{t+1}}(a^*) \le \pi_{\theta_{t}}(a^*)$. We have,
\begin{align}
\label{eq:softmax_pg_npg_gnpg_optimality_smart_intermediate_4}
    \theta_{t+1}(a) &= \theta_t(a) + \eta \cdot \left( 1 - \pi_{\theta_t}(a) \right) \cdot r(a) \qquad \left( \text{by \cref{eq:failure_probability_softmax_gnpg_special_on_policy_stochastic_gradient_intermediate_1}} \right) \\
    &\ge \theta_{t}(a) - \eta \cdot \pi_{\theta_t}(a^*) \cdot r(a).
\end{align}
And for any $a^\prime \not= a$, we have,
\begin{align}
\label{eq:softmax_pg_npg_gnpg_optimality_smart_intermediate_5}
    \theta_{t+1}(a^\prime) &= \theta_t(a^\prime) - \eta \cdot \pi_{\theta_t}(a^\prime) \cdot r(a) \qquad \left( \text{by \cref{eq:failure_probability_softmax_gnpg_special_on_policy_stochastic_gradient_intermediate_2}} \right) \\
    &\ge \theta_t(a^\prime) - \eta \cdot \pi_{\theta_t}(a^*) \cdot r(a). \qquad \left( \pi_{\theta_t}(a^*) \ge \pi_{\theta_t}(a^\prime) \right)
\end{align}
Therefore, we have,
\begin{align}
\MoveEqLeft
    \pi_{\theta_{t+1}}(a^*) = \frac{ \exp\{ \theta_{t+1}(a^*) \}}{ \exp\{ \theta_{t+1}(a) \} + \sum_{a^\prime \not= a}{ \exp\{ \theta_{t+1}(a^\prime) \} } } \\
    &\le \frac{ \exp\{ \theta_{t}(a^*) - \eta \cdot \pi_{\theta_t}(a^*) \cdot r(a) \}}{ \exp\{ \theta_{t}(a) - \eta \cdot \pi_{\theta_t}(a^*) \cdot r(a) \} + \sum_{a^\prime  \not= a}{ \exp\{ \theta_t(a^\prime) - \eta \cdot \pi_{\theta_t}(a^*) \cdot r(a) \} } } \qquad \left( \text{by \cref{eq:softmax_pg_npg_gnpg_optimality_smart_intermediate_4,eq:softmax_pg_npg_gnpg_optimality_smart_intermediate_5}} \right) \\
    &= \frac{ \exp\{ \theta_{t}(a^*) \}}{ \exp\{ \theta_{t}(a) \} + \sum_{a^\prime  \not= a}{ \exp\{ \theta_{t}(a^\prime ) \} } } \\
    &= \pi_{\theta_{t}}(a^*).
\end{align}

\textbf{Second part.} NPG is optimality-smart.

If $a_t = a^*$, then we have,
\begin{align}
    \theta_{t+1}(a^*) &= \theta_t(a^*) + \eta \cdot \frac{ r(a^*) }{ \pi_{\theta_t}(a^*) } \\
    &> \theta_t(a^*).
\end{align}
while $\theta_{t+1}(a) = \theta_{t}(a)$ for all sub-optimal actions $a \not= a^*$. Then we have,
\begin{align}
\MoveEqLeft
    \pi_{\theta_{t+1}}(a^*) = \frac{ \exp\{ \theta_{t+1}(a^*) \}}{ \exp\{ \theta_{t+1}(a^*) \} + \sum_{a \not= a^*}{ \exp\{ \theta_{t+1}(a) \} } } \\
    &\ge \frac{ \exp\{ \theta_{t}(a^*) \}}{ \exp\{ \theta_{t}(a^*) \} + \sum_{a \not= a^*}{ \exp\{ \theta_{t}(a) \} } } \\
    &= \pi_{\theta_{t}}(a^*).
\end{align}
If $a_t = a \not= a^*$, then we have,
\begin{align}
    \theta_{t+1}(a) &= \theta_t(a) + \eta \cdot \frac{ r(a) }{ \pi_{\theta_t}(a) }  \\
    &\ge \theta_t(a),
\end{align}
while $\theta_{t+1}(a^\prime) = \theta_{t}(a^\prime)$ for all the other actions $a^\prime \not= a$. Then we have,
\begin{align}
\MoveEqLeft
    \pi_{\theta_{t+1}}(a^*) = \frac{ \exp\{ \theta_{t+1}(a^*) \}}{ \exp\{ \theta_{t+1}(a) \} + \sum_{a^\prime \not= a}{ \exp\{ \theta_{t+1}(a^\prime) \} } } \\
    &\le \frac{ \exp\{ \theta_{t}(a^*) \}}{ \exp\{ \theta_{t}(a) \} + \sum_{a^\prime  \not= a}{ \exp\{ \theta_{t}(a^\prime ) \} } } \\
    &= \pi_{\theta_{t}}(a^*). \qedhere
\end{align}
\end{proof}

\textbf{\cref{thm:committal_rate_optimal_action_special}.}
Let $\gA$ be optimality-smart and pick a bandit instance.
If $\gA$ together with on-policy sampling 
 leads to $\{\theta_t\}_{t\ge 1}$ such that $\{\pi_{\theta_t}\}_{t\ge 1}$ converges to a globally optimal policy at a rate $O(1/t^\alpha)$ with positive probability, for $\alpha > 0$, then $\kappa(\gA, a^*) \ge \alpha$.  
\begin{proof}
Fix an instance $r \in (0, 1]^K$ with a unique optimal action $a^*$.
For any $\theta \in \sR^K$, we have,
\begin{align}
\label{eq:committal_rate_optimal_action_special_intermediate_1}
    \left( \pi^* - \pi_{\theta} \right)^\top r &= \sum_{a \not= a^*}{ \pi_{\theta}(a) \cdot \left( r(a^*) - r(a) \right) } \\
    &\ge \left( 1 - \pi_{\theta}(a^*) \right) \cdot \Delta,
\end{align}
where $\Delta = r(a^*) - \max_{a \not= a^*}{ r(a) } > 0$ is the reward gap. Let $\{\theta_t\}_{t\ge 1}$ 
be the sequence obtained by using $\gA$ together with online sampling on $r$.
For $\alpha>0$ let $\gE_\alpha$ be the event when for all $t\ge 1$,
\begin{align}
\label{eq:committal_rate_optimal_action_special_intermediate_2}
    \left( \pi^* - \pi_{\theta_t} \right)^\top r \le \frac{C}{t^\alpha},
\end{align}
By our assumption, there exists $\alpha>0$ such that $\probability(\gE_\alpha)>0$.
On this event, for any $t\ge 1$,
\begin{align}
\label{eq:committal_rate_optimal_action_special_intermediate_3}
    t^\alpha \cdot \left( 1 - \pi_{\theta_t}(a^*) \right) &\le \frac{1}{\Delta} \cdot t^\alpha \cdot \left( \pi^* - \pi_{\theta_t} \right)^\top r \qquad \left( \text{by \cref{eq:committal_rate_optimal_action_special_intermediate_1}} \right) \\
    &\le \frac{C}{\Delta}. \qquad \left( \text{by \cref{eq:committal_rate_optimal_action_special_intermediate_2}} \right)
\end{align}
Let $\big\{ \ttheta_t \big\}_{t\ge 1}$ with $\ttheta_1 = \theta_1$ be the sequence obtained by using $\gA$ with fixed sampling on $r$, such that $a_t = a^*$ for all $t \ge 1$. Since, by the assumption, $\gA$ is optimality-smart, we have $\pi_{\ttheta_t}(a^*) \ge \pi_{\theta_t}(a^*)$. Then, on $\gE_\alpha$, for any $t\ge 1$
\begin{align}
    t^\alpha \cdot \left( 1 - \pi_{\ttheta_t}(a^*) \right) &\le t^\alpha \cdot \left( 1 - \pi_{\theta_t}(a^*) \right) \\
    &\le \frac{C}{\Delta}, \qquad \left( \text{by \cref{eq:committal_rate_optimal_action_special_intermediate_3}} \right)\,.
\end{align}
Since $\mathbb{P}(\gE_\alpha)>0$ and $t^\alpha \cdot \left( 1 - \pi_{\ttheta_t}(a^*) \right)$ is non-random, it follows that for any $t\ge 1$, 
$t^\alpha \cdot \left( 1 - \pi_{\ttheta_t}(a^*) \right) \le C/\Delta$,
which, by \cref{def:committal_rate}, means that $\kappa(\gA, a^*) \ge \alpha$.
\end{proof}

\textbf{\cref{thm:geometry_convergence_tradeoff}} (Geometry-Convergence trade-off)\textbf{.}
If an algorithm $\gA$ is optimality-smart, and  $\kappa(\gA, a^*) = \kappa(\gA, a)$ for at least one $a \not= a^*$, then $\gA$ with on-policy sampling can only exhibit at most one of the following two behaviors: 
\textbf{(i)} $\gA$ converges to a globally optimal policy almost surely; 
\textbf{(ii)} $\gA$ converges to a deterministic policy at a rate faster than $O(1/t)$ with positive probability.
\begin{proof}
We prove that $\gA$ cannot achieve both of the two behaviors at the same time by contradiction. Suppose an algorithm $\gA$ can \textbf{(i)} converge to a globally optimal policy almost surely; and \textbf{(ii)} converges at a rate $O(1/t^\alpha)$ with positive probability, where $\alpha > 1$.

Since \textbf{(ii)} holds, according to \cref{thm:committal_rate_optimal_action_special}, we have $\kappa(\gA, a^*) \ge \alpha > 1$. By condition, there exists at least one sub-optimal action $a \not= a^*$, such that $\kappa(\gA, a) = \kappa(\gA, a^*) > 1$. According to \cref{thm:committal_rate_main_theorem}, we have $\pi_{\theta_t}(a) \to 1$ as $t \to \infty$ with positive probability, which contradicts \textbf{(i)}. Therefore, \textbf{(i)} and \textbf{(ii)} cannot hold simultaneously.
\end{proof}

\section{Proofs for Ensemble Methods (\texorpdfstring{\cref{sec:ensemble_method}}{Lg}) }

\textbf{\cref{thm:ensemble_method}.}
With probability $1 - \delta$, the best single run among $O(\log{\left( 1/ \delta \right)})$ independent runs of NPG (GNPG) converges to a globally optimal policy at an $O(e^{-c \cdot t})$ rate.
\begin{proof}
According to \cref{thm:failure_probability_softmax_natural_pg_special_on_policy_stochastic_gradient}, stochastic NPG of \cref{update_rule:softmax_natural_pg_special_on_policy_stochastic_gradient} will sample the optimal action $a^*$ forever (thus converge to the optimal policy) with probability at least
\begin{align}
    p(\text{NPG}, a^*) &\coloneqq \exp\bigg\{ - \frac{ \exp\{ \eta \cdot r(a^*) \} }{ \eta \cdot r(a^*)} \cdot \frac{ \sum_{a \not= a^*}{ \exp\{ \theta_1(a) \} } }{ \exp\{ \theta_1(a^*) \} } \bigg\} \qquad \left( \text{by \cref{eq:failure_probability_softmax_natural_pg_special_on_policy_stochastic_gradient_intermediate_19}} \right) \\
    &\in \Omega(1).
\end{align}
Moreover, with probability at least $p(\text{NPG}, a^*) $, the convergence rate is,
\begin{align}
\label{eq:ensemble_method_intermediate_1}
\MoveEqLeft
    \left( \pi^* - \pi_{\theta_t} \right)^\top r = \sum_{a \not= a^*}{ \pi_{\theta_t}(a) \cdot \left( r(a^*) - r(a) \right) } \\
    &\le 1 - \pi_{\theta_t}(a^*) \qquad \left( r \in (0, 1]^K \right) \\
    &\le \frac{ \sum_{a \not= a^*}{ \exp\{ \theta_1(a) \} } }{ \exp\{ \theta_1(a^*) + \eta \cdot r(a^*) \cdot \left( t - 1 \right) \} + \sum_{a \not= a^*}{ \exp\{ \theta_1(a) \} } } \qquad \left( \text{by \cref{eq:committal_rate_stochastic_npg_gnpg_intermediate_1}} \right) \\
    &\in O(e^{- c \cdot t}).
\end{align}
Consider $n(\text{NPG}) \in O(\log{(1/ \delta)})$ independent runs of NPG,
where
\begin{align}
\label{eq:ensemble_method_intermediate_2}
    n(\text{NPG}) \coloneqq \frac{1}{ \log{\big(\frac{1}{1 - p(\text{NPG}, a^*)}\big)}} \cdot \log{(1/ \delta)}.
\end{align}
The probability that all the $n(\text{NPG})$ runs do not converge to global optimal policy is at most
\begin{align}
\MoveEqLeft
    \left[ 1 - p(\text{NPG}, a^*) \right]^{ n(\text{NPG}) } = \left[ \exp\Big\{ \log{\big( 1 - p(\text{NPG}, a^*) \big)} \Big\} \right]^{ n(\text{NPG}) } \\
    &=\exp\bigg\{ - \log{\bigg(\frac{1}{1 - p(\text{NPG}, a^*)}\bigg)} \cdot \frac{1}{ \log{\big(\frac{1}{1 - p(\text{NPG}, a^*)}\big)}} \cdot \log{(1/ \delta)} \bigg\} \qquad \left( \text{by \cref{eq:ensemble_method_intermediate_2}} \right) \\
    &= e^{ - \log{(1/ \delta)} } = \delta,
\end{align}
which means with probability at least $1 - \delta$, the best single run converges to a globally optimal policy at an $O(e^{-c \cdot t})$ rate. 

For stochastic GNPG of \cref{update_rule:softmax_gnpg_special_on_policy_stochastic_gradient}, similar calculations show that with probability at least $1 - \delta$, the best single run among $n(\text{GNPG}) \in O(\log{(1/ \delta)})$ independent runs of GNPG converges to a globally optimal policy at an $O(e^{-c \cdot t})$ rate, where
\begin{align}
    n(\text{GNPG}) \coloneqq \frac{1}{ \log{\big(\frac{1}{1 - p(\text{GNPG}, a^*)}\big)}} \cdot \log{(1/ \delta)},
\end{align}
and
\begin{align}
    p(\text{GNPG}, a^*) &\coloneqq \exp\bigg\{ - \frac{ \sum_{a \not= a^*}{ \exp\{ \theta_1(a) \} }}{ \exp\{ \theta_1(a^*) \} } \cdot \frac{\sqrt{2} \cdot \exp\big\{ \frac{\eta}{\sqrt{2}} \big\} }{\eta} \bigg\} \qquad \left( \text{by \cref{eq:failure_probability_softmax_gnpg_special_on_policy_stochastic_gradient_intermediate_8}} \right) \\
    &\in \Omega(1),
\end{align}
thus finishing the proof.
\end{proof}

\section{General MDPs}
\label{sec:extension_general_mdps}


This section is devoted to results of general MDPs. \textbf{(i)} For convergence rate results in true gradient settings, we review relevant results in literature \citep{mei2020global,khodadadian2021linear,mei2021leveraging} without detailed proofs (except for the NPG method, for which we have a new analysis using the natural N\L{} inequality), since the conclusions are similar to one-state MDPs in \cref{sec:turnover_results}. \textbf{(ii)} For results in on-policy stochastic settings, we discuss the main ideas for the similar conclusions as in \cref{sec:turnover_results}. \textbf{(iii)} For results of the committal rate and the trade-off, we provide some calculations.

\subsection{RL Settings and Notations}

Given a finite set $\gX$, we denote $\Delta(\gX)$ as the set of all probability distributions on $\gX$. A finite MDP $\gM \coloneqq (\gS, \gA, \gP, r, \gamma)$ is determined by the finite state space $\gS$, action space $\gA$, transition function $\gP: \gS \times \gA \to \Delta(\gS)$, reward function $r: \gS \times \gA \to \sR$, and discount factor $\gamma \in [0 , 1)$. 

An agent maintains a policy $\pi: \gS \to \Delta(\gA)$. At time $t$, the agent is given a state $s_t$, and it takes an action $a_t \sim \pi(\cdot | s_t)$, receives a scalar reward $r(s_t, a_t)$ and a next-state $s_{t+1} \sim \gP( \cdot | s_t, a_t)$. The value function of $\pi$ under $s$ is defined as
\begin{align}
\label{eq:state_value_function}
    V^\pi(s) \coloneqq \expectation_{\substack{s_0 = s, a_t \sim \pi(\cdot | s_t), \\ s_{t+1} \sim \gP( \cdot | s_t, a_t)}}{\left[ \sum_{t=0}^{\infty}{\gamma^t r(s_t, a_t)} \right]}.
\end{align}
The state-action value of $\pi$ at $(s,a) \in \gS \times \gA$ is defined as
\begin{align}
\label{eq:state_action_value_function}
    Q^\pi(s, a) \coloneqq r(s, a) + \gamma \cdot \sum_{s^\prime}{ \gP( s^\prime | s, a) \cdot V^\pi(s^\prime) }.
\end{align}
The advantage function of $\pi$ is defined as
\begin{align}
    A^\pi(s,a) \coloneqq Q^\pi(s, a) - V^\pi(s).
\end{align}
The state distribution of $\pi$ is defined as,
\begin{align}
     d_{s_0}^{\pi}(s) \coloneqq (1 - \gamma) \cdot \sum_{t=0}^{\infty}{ \gamma^t \cdot \probability(s_t = s | s_0, \pi, \gP) }.
\end{align}
We also denote $V^\pi(\rho) \coloneqq \expectation_{s \sim \rho(\cdot)}{ \left[ V^\pi(s) \right]}$ and $d_{\rho}^{\pi}(s) \coloneqq \expectation_{s_0 \sim \rho(\cdot)}{\left[ d_{s_0}^{\pi}(s) \right]}$, where $\rho \in \Delta(\gS)$ is an initial state distribution. The optimal policy $\pi^*$ satisfies $V^{\pi^*}(\rho) = \sup_{\pi : \gS \to \Delta(\gA) }{ V^{\pi}(\rho)}$. For conciseness, we denote $V^* \coloneqq V^{\pi^*}$. 
Given tabular parameters $\theta(s,a) \in \sR$ for all $(s,a)$, the policy $\pi_\theta$ can be parameterized by $\theta$ as $\pi_\theta(\cdot | s) = \softmax(\theta(s, \cdot))$:
\begin{align}
\label{eq:softmax_transform_general}
    \pi_\theta(a | s) = \frac{ \exp\{ \theta(s, a) \} }{ \sum_{a^\prime \in \gA}{ \exp\{ \theta(s, a^\prime) \} } }, \text{ for all } (s, a) \in \gS \times \gA.
\end{align}
Without loss of generality, we assume $r(s,a) \in (0, 1)$ for all $(s, a) \in \gS \times \gA$. Generalizing expected reward maximization of \cref{eq:expected_reward_objective}, the problem here is then to maximize the value function, i.e.,
\begin{align}
    \sup_{\theta : \gS \times \gA \to \sR}{ V^{\pi_\theta}(\rho) }. 
\end{align}
Now we consider the above three algorithms and their results in general MDPs.

\subsection{True Gradient Settings}

\subsubsection{Softmax PG}

First, softmax PG has $\Theta(1/t)$ global convergence rates with the following assumption.
\begin{assumption}[Sufficient exploration]
\label{assump:pos_init} 
The initial state distribution satisfies $\min_s \mu(s)>0$.
\end{assumption}
\begin{figure}[ht]
\centering
\vskip -0.2in
\begin{minipage}{.6\linewidth}
    \begin{algorithm}[H]
    \caption{Softmax PG, true gradient}
    \label{alg:softmax_pg_general_true_gradient}
    \begin{algorithmic}
    \STATE {\bfseries Input:} Learning rate $\eta > 0$.
    \STATE {\bfseries Output:} Policies $\pi_{\theta_t} = \softmax(\theta_t)$.
    \STATE Initialize parameter $\theta_1(s,a)$ for all $(s,a) \in \gS \times \gA$.
    \WHILE{$t \ge 1$}
    \STATE $\theta_{t+1} \gets \theta_{t} + \eta \cdot \frac{\partial V^{\pi_{\theta_t}}(\mu)}{\partial \theta_t}$.
    \ENDWHILE
    \end{algorithmic}
    \end{algorithm}
\end{minipage}
\end{figure}
As shown in \citet{mei2020global},
the following
non-uniform \L{}ojasiewicz (N\L{}) inequality holds for value function, generalizing \cref{lem:non_uniform_lojasiewicz_softmax_special}.
\begin{lemma}[N\L{}, \citep{mei2020global}]
\label{lem:non_uniform_lojasiewicz_softmax_general}
Let $a^*(s)$ be the action that $\pi^*$ selects in state $s$. We have, for all $\theta \in \sR^{\gS \times \gA}$,
\begin{align}
    \left\| \frac{\partial V^{\pi_\theta}(\mu)}{\partial \theta }\right\|_2 \ge \bigg\| \frac{d_{\rho}^{\pi^*}}{ d_{\mu}^{\pi_\theta} } \bigg\|_\infty^{-1} \cdot \min_s{ \pi_\theta(a^*(s)|s) } \cdot \frac{1}{\sqrt{S}} \cdot \left( V^*(\rho) - V^{\pi_\theta}(\rho) \right).
\end{align}
\end{lemma}
The proof for \cref{lem:non_uniform_lojasiewicz_softmax_general} can be found in \citep[Lemma 8]{mei2020global}.

\begin{proposition}[PG upper bound \citep{mei2020global}]
\label{prop:upper_bound_softmax_pg_general_true_gradient}
Let \cref{assump:pos_init} hold and
let $\{\theta_t\}_{t\ge 1}$ be generated using \cref{alg:softmax_pg_general_true_gradient} with $\eta = (1 - \gamma)^3/8$. Then, for all $t\ge 1$,
\begin{align}
    V^*(\rho) - V^{\pi_{\theta_t}}(\rho) \le \frac{16 \cdot S }{c^2 \cdot (1-\gamma)^6 \cdot t} \cdot \bigg\| \frac{d_{\mu}^{\pi^*}}{\mu} \bigg\|_\infty^2 \cdot \bigg\| \frac{1}{\mu} \bigg\|_\infty,
\end{align}
where $c \coloneqq \inf_{s\in \cS,t\ge 1} \pi_{\theta_t}(a^*(s)|s) > 0$.
\end{proposition}
The proof for \cref{prop:upper_bound_softmax_pg_general_true_gradient} can be found in \citep[Theorem 4]{mei2020global}.
\begin{proposition}[PG lower bound \citep{mei2020global}]
\label{prop:lower_bound_softmax_pg_general_true_gradient}
For large enough $t \ge 1$, using \cref{alg:softmax_pg_general_true_gradient} with $\eta_t \in ( 0, 1] $,
\begin{align}
    V^*(\mu) - V^{\pi_{\theta_t}}(\mu) \ge \frac{ (1- \gamma)^5 \cdot (\Delta^*)^2}{12 \cdot t},
\end{align}
where $\Delta^* \coloneqq \min_{s \in \gS, a \not= a^*(s)}\{ Q^*(s, a^*(s)) - Q^*(s, a) \} > 0$ is the optimal value gap of the MDP.
\end{proposition}
The proof for \cref{prop:lower_bound_softmax_pg_general_true_gradient} can be found in \citep[Theorem 10]{mei2020global}.

\subsubsection{NPG}

The following NPG algorithm enjoys $O(e^{-c \cdot t})$ global convergence rate in general MDPs.
\begin{figure}[ht]
\centering
\vskip -0.2in
\begin{minipage}{.6\linewidth}
    \begin{algorithm}[H]
    \caption{Natural PG (NPG), true gradient}
    \label{alg:softmax_natural_pg_general_true_gradient}
    \begin{algorithmic}
    \STATE {\bfseries Input:} Learning rate $\eta > 0$.
    \STATE {\bfseries Output:} Policies $\pi_{\theta_t} = \softmax(\theta_t)$.
    \STATE Initialize parameter $\theta_1(s,a)$ for all $(s,a) \in \gS \times \gA$.
    \WHILE{$t \ge 1$}
    \STATE $\theta_{t+1} \gets \theta_{t} + \eta \cdot Q^{\pi_{\theta_{t}}}$.
    \ENDWHILE
    \end{algorithmic}
    \end{algorithm}
\end{minipage}
\end{figure}
\begin{proposition}[NPG upper bound \citep{khodadadian2021linear}]
\label{prop:upper_bound_softmax_natural_pg_general_true_gradient}
Let \cref{assump:pos_init} hold and let $\left\{ \theta_t \right\}_{t\ge 1}$ be generated using  \cref{alg:softmax_natural_pg_general_true_gradient} with $\eta > 0$, and $\pi_{\theta_1}(a | s) = 1 / A$ for all $(s, a)$. We have, for all $t\ge 1$,
\begin{align}
    V^*(\rho) - V^{\pi_{\theta_t}}(\rho) \le \frac{ 1}{ (1 - \gamma)^2 } \cdot e^{ - (t - \kappa) \cdot ( 1 - 1/\lambda ) \cdot \eta \cdot \Delta^* },
\end{align}
where $\kappa = \frac{\lambda}{\Delta^*}\cdot \big[ \frac{\log{A}}{\eta} + \frac{1}{(1 - \gamma)^2} \big] $ and $\lambda > 1$.
\end{proposition}
The proof for \cref{prop:upper_bound_softmax_natural_pg_general_true_gradient} can be found in \citep[Theorem 3.1]{khodadadian2021linear}. 

We provide a different analysis using the natural non-uniform \L{}ojasiewciz (N\L{}) inequality. The following \cref{lem:natural_lojasiewicz_discrete_general} generalizes the natural N\L{} inequality of \cref{lem:natural_lojasiewicz_discrete_special}. \cref{lem:natural_lojasiewicz_discrete_general,thm:upper_bound_softmax_natural_pg_general_true_gradient} are consistent with existing work of using adaptive / large learning rates and line search in NPG \citep{khodadadian2021linear,bhandari2020note}. As shown later in \cref{eq:natural_lojasiewicz_discrete_general_result_1}, since the $c(\theta_t)$ quantity could be very small, it is necessary to use a large learning rate $\eta > 0$ to get constant progresses.
\begin{lemma}[Natural N\L{} inequality, discrete]
\label{lem:natural_lojasiewicz_discrete_general}
Using \cref{alg:softmax_natural_pg_general_true_gradient}, we have, for all $t \ge 1$,
\begin{align}
\MoveEqLeft
    V^{\pi_{\theta_{t+1}}}(\rho) - V^{\pi_{\theta_{t}}}(\rho) \ge c(\theta_t) \cdot \left( 1 - \gamma \right) \cdot \bigg\| \frac{ d_\rho^{\pi^*} }{ \rho } \bigg\|_\infty^{-1} \cdot \left[ V^{ \pi^*}(\rho) - V^{\pi_{\theta_{t}}}(\rho) \right],
\end{align}
where $c(\theta_t) > 0$ depends on $\theta_t$ is given by
\begin{align}
\label{eq:natural_lojasiewicz_discrete_general_result_1}
    c(\theta_t) \coloneqq \min_{s \in \gS}{ \left[ 1 - \frac{1}{ \pi_{\theta_t}(\bar{a}_t(s) | s) \cdot \left( e^{ \eta \cdot \Delta_t(s) } - 1 \right) + 1} \right] } \in (0, 1),
\end{align}
and $\bar{a}_t(s)$ is the greedy action under state $s$, i.e.,
\begin{align}
\label{eq:natural_lojasiewicz_discrete_general_result_2}
    \bar{a}_t(s) \coloneqq \argmax_{a \in \gA}{ Q^{\pi_{\theta_{t}}}(s,a) },
\end{align}
and $\Delta_t(s)$ is the gap w.r.t. the greedy action $\bar{a}_t(s)$ under state $s$, i.e.,
\begin{align}
\label{eq:natural_lojasiewicz_discrete_general_result_3}
    \Delta_t(s) \coloneqq Q^{\pi_{\theta_{t}}}(s, \bar{a}_t(s) ) - \max_{a \not= \bar{a}_t(s)}\{ Q^{\pi_{\theta_{t}}}(s,a) \} > 0.
\end{align}
\end{lemma}
\begin{proof}
According to the performance difference \cref{lem:performance_difference_general}, we have
\begin{align}
\MoveEqLeft
    V^{\pi_{\theta_{t+1}}}(\rho) - V^{\pi_{\theta_{t}}}(\rho) = \frac{1}{1 - \gamma} \cdot \sum_{s}{ d_\rho^{ \pi_{\theta_{t+1}} }(s) \cdot \sum_{a}{ \left( \pi_{\theta_{t+1}}(a | s) - \pi_{\theta_{t}}(a | s) \right) \cdot Q^{\pi_{\theta_{t}}}(s,a) } } \\
    &\ge \sum_{s}{ \frac{ d_\rho^{ \pi_{\theta_{t+1}} }(s) }{1 - \gamma} \cdot \left[ 1 - \frac{1}{ \pi_{\theta_t}(\bar{a}_t(s) | s) \cdot \left( e^{ \eta \cdot \Delta_t(s) } - 1 \right) + 1} \right] \cdot \sum_{a}{ \left( \bar{\pi}_t(a | s) - \pi_{\theta_{t}}(a | s) \right) \cdot Q^{\pi_{\theta_{t}}}(s,a) } } \quad \left( \text{\cref{lem:natural_lojasiewicz_discrete_special}} \right) \\
    &\ge \sum_{s}{ \frac{ d_\rho^{ \pi_{\theta_{t+1}} }(s)}{1 - \gamma} \cdot \left[ 1 - \frac{1}{ \pi_{\theta_t}(\bar{a}_t(s) | s) \cdot \left( e^{ \eta \cdot \Delta_t(s) } - 1 \right) + 1} \right] \cdot \sum_{a}{ \left( \pi^*(a | s) - \pi_{\theta_{t}}(a | s) \right) \cdot Q^{\pi_{\theta_{t}}}(s,a) } },
\end{align}
where in the second last inequality $\bar{\pi}_t(\cdot | s)$ is the greedy policy under state $s$, i.e.,
\begin{align}
    \sum_{a \in \gA}{ \bar{\pi}_t(a | s) \cdot Q^{\pi_{\theta_{t}}}(s,a) } = \max_{\pi: \gS \to \Delta(\gA)}{ \sum_{a \in \gA}{ \pi(a | s) \cdot Q^{\pi_{\theta_{t}}}(s,a) }  }.
\end{align}
Next we have,
\begin{align}
\MoveEqLeft
    V^{\pi_{\theta_{t+1}}}(\rho) - V^{\pi_{\theta_{t}}}(\rho) \ge \frac{ c(\theta_t) }{1 - \gamma} \cdot \sum_{s}{ d_\rho^{ \pi_{\theta_{t+1}} }(s) \cdot \sum_{a}{ \left( \pi^*(a | s) - \pi_{\theta_{t}}(a | s) \right) \cdot Q^{\pi_{\theta_{t}}}(s,a) } } \\
    &= \frac{ c(\theta_t) }{1 - \gamma} \cdot \sum_{s}{ d_\rho^{\pi^*}(s) \cdot \frac{ d_\rho^{ \pi_{\theta_{t+1}} }(s) }{ d_\rho^{\pi^*}(s) } \cdot \sum_{a}{ \left( \pi^*(a | s) - \pi_{\theta_{t}}(a | s) \right) \cdot Q^{\pi_{\theta_{t}}}(s,a) } } \\
    &\ge c(\theta_t) \cdot \bigg\| \frac{ d_\rho^{\pi^*} }{ d_\rho^{ \pi_{\theta_{t+1}} } } \bigg\|_\infty^{-1} \cdot \frac{1}{1 - \gamma} \cdot \sum_{s}{ d_\rho^{\pi^*}(s) \cdot \sum_{a}{ \left( \pi^*(a | s) - \pi_{\theta_{t}}(a | s) \right) \cdot Q^{\pi_{\theta_{t}}}(s,a) } } \\
    &= c(\theta_t) \cdot \bigg\| \frac{ d_\rho^{\pi^*} }{ d_\rho^{ \pi_{\theta_{t+1}} } } \bigg\|_\infty^{-1} \cdot \left[ V^{ \pi^*}(\rho) - V^{\pi_{\theta_{t}}}(\rho) \right] \qquad \left( \text{by \cref{lem:performance_difference_general}} \right) \\
    &\ge c(\theta_t) \cdot \left( 1 - \gamma \right) \cdot \bigg\| \frac{ d_\rho^{\pi^*} }{ \rho } \bigg\|_\infty^{-1} \cdot \left[ V^{ \pi^*}(\rho) - V^{\pi_{\theta_{t}}}(\rho) \right]. \qquad \left( \text{by \cref{eq:stationary_distribution_dominate_initial_state_distribution}} \right) \qedhere
\end{align}
\end{proof}

\begin{theorem}[NPG upper bound]
\label{thm:upper_bound_softmax_natural_pg_general_true_gradient}
Using \cref{alg:softmax_natural_pg_general_true_gradient} with the following learning rate, for all $t \ge 1$,
\begin{align}
\label{eq:upper_bound_softmax_natural_pg_general_true_gradient_result_1}
    \eta_t = \frac{1}{ \min_{s \in \gS}\{ \pi_{\theta_t}(\bar{a}_t(s) | s) \cdot \Delta_t(s) \} },
\end{align}
where $\bar{a}_t(s)$ and $\Delta_t(s)$ are defined in \cref{eq:natural_lojasiewicz_discrete_general_result_2,eq:natural_lojasiewicz_discrete_general_result_3}, we have, for all $t \ge 1$,
\begin{align}
    V^{ \pi^*}(\rho) - V^{\pi_{\theta_{t}}}(\rho) \le \exp\left\{ - \frac{1 - \gamma}{2} \cdot \bigg\| \frac{ d_\rho^{\pi^*} }{ \rho } \bigg\|_\infty^{-1} \cdot \left( t - 1 \right) \right\} \cdot \left[ V^{ \pi^*}(\rho) - V^{\pi_{\theta_{1}}}(\rho) \right].
\end{align}
\end{theorem}
\begin{proof}
We have, for all state $s \in \gS$ and $t \ge 1$,
\begin{align}
\MoveEqLeft
    \pi_{\theta_t}(\bar{a}_t(s) | s) \cdot \left( e^{ \eta_t \cdot \Delta_t(s) } - 1 \right) \ge \pi_{\theta_t}(\bar{a}_t(s) | s) \cdot \eta_t \cdot \Delta_t(s) \qquad \left( e^x \ge 1 + x \right) \\
    &= \frac{ \pi_{\theta_t}(\bar{a}_t(s) | s) \cdot \Delta_t(s) }{ \min_{s \in \gS}\{ \pi_{\theta_t}(\bar{a}_t(s) | s) \cdot \Delta_t(s) \} } \qquad \left( \text{by \cref{eq:upper_bound_softmax_natural_pg_general_true_gradient_result_1}} \right) \\
    &\ge 1,
\end{align}
which implies that,
\begin{align}
\MoveEqLeft
    c(\theta_t) = \min_{s \in \gS}{ \left[ 1 - \frac{1}{ \pi_{\theta_t}(\bar{a}_t(s) | s) \cdot \left( e^{ \eta \cdot \Delta_t(s) } - 1 \right) + 1} \right] } \qquad \left( \text{by \cref{eq:natural_lojasiewicz_discrete_general_result_1}} \right) \\
    &\ge 1 - 1/2 = 1/2.
\end{align}
According to \cref{lem:natural_lojasiewicz_discrete_general}, we have, for all $t \ge 1$,
\begin{align}
    V^{\pi_{\theta_{t+1}}}(\rho) - V^{\pi_{\theta_{t}}}(\rho) &\ge c(\theta_t) \cdot \left( 1 - \gamma \right) \cdot \bigg\| \frac{ d_\rho^{\pi^*} }{ \rho } \bigg\|_\infty^{-1} \cdot \left[ V^{ \pi^*}(\rho) - V^{\pi_{\theta_{t}}}(\rho) \right] \\
    &\ge \frac{1 - \gamma}{2} \cdot \bigg\| \frac{ d_\rho^{\pi^*} }{ \rho } \bigg\|_\infty^{-1} \cdot \left[ V^{ \pi^*}(\rho) - V^{\pi_{\theta_{t}}}(\rho) \right],
\end{align}
which leads to the final result,
\begin{align}
\MoveEqLeft
    V^{ \pi^*}(\rho) - V^{\pi_{\theta_{t}}}(\rho) = V^{ \pi^*}(\rho) - V^{\pi_{\theta_{t-1}}}(\rho) - \left[ V^{\pi_{\theta_{t}}}(\rho) - V^{\pi_{\theta_{t-1}}}(\rho) \right] \\
    &\le \left( 1 - \frac{1 - \gamma}{2} \cdot \bigg\| \frac{ d_\rho^{\pi^*} }{ \rho } \bigg\|_\infty^{-1} \right) \cdot \left[ V^{ \pi^*}(\rho) - V^{\pi_{\theta_{t-1}}}(\rho) \right] \\
    &\le \exp\left\{ - \frac{1 - \gamma}{2} \cdot \bigg\| \frac{ d_\rho^{\pi^*} }{ \rho } \bigg\|_\infty^{-1} \cdot \left( t - 1 \right) \right\} \cdot \left[ V^{ \pi^*}(\rho) - V^{\pi_{\theta_{1}}}(\rho) \right]. \qedhere
\end{align}
\end{proof}

\subsubsection{GNPG}

Finally, GNPG also enjoys $O(e^{-c \cdot t})$ global convergence rate in general MDPs.
\begin{figure}[ht]
\centering
\vskip -0.2in
\begin{minipage}{.7\linewidth}
    \begin{algorithm}[H]
    \caption{Geometry-award normalized PG (GNPG), true gradient}
    \label{alg:softmax_gnpg_general_true_gradient}
    \begin{algorithmic}
    \STATE {\bfseries Input:} Learning rate $\eta > 0$.
    \STATE {\bfseries Output:} Policies $\pi_{\theta_t} = \softmax(\theta_t)$.
    \STATE Initialize parameter $\theta_1(s,a)$ for all $(s,a) \in \gS \times \gA$.
    \WHILE{$t \ge 1$}
    \STATE $\theta_{t+1} \gets \theta_{t} + \eta \cdot \frac{\partial V^{\pi_{\theta_t}}(\mu)}{\partial \theta_t} \Big/ \Big\| \frac{\partial V^{\pi_{\theta_t}}(\mu)}{\partial \theta_t} \Big\|_2$.
    \ENDWHILE
    \end{algorithmic}
    \end{algorithm}
\end{minipage}
\end{figure}
\begin{proposition}[GNPG upper bound \citep{mei2021leveraging}]
\label{prop:upper_bound_softmax_gnpg_general_true_gradient}
Let \cref{assump:pos_init} hold and let $\left\{ \theta_t \right\}_{t\ge 1}$ be generated using  \cref{alg:softmax_gnpg_general_true_gradient} with $\eta = \frac{ ( 1 - \gamma) \cdot \gamma }{ 6 \cdot ( 1 - \gamma) \cdot \gamma  + 4 \cdot \left( C_\infty - (1 - \gamma) \right) } \cdot \frac{ 1}{ \sqrt{S} }$, where $C_\infty \coloneqq \max_{\pi}{ \left\| \frac{d_{\mu}^{\pi}}{ \mu} \right\|_\infty}$. Denote $C_\infty^\prime \coloneqq \max_{\pi}{ \left\| \frac{d_{\rho}^{\pi}}{ \mu} \right\|_\infty}$. We have, for all $t\ge 1$,
\begin{align}
    V^*(\rho) - V^{\pi_{\theta_t}}(\rho) \le \frac{ \left( V^*(\mu) - V^{\pi_{\theta_1}}(\mu) \right) \cdot C_\infty^\prime }{ 1 - \gamma} \cdot e^{ - C \cdot (t-1)},
\end{align}
where $C = \frac{ ( 1 - \gamma)^2 \cdot \gamma \cdot c }{ 12 \cdot ( 1 - \gamma) \cdot \gamma  + 8 \cdot \left( C_\infty - (1 - \gamma) \right) } \cdot \frac{ 1}{ S } \cdot \Big\| \frac{d_{\mu}^{\pi^*}}{\mu} \Big\|_\infty^{-1} $ and $c \coloneqq \inf_{s\in \cS,t\ge 1} \pi_{\theta_t}(a^*(s)|s) > 0$.
\end{proposition}
The proof for \cref{prop:upper_bound_softmax_gnpg_general_true_gradient} can be found in \citep[Theorem 3]{mei2021leveraging}.

\subsection{On-policy Stochastic Gradient Settings}

For general MDPs with multiple states, we define the following on-policy parallel importance sampling (IS) estimator, using one sampled action under each state to estimate the policy gradient.

\begin{definition}[On-policy parallel IS]
\label{def:on_policy_parallel_importance_sampling}
At iteration $t$, under each state $s$, sample one action $a_t(s) \sim \pi_{\theta_t}(\cdot | s)$. The IS state-action value estimator $\hat{Q}^{\pi_{\theta_t}}$ is constructed as
\begin{align}
    \hat{Q}^{\pi_{\theta_t}}(s,a) = \frac{ \sI\left\{ a_t(s) = a \right\} }{ \pi_{\theta_t}(a | s) } \cdot Q^{\pi_{\theta_t}}(s,a),
\end{align}
 for all $(s, a) \in \gS \times \gA$.
\end{definition}
\cref{def:on_policy_parallel_importance_sampling} is a generalized version of \cref{def:on_policy_importance_sampling} to general MDPs, which does not specify how to estimate $Q^{\pi_{\theta_t}}(s,a)$. With \cref{def:on_policy_parallel_importance_sampling}, the on-policy stochastic softmax PG, NPG, and GNPG methods can be generalized to general MDPs, replacing the true action value $Q^{\pi_{\theta_t}}(s, \cdot) \in \sR^{|\gA|}$ in the policy gradient with $\hat{Q}^{\pi_{\theta_t}}(s, \cdot)$, which uses one sampled action under each state $s$. In practice, $\hat{Q}^{\pi_{\theta_t}}(s, \cdot)$ can be calculated using realistic PG estimators from on-policy roll-outs \citep[Algorithm 1]{zhang2020global}. Here we use \cref{def:on_policy_parallel_importance_sampling} to show the main ideas of this work.

\subsubsection{Softmax PG}

We show that \cref{thm:almost_sure_global_convergence_softmax_pg_special_on_policy_stochastic_gradient} can be generalized to general MDPs, using unbiased and bounded softmax PG properties.

According to the policy gradient theorem \citep{sutton2000policy}, the softmax PG used in \cref{alg:softmax_pg_general_true_gradient} is, for all $s \in \gS$,
\begin{align}
\MoveEqLeft
    \frac{\partial V^{\pi_\theta}(\mu)}{\partial \theta(s, \cdot)} = \frac{1}{1-\gamma} \cdot \sum_{s^\prime}{ d_{\mu}^{\pi_\theta}(s^\prime) \cdot \left[ \sum_{a}  \frac{\partial \pi_\theta(a | s^\prime)}{\partial \theta(s, \cdot) } \cdot Q^{\pi_\theta}(s^\prime,a) \right] } \\
    &= \frac{1}{1-\gamma} \cdot d_{\mu}^{\pi_\theta}(s) \cdot { \left[ \sum_{a} \frac{\partial \pi_\theta(a | s)}{\partial \theta(s, \cdot)} \cdot Q^{\pi_\theta}(s,a) \right] }. \qquad \left( \frac{\partial \pi_\theta(a | s^\prime)}{\partial \theta(s, \cdot)} = \rvzero, \ \forall s^\prime \not= s \right)
\end{align}
Calculating $\frac{\partial \pi_{\theta}(a | s)}{\partial \theta(s, \cdot)}$ for $\pi_\theta(\cdot | s) = \softmax(\theta(s, \cdot)) $, we have \citep{agarwal2020optimality,mei2020global}, for all $(s, a) \in \gS \times \gA$ ,
\begin{align}
\label{eq:true_softmax_pg_general_each_state_action}
    \frac{\partial V^{\pi_\theta}(\mu)}{\partial \theta(s,a)} = \frac{1}{1-\gamma} \cdot d_{\mu}^{\pi_\theta}(s) \cdot \pi_\theta(a|s) \cdot A^{\pi_\theta}(s,a).
\end{align}
Replacing $Q^{\pi_\theta}(s,a)$ with $\hat{Q}^{\pi_\theta}(s,a)$ in \cref{def:on_policy_parallel_importance_sampling}, we have \cref{alg:softmax_pg_general_on_policy_stochastic_gradient}.
\begin{figure}[ht]
\centering
\vskip -0.2in
\begin{minipage}{.75\linewidth}
    \begin{algorithm}[H]
    \caption{Softmax PG, on-policy stochastic gradient}
    \label{alg:softmax_pg_general_on_policy_stochastic_gradient}
    \begin{algorithmic}
    \STATE {\bfseries Input:} Learning rate $\eta > 0$.
    \STATE {\bfseries Output:} Policies $\pi_{\theta_t} = \softmax(\theta_t)$.
    \STATE Initialize parameter $\theta_1(s,a)$ for all $(s,a) \in \gS \times \gA$.
    \WHILE{$t \ge 1$}
    \STATE Sample $a_t(s) \sim \pi_{\theta_t}(\cdot | s)$ for all $s \in \gS$.
    \STATE $\hat{Q}^{\pi_{\theta_t}}(s,a) \gets \frac{ \sI\left\{ a_t(s) = a \right\} }{ \pi_{\theta_t}(a | s) } \cdot Q^{\pi_{\theta_t}}(s,a)$. $\qquad \left( \text{by \cref{def:on_policy_parallel_importance_sampling}} \right)$
    \STATE $\hat{g}_t(s, \cdot) \gets \frac{ 1 }{1-\gamma} \cdot d_{\mu}^{\pi_{\theta_t}}(s) \cdot { \left[ \sum_{a} \frac{\partial \pi_{\theta_t}(a | s)}{\partial \theta_t(s, \cdot)} \cdot \hat{Q}^{\pi_{\theta_t}}(s,a) \right] }$.
    \STATE $\theta_{t+1} \gets \theta_t + \eta \cdot \hat{g}_t$.
    \ENDWHILE
    \end{algorithmic}
    \end{algorithm}
\end{minipage}
\end{figure}

Similarly, we have, for all $(s, a) \in \gS \times \gA$,
\begin{align}
\label{eq:stochastic_softmax_pg_general_each_state_action}
    \theta_{t+1}(s, a) &\gets \theta_{t}(s,a) + \frac{ \eta }{1-\gamma} \cdot d_{\mu}^{\pi_{\theta_t}}(s) \cdot \pi_{\theta_t}(a | s) \cdot \left[ \hat{Q}^{\pi_{\theta_t}}(s, a) - \pi_{\theta_t}(\cdot | s)^\top \hat{Q}^{\pi_{\theta_t}}(s, \cdot) \right].
\end{align}

We show that the PG estimator in \cref{alg:softmax_pg_general_on_policy_stochastic_gradient} is unbiased and bounded, generalizing \cref{lem:unbiased_bounded_variance_softmax_pg_special_on_policy_stochastic_gradient}.
\begin{lemma}
\label{lem:unbiased_bounded_variance_softmax_pg_general_on_policy_stochastic_gradient}
Let $\hat{Q}^{\pi_{\theta}}(s,\cdot)$ be the IS parallel estimator using on-policy sampling $a(s) \sim \pi_{\theta}(\cdot | s)$, for all $s$. The stochastic softmax PG estimator is unbiased and bounded, i.e., 
\begin{align}
    &\expectation_{a(s) \sim \pi_\theta(\cdot | s)}{ \left[ \frac{1}{1-\gamma} \cdot d_{\mu}^{\pi_{\theta}}(s) \cdot \pi_{\theta}(a | s) \cdot \left( \hat{Q}^{\pi_{\theta}}(s, a) - \pi_{\theta}(\cdot | s)^\top \hat{Q}^{\pi_{\theta}}(s, \cdot) \right) \right] } = \frac{\partial V^{\pi_\theta}(\mu)}{\partial \theta(s,a)}, \\
    &\expectation_{a(s) \sim \pi_\theta(\cdot | s)}{ \left[ \sum_{(s,a)}{ \frac{ d_{\mu}^{\pi_{\theta}}(s)^2 \cdot \pi_{\theta}(a | s)^2  }{(1 - \gamma)^2} \cdot \left( \hat{Q}^{\pi_{\theta}}(s, a) - \pi_{\theta}(\cdot | s)^\top \hat{Q}^{\pi_{\theta}}(s, \cdot) \right)^2 } \right] } \le \frac{2}{ \left( 1 - \gamma \right)^4 }.
\end{align}
\end{lemma}
\begin{proof}
\textbf{First part.} Unbiased.

According to \cref{def:on_policy_parallel_importance_sampling}, we have,
\begin{align}
\MoveEqLeft
    \pi_{\theta}(a | s) \cdot \left( \hat{Q}^{\pi_{\theta}}(s, a) - \pi_{\theta}(\cdot | s)^\top \hat{Q}^{\pi_{\theta}}(s, \cdot) \right) \\
    &= \pi_{\theta}(a | s) \cdot \left( \frac{ \sI\left\{ a(s) = a \right\} }{ \pi_{\theta}(a | s) } \cdot Q^{\pi_{\theta}}(s,a) - \sum_{a^\prime \in \gA}{ \sI\left\{ a(s) = a^\prime \right\} \cdot Q^{\pi_{\theta}}(s,a^\prime) } \right) \\
    &= \sI\left\{ a(s) = a \right\} \cdot Q^{\pi_{\theta}}(s,a) - \pi_{\theta}(a | s) \cdot Q^{\pi_{\theta}}(s,a(s)).
\end{align}
Taking expectation, we have,
\begin{align}
\MoveEqLeft
    \expectation_{a(s) \sim \pi_\theta(\cdot | s)}{ \left[ \pi_{\theta}(a | s) \cdot \left( \hat{Q}^{\pi_{\theta}}(s, a) - \pi_{\theta}(\cdot | s)^\top \hat{Q}^{\pi_{\theta}}(s, \cdot) \right)  \right] } \\
    &= \sum_{a(s) \in \gA}{ \pi_\theta(a(s) | s) \cdot \left[ \sI\left\{ a(s) = a \right\} \cdot Q^{\pi_{\theta}}(s,a) - \pi_{\theta}(a | s) \cdot Q^{\pi_{\theta}}(s,a(s)) \right] } \\
    &= \pi_\theta(a | s) \cdot Q^{\pi_{\theta}}(s,a) - \pi_\theta(a | s) \cdot V^{\pi_{\theta}}(s) \\
    &= \pi_\theta(a|s) \cdot A^{\pi_\theta}(s,a).
\end{align}

\textbf{Second part.} Bounded.

First, using similar calculations in the second part of \cref{lem:unbiased_bounded_variance_softmax_pg_special_on_policy_stochastic_gradient}, we have, for all $s \in \gS$,
\begin{align}
\label{eq:unbiased_bounded_variance_softmax_pg_general_on_policy_stochastic_gradient_intermediate_1}
\MoveEqLeft
    \sum_{a \in \gA}{ \pi_{\theta}(a | s)^2 \cdot \left( \hat{Q}^{\pi_{\theta}}(s, a) - \pi_{\theta}(\cdot | s)^\top \hat{Q}^{\pi_{\theta}}(s, \cdot) \right)^2 } \\
    &= \sum_{a \in \gA} \pi_{\theta}(a | s)^2 \cdot \bigg[ \frac{ \sI\left\{ a(s) = a \right\} }{ \pi_{\theta}(a | s)^2 } \cdot Q^{\pi_{\theta}}(s,a)^2  \\
    &\qquad - 2 \cdot \frac{ \sI\left\{ a(s) = a \right\} }{ \pi_{\theta}(a | s) } \cdot Q^{\pi_{\theta}}(s,a) \cdot \pi_{\theta}(\cdot | s)^\top \hat{Q}^{\pi_{\theta}}(s, \cdot) + \left( \pi_{\theta}(\cdot | s)^\top \hat{Q}^{\pi_{\theta}}(s, \cdot) \right)^2 \bigg]  \\
    &= Q^{\pi_{\theta}}(s,a(s))^2 - 2 \cdot \pi_{\theta}(a(s) | s) \cdot Q^{\pi_{\theta}}(s,a(s))^2 + \sum_{a \in \gA}{ \pi_{\theta}(a | s)^2 } \cdot Q^{\pi_{\theta}}(s,a(s))^2 \\
    &=  \left[ 1 - \pi_{\theta}(a(s) | s) \right]^2 \cdot Q^{\pi_{\theta}}(s,a(s))^2 + \sum_{a \not= a(s)}{ \pi_{\theta}(a | s)^2 } \cdot Q^{\pi_{\theta}}(s,a(s))^2 \\
    &\le 2 \cdot \left[ 1 - \pi_{\theta}(a(s) | s) \right]^2 \cdot Q^{\pi_{\theta}}(s,a(s))^2 \qquad \left( \left\| x \right\|_2 \le \left\| x \right\|_1 \right) \\
    &\le \frac{2}{\left( 1 - \gamma \right)^2 }. \qquad \left( \pi_{\theta}(a(s) | s) \in (0, 1), \text{ and } Q^{\pi_{\theta}}(s,a(s)) \in (0, 1/(1 - \gamma)] \right) 
\end{align}
Therefore we have,
\begin{align}
\MoveEqLeft
    \expectation_{a(s) \sim \pi_\theta(\cdot | s)}{ \left[ \sum_{(s,a)}{ \frac{ d_{\mu}^{\pi_{\theta}}(s)^2 \cdot \pi_{\theta}(a | s)^2  }{(1 - \gamma)^2} \cdot \left( \hat{Q}^{\pi_{\theta}}(s, a) - \pi_{\theta}(\cdot | s)^\top \hat{Q}^{\pi_{\theta}}(s, \cdot) \right)^2 } \right] } \\
    &\le \sum_{s \in \gS}{ \frac{ d_{\mu}^{\pi_{\theta}}(s)^2  }{(1 - \gamma)^2} \cdot \frac{2}{\left( 1 - \gamma \right)^2 } } \qquad \left( \text{by \cref{eq:unbiased_bounded_variance_softmax_pg_general_on_policy_stochastic_gradient_intermediate_1}} \right) \\
    &\le \frac{2}{\left( 1 - \gamma \right)^4 } \cdot \left[ \sum_{s \in \gS}{ d_{\mu}^{\pi_{\theta}}(s)  } \right]^2 \qquad \left( \left\| x \right\|_2 \le \left\| x \right\|_1 \right) \\
    &= \frac{2}{\left( 1 - \gamma \right)^4 }. \qedhere
\end{align}
\end{proof}

The following lemma generalizes \cref{lem:non_uniform_smoothness_special_two_iterations}.
\begin{lemma}[Non-uniform Smoothness (NS) between two iterations]
\label{lem:non_uniform_smoothness_general_two_iterations}
Using stochastic softmax PG update, i.e.,
\begin{align}
\label{eq:non_uniform_smoothness_general_two_iterations_result_1}
    \theta^\prime &= \theta + \eta \cdot \hat{g} \\
    &\coloneqq \theta + \eta \cdot \frac{ 1 }{1-\gamma} \cdot  \expectation_{s^\prime \sim d_{\mu}^{\pi_\theta} } { \left[ \sum_{a}  \frac{\partial \pi_\theta(a | s^\prime)}{\partial \theta} \cdot \hat{Q}^{\pi_\theta}(s^\prime,a) \right] },
\end{align}
and using the true softmax PG norm in learning rate, i.e., 
\begin{align}
\label{eq:non_uniform_smoothness_general_two_iterations_result_2}
    \eta = \frac{\left( 1 - \gamma \right)^4 }{4 \cdot C} \cdot \left\| \frac{\partial V^{\pi_{\theta}}(\mu)}{\partial {\theta}} \right\|_2,
\end{align}
where 
\begin{align}
\label{eq:non_uniform_smoothness_general_two_iterations_result_3}
    C \coloneqq \left[ 3 + \frac{ 2 \cdot  \left( C_\infty - (1 - \gamma) \right) }{ (1 - \gamma) \cdot \gamma } \right] \cdot \sqrt{S},
\end{align}
and $C_\infty \coloneqq \max_{\pi}{ \left\| \frac{d_{\mu}^{\pi}}{ \mu} \right\|_\infty} \le \frac{1}{ \min_s \mu(s) } < \infty$ given \cref{assump:pos_init} hold, we have,
\begin{align}
\label{eq:non_uniform_smoothness_general_two_iterations_result_4}
    \left| V^{\pi_{\theta^\prime}}(\mu) - V^{\pi_{\theta}}(\mu) - \Big\langle \frac{\partial V^{\pi_\theta}(\mu)}{\partial \theta}, \theta^\prime - \theta \Big\rangle \right| \le C \cdot \left\| \frac{\partial V^{\pi_\theta}(\mu)}{\partial \theta} \right\|_2 \cdot \| \theta^\prime - \theta \|_2^2.
\end{align}
\end{lemma}
\begin{proof}
According to the non-uniform smoothness of value function \citep[Lemma 6]{mei2021leveraging}, we have, for all $y \in \sR^{S A}$ and $\theta$,
\begin{align}
\label{eq:non_uniform_smoothness_general_two_iterations_intermediate_0}
\MoveEqLeft
    \left| y^\top \frac{\partial^2 V^{\pi_\theta}(\mu)}{\partial \theta^2} y \right| \le C \cdot \left\| \frac{\partial V^{\pi_\theta}(\mu)}{\partial \theta }\right\|_2 \cdot \| y \|_2^2.
\end{align}

The proof is then similar to \cref{lem:non_uniform_smoothness_special_two_iterations}.
Denote $\theta_\zeta \coloneqq \theta + \zeta \cdot (\theta^\prime - \theta)$ with some $\zeta \in [0,1]$. According to Taylor's theorem, $\forall \theta, \ \theta^\prime$,
\begin{align}
\label{eq:non_uniform_smoothness_general_two_iterations_intermediate_1}
\MoveEqLeft
    \left| V^{\pi_{\theta^\prime}}(\mu) - V^{\pi_{\theta}}(\mu) - \Big\langle \frac{\partial V^{\pi_\theta}(\mu)}{\partial \theta}, \theta^\prime - \theta \Big\rangle \right| = \frac{1}{2} \cdot \left| \left( \theta^\prime - \theta \right)^\top \frac{\partial^2 V^{\pi_{\theta_\zeta}}(\mu)}{\partial {\theta_\zeta}^2} \left( \theta^\prime - \theta \right) \right| \\
    &\le \frac{ C }{2} \cdot \bigg\| \frac{\partial V^{\pi_{\theta_\zeta}}(\mu)}{\partial {\theta_\zeta}} \bigg\|_2 \cdot \| \theta^\prime - \theta \|_2^2, \qquad \left( \text{by \cref{eq:non_uniform_smoothness_general_two_iterations_intermediate_0}} \right)
\end{align}
where for conciseness we denote,

Denote $\zeta_1 \coloneqq \zeta$. Also denote $\theta_{\zeta_2} \coloneqq \theta + \zeta_2 \cdot (\theta_{\zeta_1} - \theta)$ with some $\zeta_2 \in [0,1]$. We have,
\begin{align}
\label{eq:non_uniform_smoothness_general_two_iterations_intermediate_2}
\MoveEqLeft
    \left\| \frac{\partial V^{\pi_{\theta_{\zeta_1}}}(\mu)}{\partial {\theta_{\zeta_1}}} - \frac{\partial V^{\pi_{\theta}}(\mu)}{\partial {\theta}} \right\|_2 = \left\| \int_{0}^{1} \bigg\langle \frac{\partial^2 V^{\pi_{\theta_{\zeta_2}}}(\mu)}{\partial {\theta_{\zeta_2}}^2}, \theta_{\zeta_1} - \theta \bigg\rangle d \zeta_2 \right\|_2 \\
    &\le \int_{0}^{1} \left\| \frac{\partial^2 V^{\pi_{\theta_{\zeta_2}}}(\mu)}{\partial {\theta_{\zeta_2}}^2} \right\|_2 \cdot \left\| \theta_{\zeta_1} - \theta \right\|_2 d \zeta_2 \qquad \left( \text{by Cauchy–Schwarz} \right) \\
    &\le \int_{0}^{1} C \cdot \left\| \frac{\partial V^{\pi_{\theta_{\zeta_2}}}(\mu)}{\partial {\theta_{\zeta_2}}} \right\|_2 \cdot \zeta_1 \cdot \left\| \theta^\prime - \theta \right\|_2 d \zeta_2 \qquad \left( \text{by \cref{eq:non_uniform_smoothness_general_two_iterations_intermediate_0}} \right)  \\
    &\le \int_{0}^{1} C \cdot \left\| \frac{\partial V^{\pi_{\theta_{\zeta_2}}}(\mu)}{\partial {\theta_{\zeta_2}}} \right\|_2 \cdot \eta \cdot \left\| \hat{g} \right\|_2 \ d \zeta_2, \qquad \left( \zeta_1 \in [0, 1], \text{ using  \cref{eq:non_uniform_smoothness_general_two_iterations_result_1}} \right)
\end{align}
where the second inequality is because of the Hessian is symmetric, and its operator norm is equal to its spectral radius. Therefore we have,
\begin{align}
\label{eq:non_uniform_smoothness_general_two_iterations_intermediate_4}
    \left\| \frac{\partial V^{\pi_{\theta_{\zeta_1}}}(\mu)}{\partial {\theta_{\zeta_1}}} \right\|_2 &\le \left\|  \frac{\partial V^{\pi_{\theta}}(\mu)}{\partial {\theta}} \right\|_2 + \left\| \frac{\partial V^{\pi_{\theta_{\zeta_1}}}(\mu)}{\partial {\theta_{\zeta_1}}} - \frac{\partial V^{\pi_{\theta}}(\mu)}{\partial {\theta}} \right\|_2 \qquad \left( \text{by triangle inequality} \right) \\
    &\le \left\| \frac{\partial V^{\pi_{\theta}}(\mu)}{\partial {\theta}} \right\|_2 + C \cdot \eta \cdot \left\| \hat{g} \right\|_2 \cdot \int_{0}^{1} \left\| \frac{\partial V^{\pi_{\theta_{\zeta_2}}}(\mu)}{\partial {\theta_{\zeta_2}}} \right\|_2 d \zeta_2. \qquad \left( \text{by \cref{eq:non_uniform_smoothness_general_two_iterations_intermediate_2}} \right)
\end{align}
Denote $\theta_{\zeta_3} \coloneqq \theta + \zeta_3 \cdot (\theta_{\zeta_2} - \theta)$ with $\zeta_3 \in [0,1]$. Using similar calculation in \cref{eq:non_uniform_smoothness_general_two_iterations_intermediate_2}, we have,
\begin{align}
\label{eq:non_uniform_smoothness_general_two_iterations_intermediate_5}
    \left\| \frac{\partial V^{\pi_{\theta_{\zeta_2}}}(\mu)}{\partial {\theta_{\zeta_2}}} \right\|_2 &\le \left\| \frac{\partial V^{\pi_{\theta}}(\mu)}{\partial {\theta}} \right\|_2 + \left\| \frac{\partial V^{\pi_{\theta_{\zeta_2}}}(\mu)}{\partial {\theta_{\zeta_2}}} - \frac{\partial V^{\pi_{\theta}}(\mu)}{\partial {\theta}} \right\|_2 \\
    &\le \left\| \frac{\partial V^{\pi_{\theta}}(\mu)}{\partial {\theta}} \right\|_2 + C \cdot \eta \cdot \left\| \hat{g} \right\|_2 \cdot \int_{0}^{1}  \left\| \frac{\partial V^{\pi_{\theta_{\zeta_3}}}(\mu)}{\partial {\theta_{\zeta_3}}} \right\|_2 d \zeta_3.
\end{align}
Combining \cref{eq:non_uniform_smoothness_general_two_iterations_intermediate_4,eq:non_uniform_smoothness_general_two_iterations_intermediate_5}, we have,
\begin{align}
\label{eq:non_uniform_smoothness_general_two_iterations_intermediate_6}
    \left\| \frac{\partial V^{\pi_{\theta_{\zeta_1}}}(\mu)}{\partial {\theta_{\zeta_1}}} \right\|_2 &\le \left( 1 + C \cdot \eta \cdot \left\| \hat{g} \right\|_2 \right) \cdot \left\| \frac{\partial V^{\pi_{\theta}}(\mu)}{\partial {\theta}} \right\|_2 \\
    &\qquad + \left( C \cdot \eta \cdot \left\| \hat{g} \right\|_2 \right)^2 \cdot \int_{0}^{1} \int_{0}^{1} \left\| \frac{\partial V^{\pi_{\theta_{\zeta_3}}}(\mu)}{\partial {\theta_{\zeta_3}}} \right\|_2 d \zeta_3 d \zeta_2,
\end{align}
which implies,
\begin{align}
\label{eq:non_uniform_smoothness_general_two_iterations_intermediate_7}
\MoveEqLeft
    \left\| \frac{\partial V^{\pi_{\theta_{\zeta_1}}}(\mu)}{\partial {\theta_{\zeta_1}}} \right\|_2 \le \sum_{i = 0}^{\infty}{ \left( C \cdot \eta \cdot \left\| \hat{g} \right\|_2 \right)^i } \cdot \left\| \frac{\partial V^{\pi_{\theta}}(\mu)}{\partial {\theta}} \right\|_2 \\
    &= \frac{1}{1 - C \cdot  \eta \cdot \left\| \hat{g} \right\|_2} \cdot \left\| \frac{\partial V^{\pi_{\theta}}(\mu)}{\partial {\theta}} \right\|_2 \qquad \left( C \cdot \eta \cdot \left\| \hat{g} \right\|_2 \in ( 0 , 1), \text{ see below} \right) \\
    &= \frac{1}{1 - \frac{ \left( 1 - \gamma \right)^4 }{4} \cdot \left\| \hat{g} \right\|_2 \cdot \left\| \frac{\partial V^{\pi_{\theta}}(\mu)}{\partial {\theta}} \right\|_2 } \cdot \left\| \frac{\partial V^{\pi_{\theta}}(\mu)}{\partial {\theta}} \right\|_2 \qquad \left( \eta = \frac{\left( 1 - \gamma \right)^4 }{4 \cdot C} \cdot \left\| \frac{\partial V^{\pi_{\theta}}(\mu)}{\partial {\theta}} \right\|_2 \right) \\
    &\le \frac{1}{1 - \frac{\left( 1 - \gamma \right)^2}{4} \cdot \left\| \hat{g} \right\|_2} \cdot \left\| \frac{\partial V^{\pi_{\theta}}(\mu)}{\partial {\theta}} \right\|_2 \qquad \left( \left\| \frac{\partial V^{\pi_{\theta}}(\mu)}{\partial {\theta}} \right\|_2 \le \frac{1}{ \left( 1 - \gamma \right)^2}, \text{ see below} \right) \\
    &\le \frac{1}{1 - \frac{1}{2}} \cdot \left\| \frac{\partial V^{\pi_{\theta}}(\mu)}{\partial {\theta}} \right\|_2 \qquad \left( \left\| \hat{g} \right\|_2 \le \frac{2}{ \left( 1 - \gamma \right)^2 }, \text{ see below} \right) \\
    &= 2 \cdot \left\| \frac{\partial V^{\pi_{\theta}}(\mu)}{\partial {\theta}} \right\|_2,
\end{align}
where the last inequality is from,
\begin{align}
\label{eq:non_uniform_smoothness_general_two_iterations_intermediate_8}
\MoveEqLeft
    \left\| \hat{g} \right\|_2^2 = \sum_{(s,a) \in \gS \times \gA}{ \hat{g}(s, a)^2 } \\
    &= \sum_{(s,a)}{ \frac{ d_{\mu}^{\pi_{\theta}}(s)^2 \cdot \pi_{\theta}(a | s)^2  }{(1 - \gamma)^2} \cdot \left( \hat{Q}^{\pi_{\theta}}(s, a) - \pi_{\theta}(\cdot | s)^\top \hat{Q}^{\pi_{\theta}}(s, \cdot) \right)^2 } \qquad \left( \text{by \cref{eq:stochastic_softmax_pg_general_each_state_action}} \right) \\
    &\le \frac{2}{\left( 1 - \gamma \right)^4 }, \qquad \left( \text{by \cref{lem:unbiased_bounded_variance_softmax_pg_general_on_policy_stochastic_gradient}} \right)
\end{align}
which implies that,
\begin{align}
    \left\| \hat{g} \right\|_2 &\le \frac{\sqrt{2}}{ \left( 1 - \gamma \right)^2 } \le \frac{ 2 }{ \left( 1 - \gamma \right)^2 },
\end{align}
and the second last inequality is from,
\begin{align}
\MoveEqLeft
    \left\| \frac{\partial V^{\pi_{\theta}}(\mu)}{\partial {\theta}} \right\|_2^2 = \sum_{s \in \gS} \frac{ d_{\mu}^{\pi_\theta}(s)^2 }{ \left( 1-\gamma \right)^2} \cdot \sum_{a \in \gA} \pi_\theta(a|s)^2 \cdot A^{\pi_\theta}(s,a)^2 \qquad \left( \text{by \cref{eq:true_softmax_pg_general_each_state_action}} \right) \\
    &\le \frac{1}{ \left( 1 - \gamma \right)^4} \cdot \sum_{s \in \gS} d_{\mu}^{\pi_\theta}(s)^2 \cdot \sum_{a \in \gA} \pi_\theta(a|s)^2 \qquad \left( \left| A^{\pi_\theta}(s,a) \right| \le \frac{1}{ 1 - \gamma } \right) \\
    &\le \frac{1}{ \left( 1 - \gamma \right)^4} \cdot \sum_{s \in \gS} d_{\mu}^{\pi_\theta}(s)^2 \cdot \left[ \sum_{a \in \gA} \pi_\theta(a|s) \right]^2 \qquad \left( \left\| x \right\|_2 \le \left\| x \right\|_1 \right) \\
    &\le \frac{1}{ \left( 1 - \gamma \right)^4} \cdot \left[ \sum_{s \in \gS} d_{\mu}^{\pi_\theta}(s) \right]^2 \qquad \left( \left\| x \right\|_2 \le \left\| x \right\|_1 \right) \\
    &= \frac{1}{ \left( 1 - \gamma \right)^4}, \\
\end{align}
which implies that,
\begin{align}
    \left\| \frac{\partial V^{\pi_{\theta}}(\mu)}{\partial {\theta}} \right\|_2 \le \frac{1}{ \left( 1 - \gamma \right)^2}.
\end{align}

Combining \cref{eq:non_uniform_smoothness_general_two_iterations_intermediate_1,eq:non_uniform_smoothness_general_two_iterations_intermediate_7} finishes the proof.
\end{proof}

The following result generalizes \cref{thm:almost_sure_global_convergence_softmax_pg_special_on_policy_stochastic_gradient}.
\begin{theorem}
\label{thm:almost_sure_global_convergence_softmax_pg_general_on_policy_stochastic_gradient}
Let $\{ \theta_t \}_{t \ge 1}$ be generated by using \cref{alg:softmax_pg_general_on_policy_stochastic_gradient}, i.e., for all $t \ge 1$,
\begin{align}
\label{eq:almost_sure_global_convergence_softmax_pg_general_on_policy_stochastic_gradient_result_1}
    \theta_{t+1} &= \theta_t + \eta \cdot \hat{g}_t \\
    &\coloneqq \theta_t + \eta \cdot \frac{ 1 }{1-\gamma} \cdot  \expectation_{s^\prime \sim d_{\mu}^{\pi_{\theta_t}} } { \left[ \sum_{a}  \frac{\partial \pi_{\theta_t}(a | s^\prime)}{\partial \theta_t} \cdot \hat{Q}^{\pi_{\theta_t}}(s^\prime,a) \right] },
\end{align}
with learning rate
\begin{align}
\label{eq:almost_sure_global_convergence_softmax_pg_general_on_policy_stochastic_gradient_result_2}
    \eta = \frac{ \left( 1 - \gamma \right)^4 }{4 \cdot C} \cdot \left\| \frac{\partial V^{\pi_{\theta_t}}(\mu)}{\partial {\theta_t}} \right\|_2
\end{align}
for all $t \ge 1$, and
\begin{align}
\label{eq:almost_sure_global_convergence_softmax_pg_general_on_policy_stochastic_gradient_result_3}
    C \coloneqq \left[ 3 + \frac{ 2 \cdot  \left( C_\infty - (1 - \gamma) \right) }{ (1 - \gamma) \cdot \gamma } \right] \cdot \sqrt{S}
\end{align}
as defined in \cref{eq:non_uniform_smoothness_general_two_iterations_result_3}, where $C_\infty \coloneqq \max_{\pi}{ \left\| \frac{d_{\mu}^{\pi}}{ \mu} \right\|_\infty} \le \frac{1}{ \min_s \mu(s) } < \infty$. Denote $C_\infty^\prime \coloneqq \max_{\pi}{ \left\| \frac{d_{\rho}^{\pi}}{ \mu} \right\|_\infty}$. We have, $V^*(\rho) - V^{\pi_{\theta_t}}(\rho) \to 0$ as $t \to \infty$ in probability, and for all $t \ge 1$,
\begin{align}
\MoveEqLeft
    \expectation{ \left[ V^*(\rho) - V^{\pi_{\theta_t}}(\rho) \right] } \le \frac{ 2 \cdot S }{ \sqrt{c}} \cdot \frac{\sqrt{3 \cdot (1 - \gamma) \cdot \gamma + 2 \cdot  \left( C_\infty - (1 - \gamma) \right)}}{ \left( 1 - \gamma \right)^5 \cdot \sqrt{\gamma} } \cdot \bigg\| \frac{d_{\rho}^{\pi^*}}{ \mu } \bigg\|_\infty^{3/2} \cdot \frac{ C_\infty^\prime }{ \sqrt{t} },
\end{align}
where $c > 0$ is independent with $t$.
\end{theorem}
\begin{proof}
First note that for any $\theta$ and $\mu$,
\begin{align}
\label{eq:stationary_distribution_dominate_initial_state_distribution}
    d_{\mu}^{\pi_\theta}(s) &= \expectation_{s_0 \sim \mu}{ \left[ d_{\mu}^{\pi_\theta}(s) \right] } \\
    &= \expectation_{s_0 \sim \mu}{ \left[ (1 - \gamma) \cdot  \sum_{t=0}^{\infty}{ \gamma^t \probability(s_t = s | s_0, \pi_\theta, \gP) } \right] } \\
    &\ge \expectation_{s_0 \sim \mu}{ \left[ (1 - \gamma) \cdot  \probability(s_0 = s | s_0)  \right] } \\
    &= (1 - \gamma) \cdot \mu(s).
\end{align}
Next, according to \cref{lem:value_suboptimality}, we have,
\begin{align}
\MoveEqLeft
    V^*(\rho) - V^{\pi_\theta}(\rho) = \frac{1}{1 - \gamma} \sum_{s}{ d_{\rho}^{\pi_\theta}(s)  \sum_{a}{ \left( \pi^*(a | s) - \pi_\theta(a | s) \right) \cdot Q^*(s,a) } } \\
    &= \frac{1}{1 - \gamma} \sum_{s} \frac{d_{\rho}^{\pi_\theta}(s)}{d_{\mu}^{\pi_\theta}(s)} \cdot d_{\mu}^{\pi_\theta}(s) \sum_{a}{ \left( \pi^*(a | s) - \pi_\theta(a | s) \right) \cdot Q^*(s,a) } \\
    &\le \frac{1}{1 - \gamma} \cdot \left\| \frac{d_{\rho}^{\pi_\theta}}{d_{\mu}^{\pi_\theta}} \right\|_\infty \sum_{s} d_{\mu}^{\pi_\theta}(s) \sum_{a}{ \left( \pi^*(a | s) - \pi_\theta(a | s) \right) \cdot Q^*(s,a) } \qquad \left( \sum_{a}{ \left( \pi^*(a | s) - \pi_\theta(a | s) \right) \cdot Q^*(s,a) } \ge 0 \right) \\
    &\le \frac{1}{(1 - \gamma)^2} \cdot \left\| \frac{d_{\rho}^{\pi_\theta}}{\mu} \right\|_\infty \sum_{s} d_{\mu}^{\pi_\theta}(s) \sum_{a}{ \left( \pi^*(a | s) - \pi_\theta(a | s) \right) \cdot Q^*(s,a) } 
    \qquad \left( \text{by \cref{eq:stationary_distribution_dominate_initial_state_distribution} and } \min_s\mu(s)>0 \right)
    \\
    &\le \frac{1}{(1 - \gamma)^2} \cdot C_\infty^\prime \cdot \sum_{s} d_{\mu}^{\pi_\theta}(s) \sum_{a}{ \left( \pi^*(a | s) - \pi_\theta(a | s) \right) \cdot Q^*(s,a) } \\
    &= \frac{1}{1 - \gamma} \cdot C_\infty^\prime \cdot \left[ V^*(\mu) - V^{\pi_\theta}(\mu) \right]. \qquad \left ( \text{by \cref{lem:value_suboptimality}} \right)
\end{align}
Denote $\delta(\theta_t) \coloneqq  V^*(\mu) - V^{\pi_{\theta_t}}(\mu)$. Let 
 We have, for all $t \ge 1$,
\begin{align}
\label{eq:almost_sure_global_convergence_softmax_pg_general_on_policy_stochastic_gradient_intermediate_1}
\MoveEqLeft
    \delta(\theta_{t+1}) - \delta(\theta_t) \\
    &= - V^{\pi_{\theta_{t+1}}}(\mu) + V^{\pi_{\theta_t}}(\mu) + \Big\langle \frac{\partial V^{\pi_{\theta_t}}(\mu)}{\partial {\theta_t}}, \theta_{t+1} - \theta_{t} \Big\rangle - \Big\langle \frac{\partial V^{\pi_{\theta_t}}(\mu)}{\partial {\theta_t}}, \theta_{t+1} - \theta_{t} \Big\rangle\\
    &\le C \cdot \left\| \frac{\partial V^{\pi_{\theta_t}}(\mu)}{\partial {\theta_t}} \right\|_2 \cdot \| \theta_{t+1} - \theta_{t} \|_2^2  - \Big\langle \frac{\partial V^{\pi_{\theta_t}}(\mu)}{\partial {\theta_t}}, \theta_{t+1} - \theta_{t} \Big\rangle \qquad \left( \text{by \cref{lem:non_uniform_smoothness_general_two_iterations}} \right) \\
    &= C \cdot \eta^2 \cdot  \left\| \frac{\partial V^{\pi_{\theta_t}}(\mu)}{\partial {\theta_t}} \right\|_2 \cdot \left\| \hat{g}_t \right\|_2^2 - \eta \cdot \Big\langle \frac{\partial V^{\pi_{\theta_t}}(\mu)}{\partial {\theta_t}}, \hat{g}_t \Big\rangle. \qquad \left( \text{using \cref{eq:almost_sure_global_convergence_softmax_pg_general_on_policy_stochastic_gradient_result_1}} \right)
\end{align}
Next, taking expectation over the random sampling on \cref{eq:almost_sure_global_convergence_softmax_pg_general_on_policy_stochastic_gradient_intermediate_1}, we have, 
\begin{align}
\label{eq:almost_sure_global_convergence_softmax_pg_general_on_policy_stochastic_gradient_intermediate_2}
\MoveEqLeft
    \expectation{ \left[ \delta(\theta_{t+1}) \right] } - \expectation{ \left[ \delta(\theta_{t}) \right] } \le C \cdot \eta^2 \cdot  \left\| \frac{\partial V^{\pi_{\theta_t}}(\mu)}{\partial {\theta_t}} \right\|_2 \cdot \expectation{ \left[ \left\| \hat{g}_t \right\|_2^2 \right] } - \eta \cdot \Big\langle \frac{\partial V^{\pi_{\theta_t}}(\mu)}{\partial {\theta_t}}, \expectation{ \left[ \hat{g}_t \right] } \Big\rangle \\
    &= C \cdot \eta^2 \cdot  \left\| \frac{\partial V^{\pi_{\theta_t}}(\mu)}{\partial {\theta_t}} \right\|_2 \cdot \expectation{ \left[ \left\| \hat{g}_t \right\|_2^2 \right] } - \eta \cdot \bigg\| \frac{\partial V^{\pi_{\theta_t}}(\mu)}{\partial {\theta_t}} \bigg\|_2^2 \qquad \left( \text{unbiased PG, by \cref{lem:unbiased_bounded_variance_softmax_pg_general_on_policy_stochastic_gradient}} \right) \\
    &\le \frac{2 \cdot C}{ \left( 1 - \gamma \right)^4 } \cdot \eta^2 \cdot  \left\| \frac{\partial V^{\pi_{\theta_t}}(\mu)}{\partial {\theta_t}} \right\|_2 - \eta \cdot \bigg\| \frac{\partial V^{\pi_{\theta_t}}(\mu)}{\partial {\theta_t}} \bigg\|_2^2 \qquad \left( \expectation{ \left[ \left\| \hat{g}_t \right\|_2^2 \right] } \le \frac{2}{ \left( 1 - \gamma \right)^4 }, \text{ by \cref{lem:unbiased_bounded_variance_softmax_pg_general_on_policy_stochastic_gradient}} \right) \\
    &= - \frac{ \left( 1 - \gamma \right)^4 }{8 \cdot C} \cdot \bigg\| \frac{\partial V^{\pi_{\theta_t}}(\mu)}{\partial {\theta_t}} \bigg\|_2^3 \qquad \left( \text{by \cref{eq:almost_sure_global_convergence_softmax_pg_general_on_policy_stochastic_gradient_result_2}} \right) \\
    &\le - \frac{ \left( 1 - \gamma \right)^4 }{8 \cdot C} \cdot \expectation{ \left[ \min_s{ \pi_{\theta_t}(a^*(s)|s) }^3 \right] } \cdot \expectation{ \left[ \delta(\theta_{t})^3 \right] } \cdot \bigg\| \frac{d_{\rho}^{\pi^*}}{ d_{\mu}^{\pi_{\theta_t}} } \bigg\|_\infty^{-3} \cdot \frac{1}{ S \cdot \sqrt{S}} 
    \qquad \left( \text{by \cref{lem:non_uniform_lojasiewicz_softmax_general}} \right) \\
    &\le - \frac{ \left( 1 - \gamma \right)^4 }{8 \cdot C} \cdot \left( \expectation{ \left[ \delta(\theta_{t}) \right] } \right)^3 \cdot \bigg\| \frac{d_{\rho}^{\pi^*}}{ d_{\mu}^{\pi_{\theta_t}} } \bigg\|_\infty^{-3} \cdot \frac{c}{ S \cdot \sqrt{S}}, \qquad \left( \text{by Jensen's inequality}\right)
\end{align}
where
\begin{align}
\label{eq:almost_sure_global_convergence_softmax_pg_general_on_policy_stochastic_gradient_intermediate_3}
    c &\coloneqq \inf_{t\ge 1} \expectation{ \left[ \min_s{ \pi_{\theta_t}(a^*(s)|s) }^3 \right] } \\
    &\ge \inf_{t\ge 1} \left( \expectation{ \left[ \min_s{ \pi_{\theta_t}(a^*(s)|s) } \right] } \right)^3 \qquad \left( \text{by Jensen's inequality}\right) \\
    &> 0,
\end{align}
and the last inequality is from \citep[Lemma 9]{mei2020global}, since the expected iteration equals the true gradient update, which converges to global optimal policy.
According to \cref{eq:stationary_distribution_dominate_initial_state_distribution}, we have,
\begin{align}
\label{eq:almost_sure_global_convergence_softmax_pg_general_on_policy_stochastic_gradient_intermediate_4}
\MoveEqLeft
    \expectation{ \left[ \delta(\theta_{t+1}) \right] } - \expectation{ \left[ \delta(\theta_{t}) \right] } \le - \frac{ \left( 1 - \gamma \right)^7 }{8 \cdot C} \cdot \left( \expectation{ \left[ \delta(\theta_{t}) \right] } \right)^3 \cdot \bigg\| \frac{d_{\rho}^{\pi^*}}{ \mu } \bigg\|_\infty^{-3} \cdot \frac{c}{ S \cdot \sqrt{S}}.
\end{align}
Denote $\tilde{\delta}(\theta_t) \coloneqq \expectation{ \left[ \delta(\theta_{t}) \right] } $. Using similar calculations in \cref{eq:almost_sure_global_convergence_softmax_pg_special_on_policy_stochastic_gradient_intermediate_4}, we have, for all $t \ge 1$,
\begin{align}
\label{eq:almost_sure_global_convergence_softmax_pg_general_on_policy_stochastic_gradient_intermediate_5}
\MoveEqLeft
    \frac{1}{ \tilde{\delta}(\theta_t)^{2} } \ge \frac{1}{\tilde{\delta}(\theta_1)^{2}} + 2 \cdot \sum_{s=1}^{t-1}{ \frac{1}{ \tilde{\delta}(\theta_{s})^{3} } \cdot \left( \tilde{\delta}(\theta_{s}) - \tilde{\delta}(\theta_{s+1}) \right) } \\
    &\ge \frac{1}{\tilde{\delta}(\theta_1)^{2}} + 2 \cdot \sum_{s=1}^{t-1}{ \frac{1}{ \bcancel{ \tilde{\delta}(\theta_{s})^{3} } } \cdot \frac{ \left( 1 - \gamma \right)^7 }{8 \cdot C} \cdot \bigg\| \frac{d_{\rho}^{\pi^*}}{ \mu } \bigg\|_\infty^{-3} \cdot \frac{c}{ S \cdot \sqrt{S}} \cdot \bcancel{\tilde{\delta}(\theta_{s})^{3}} } \qquad \left( \text{by \cref{eq:almost_sure_global_convergence_softmax_pg_general_on_policy_stochastic_gradient_intermediate_4}} \right) \\
    &= \frac{1}{\tilde{\delta}(\theta_1)^{2}} + \frac{ \left( 1 - \gamma \right)^7 }{4 \cdot C} \cdot \bigg\| \frac{d_{\rho}^{\pi^*}}{ \mu } \bigg\|_\infty^{-3} \cdot \frac{c}{ S \cdot \sqrt{S}} \cdot \left( t - 1 \right) \\
    &\ge \frac{ \left( 1 - \gamma \right)^7 }{4 \cdot C} \cdot \bigg\| \frac{d_{\rho}^{\pi^*}}{ \mu } \bigg\|_\infty^{-3} \cdot \frac{c}{ S \cdot \sqrt{S}} \cdot t \qquad \left( \tilde{\delta}(\theta_1)^2 \le \frac{1}{ \left( 1 - \gamma \right)^2 } < \frac{4 \cdot C}{ \left( 1 - \gamma \right)^7 } \cdot \bigg\| \frac{d_{\rho}^{\pi^*}}{ \mu } \bigg\|_\infty^{3} \cdot \frac{S \cdot \sqrt{S}}{c} \right) \\
    &= \frac{ \left( 1 - \gamma \right)^7 }{4} \cdot \bigg\| \frac{d_{\rho}^{\pi^*}}{ \mu } \bigg\|_\infty^{-3} \cdot \frac{c}{ S^2 } \cdot \frac{ (1 - \gamma) \cdot \gamma }{ 3 \cdot (1 - \gamma) \cdot \gamma + 2 \cdot  \left( C_\infty - (1 - \gamma) \right) } \cdot t \qquad \left( \text{by \cref{eq:almost_sure_global_convergence_softmax_pg_general_on_policy_stochastic_gradient_result_3}} \right) \\
    &= \frac{ \left( 1 - \gamma \right)^8 }{4} \cdot \bigg\| \frac{d_{\rho}^{\pi^*}}{ \mu } \bigg\|_\infty^{-3} \cdot \frac{c}{ S^2 } \cdot \frac{\gamma }{ 3 \cdot (1 - \gamma) \cdot \gamma + 2 \cdot  \left( C_\infty - (1 - \gamma) \right) } \cdot t,
\end{align}
which implies that,
\begin{align}
    \expectation{ \left[ V^*(\mu) - V^{\pi_{\theta_t}}(\mu) \right] } \le \frac{ 2 \cdot S }{ \sqrt{c}} \cdot \frac{\sqrt{3 \cdot (1 - \gamma) \cdot \gamma + 2 \cdot  \left( C_\infty - (1 - \gamma) \right)}}{ \left( 1 - \gamma \right)^4 \cdot \sqrt{\gamma} } \cdot \bigg\| \frac{d_{\rho}^{\pi^*}}{ \mu } \bigg\|_\infty^{3/2} \cdot \frac{1}{ \sqrt{t} },
\end{align}
where $c$ is from \cref{eq:almost_sure_global_convergence_softmax_pg_general_on_policy_stochastic_gradient_intermediate_3}. This leads to the final result,
\begin{align}
\MoveEqLeft
    \expectation{ \left[ V^*(\rho) - V^{\pi_{\theta_t}}(\rho) \right] } \le \frac{1}{1 - \gamma} \cdot C_\infty^\prime \cdot \expectation{ \left[ V^*(\mu) - V^{\pi_{\theta_t}}(\mu) \right] } \\
    &\le \frac{ 2 \cdot S }{ \sqrt{c}} \cdot \frac{\sqrt{3 \cdot (1 - \gamma) \cdot \gamma + 2 \cdot  \left( C_\infty - (1 - \gamma) \right)}}{ \left( 1 - \gamma \right)^5 \cdot \sqrt{\gamma} } \cdot \bigg\| \frac{d_{\rho}^{\pi^*}}{ \mu } \bigg\|_\infty^{3/2} \cdot \frac{ C_\infty^\prime }{ \sqrt{t} },
\end{align}
which implies that $V^*(\rho) - V^{\pi_{\theta_t}}(\rho) \to 0$ as $t \to \infty$ in probability, i.e.,
\begin{align}
    \lim_{t \to \infty}{ \probability{ \left( V^*(\rho) - V^{\pi_{\theta_t}}(\rho) > \epsilon \right) } } = 0,
\end{align}
for all $\epsilon > 0$.
\end{proof}

\subsubsection{NPG}

We show that results similar to \cref{thm:failure_probability_softmax_natural_pg_special_on_policy_stochastic_gradient} hold in general MDPs, given the following positive reward assumption, which generalizes \cref{asmp:positive_reward}.
\begin{assumption}[Positive reward]
\label{asmp:positive_reward_general}
$r(s, a) \in (0,1]$, for all $(s,a) \in \gS \times \gA$.
\end{assumption}
Given \cref{asmp:positive_reward_general}, we have, for all policy $\pi_\theta$,
\begin{align}
\label{eq:positive_Q_value}
    Q^{\pi_\theta}(s, a) \in \left(0 , 1 / \left(1 - \gamma \right) \right].
\end{align}
Replacing $Q^{\pi_{\theta_t}}$ in \cref{alg:softmax_natural_pg_general_true_gradient} with $\hat{Q}^{\pi_{\theta_t}}$ in \cref{def:on_policy_parallel_importance_sampling}, we have \cref{alg:softmax_natural_pg_general_on_policy_stochastic_gradient}.
\begin{figure}[ht]
\centering
\vskip -0.2in
\begin{minipage}{.7\linewidth}
    \begin{algorithm}[H]
    \caption{Natural PG (NPG), on-policy stochastic gradient}
    \label{alg:softmax_natural_pg_general_on_policy_stochastic_gradient}
    \begin{algorithmic}
    \STATE {\bfseries Input:} Learning rate $\eta > 0$.
    \STATE {\bfseries Output:} Policies $\pi_{\theta_t} = \softmax(\theta_t)$.
    \STATE Initialize parameter $\theta_1(s,a)$ for all $(s,a) \in \gS \times \gA$.
    \WHILE{$t \ge 1$}
    \STATE Sample $a_t(s) \sim \pi_{\theta_t}(\cdot | s)$ for all $s \in \gS$.
    \STATE $\hat{Q}^{\pi_{\theta_t}}(s,a) \gets \frac{ \sI\left\{ a_t(s) = a \right\} }{ \pi_{\theta_t}(a | s) } \cdot Q^{\pi_{\theta_t}}(s,a)$. $\qquad \left( \text{by \cref{def:on_policy_parallel_importance_sampling}} \right)$
    \STATE $\theta_{t+1} \gets \theta_{t} + \eta \cdot \hat{Q}^{\pi_{\theta_{t}}}$.
    \ENDWHILE
    \end{algorithmic}
    \end{algorithm}
\end{minipage}
\end{figure}

First, the following result shows that the on-policy stochastic natural PG is unbiased but unbounded, generalizing \cref{lem:bias_variance_softmax_natural_pg_special_on_policy_stochastic_gradient}.
\begin{lemma}
\label{lem:bias_variance_softmax_natural_pg_general_on_policy_stochastic_gradient}
Let $\hat{Q}^{\pi_{\theta}}(s,\cdot)$ be the IS parallel estimator using on-policy sampling $a(s) \sim \pi_{\theta}(\cdot | s)$, for all $s$. For NPG, we have,
\begin{align}
    \expectation_{a(s) \sim \pi_\theta(\cdot | s)}{ \left[ \hat{Q}^{\pi_{\theta}} \right] } &= Q^{\pi_{\theta}}, \\
    \expectation_{a(s) \sim \pi_\theta(\cdot | s)}{ \left[ \sum_{(s,a)}{ \hat{Q}^{\pi_{\theta}}(s,a)^2 } \right] } &= \sum_{(s,a)}{ \frac{ Q^{\pi_{\theta}}(s,a)^2  }{ \pi_{\theta}(a | s) } }.
\end{align}
\end{lemma}
\begin{proof}
\textbf{First part.} Unbiased.

According to \cref{def:on_policy_parallel_importance_sampling}, we have, for all $(s, a) \in \gS \times \gA$,
\begin{align}
    \expectation_{a(s) \sim \pi_\theta(\cdot | s)}{ \left[ \hat{Q}^{\pi_{\theta}}(s, a) \right] } &= \sum_{a(s) \in \gA}{ \pi_\theta(a(s) | s) } \cdot \frac{ \sI\left\{ a(s) = a \right\} }{ \pi_{\theta}(a | s) } \cdot Q^{\pi_{\theta}}(s,a) \\
    &= Q^{\pi_{\theta}}(s,a).
\end{align}

\textbf{Second part.} Unbounded.

We have, for all $s \in \gS$,
\begin{align}
    \sum_{a \in \gA}{ \hat{Q}^{\pi_{\theta}}(s, a)^2 } &= \sum_{a \in \gA}{ \frac{ \sI\left\{ a(s) = a \right\} }{ \pi_{\theta}(a | s)^2 } \cdot Q^{\pi_{\theta}}(s,a)^2 }.
\end{align}
Then we have,
\begin{align}
    \expectation_{a(s) \sim \pi_\theta(\cdot | s)}{ \left[ \sum_{(s,a)}{ \hat{Q}^{\pi_{\theta}}(s,a)^2 } \right] } &= \sum_{s \in \gS}{ \sum_{a(s) \in \gA}{ \pi_{\theta}(a(s) | s) \cdot \sum_{a \in \gA}{ \frac{ \sI\left\{ a(s) = a \right\} }{ \pi_{\theta}(a | s)^2 } \cdot Q^{\pi_{\theta}}(s,a)^2 } } } \\
    &= \sum_{s \in \gS}{ \sum_{a(s) \in \gA}{ \pi_{\theta}(a(s) | s) \cdot \frac{ 1 }{ \pi_{\theta}(a(s) | s)^2 } \cdot Q^{\pi_{\theta}}(s,a(s))^2  } } \\
    &= \sum_{(s,a)}{ \frac{ Q^{\pi_{\theta}}(s,a)^2  }{ \pi_{\theta}(a | s) } }. \qedhere
\end{align}
\end{proof}

The following results generalize \cref{thm:failure_probability_softmax_natural_pg_special_on_policy_stochastic_gradient} to general MDPs.
\begin{theorem}
\label{thm:failure_probability_softmax_natural_pg_general_on_policy_stochastic_gradient}
Let $a^*(s)$ be the action that $\pi^*$ selects in state $s$. Using \cref{alg:softmax_natural_pg_general_on_policy_stochastic_gradient}, we have: \textbf{(i)} for any state $s \in \gS$, with positive probability, $\sum_{a \not= a^*(s)}{ \pi_{\theta_t}(a | s)} \to 1$ as $t \to \infty$; \textbf{(ii)} for all $ (s,a) \in \gS \times \gA$, with positive probability, $\pi_{\theta_t}(a | s) \to 1$, as $t \to \infty$.
\end{theorem}
\begin{proof}
\textbf{First part.} For any state $s \in \gS$, with positive probability, $\sum_{a \not= a^*(s)}{ \pi_{\theta_t}(a | s)} \to 1$ as $t \to \infty$.

The proof is a generalization of the first part of \cref{thm:failure_probability_softmax_natural_pg_special_on_policy_stochastic_gradient}.

Let $\probability$ denote the probability measure that over the probability space $(\Omega,\gF)$ that holds all of our random variables.
Let $\gB = \{ a \in \gA: a \ne a^*(s) \}$.
By abusing notation, for any $\pi(\cdot | s) : \gA \to [0,1]$ map
we let $\pi_{\theta_t}(\gB | s)$ to stand for $\sum_{a \in \gB}\pi_{\theta_t}(a | s)$.
Define for $t\ge 1$ the event
$\gB_t = \{ a_t(s) \ne a^*(s) \} (= \{ a_t(s) \in \gB \})$ and let $\gE_t = \gB_1 \cap \dots \cap \gB_t$.
Thus, $\gE_t$ is the event that $a^*$ was not chosen in the first $t$ time steps.
Note that $\{\gE_t\}_{t\ge 1}$ is a nested sequence and thus, by the monotone convergence theorem, 
\begin{align}
\label{eq:failure_probability_softmax_natural_pg_general_on_policy_stochastic_gradient_intermediate_1}
    \lim_{t\to\infty}\probability{(\gE_t)} = \probability{(\gE)}\,,
\end{align}
where $\gE = \cap_{t\ge 1}\gB_t$.
We start with a lower bound on the probability of $\gE_t$.
The lower bound is stated in a generic form:
In particular, let  $(b_t)_{t\ge 1}$ be a deterministic sequence which satisfies that for any $t\ge 1$,
\begin{align}
\label{eq:failure_probability_softmax_natural_pg_general_on_policy_stochastic_gradient_intermediate_2}
    \mathbb{I}_{\gE_{t-1}} \cdot  \pi_{\theta_t}(\gB | s) \ge \mathbb{I}_{\gE_{t-1}} \cdot b_t \qquad \text{  holds $\probability$-almost surely},
\end{align}
where we let $\gE_0=\Omega$ and for an event $\gE$, $\mathbb{I}_{\gE}$ stands for the characteristic function of $\gE$ (i.e., $\mathbb{I}_{\gE}(\omega)=1$ if $\omega\in \gE$ and $\mathbb{I}_{\gE}(\omega)=0$, otherwise).
We make the following claim:

\noindent \underline{Claim 1}: Under the above assumption, for any $t\ge 1$
it holds that 
\begin{align}
\label{eq:failure_probability_softmax_natural_pg_general_on_policy_stochastic_gradient_intermediate_3}
    \probability(\gE_t) \ge \prod_{s=1}^t b_s\,.
\end{align}

For the proof of this claim let $\gH_t$ denote the sequence formed of the first $t$ actions:
\begin{align}
    \gH_t \coloneqq \left( a_1(s), a_2(s), \cdots, a_t(s) \right).
\end{align}
By definition, 
\begin{align}
    \theta_t(s, \cdot) = \gA\left( \theta_1(s, \cdot), a_1(s), Q^{\pi_{\theta_1}}(s, a_1(s)), \cdots, \theta_{t-1}(s, \cdot), a_{t-1}(s), Q^{\pi_{\theta_{t-1}}}(s, a_{t-1}(s)) \right).
\end{align}
By our assumption that the $t$th action $a_t(s)$ is chosen from $\pi_{\theta_t}(\cdot | s)$, it follows that 
$\probability$ satisfies that
for all $a$ and $t \ge 1$, 
\begin{align}
\label{eq:failure_probability_softmax_natural_pg_general_on_policy_stochastic_gradient_intermediate_4}
    \probability{ \left( a_t(s) = a \ | \ \gH_{t-1} \right) } = \pi_{\theta_t}(a | s) \qquad \text{ $\probability$-almost surely.} 
\end{align}

We prove the claim by induction on $t$.
For $t=1$, from \cref{eq:failure_probability_softmax_natural_pg_general_on_policy_stochastic_gradient_intermediate_2,eq:failure_probability_softmax_natural_pg_general_on_policy_stochastic_gradient_intermediate_4}, using that $\gE_0 = \Omega$ and $H_0=()$, we have that $\probability$-almost surely,
\begin{align}
    \probability{ ( \gE_1 ) } = \pi_{\theta_1}(\gB | s).
\end{align}
Suppose the claim holds up to $t-1$. We have,
\begin{align}
\label{eq:failure_probability_softmax_natural_pg_general_on_policy_stochastic_gradient_intermediate_5}
\MoveEqLeft
    \probability{ ( \gE_t ) } = \expectation{ \left[ \probability{ (  \gE_t  \ | \ \gH_{t-1} ) } \right] } \qquad \left( \text{by the tower rule} \right) \\
    &= \expectation{  \left[ \sI_{\gE_{t-1}} \cdot \probability{ ( \gB_t \ | \ \gH_{t-1} ) } \right] } \qquad \left( \gE_{t-1} \text{ is } \gH_{t-1} \text{-measurable} \right) \\
    &= \expectation{  \left[ \sI_{\gE_{t-1}} \cdot \pi_{\theta_t}(\gB | s ) \right] } \qquad \left(
    \text{by \cref{eq:failure_probability_softmax_natural_pg_general_on_policy_stochastic_gradient_intermediate_4}} \right) \\
    &\ge \expectation{  \left[ \sI_{\gE_{t-1}} \cdot b_t \right] } \qquad \left( \text{by \cref{eq:failure_probability_softmax_natural_pg_general_on_policy_stochastic_gradient_intermediate_2}} \right) \\
    &= b_t \cdot \probability{ ( \gE_{t-1} ) } \qquad \left( b_t \text{ is deterministic} \right) \\
    &= \prod_{s=1}^{t} b_s\,. \qquad \left( \text{induction hypothesis} \right)
\end{align}

Now, we claim the following:

\noindent \underline{Claim 2}: A suitable choice for $b_t$ is
\begin{align}
\label{eq:failure_probability_softmax_natural_pg_general_on_policy_stochastic_gradient_intermediate_6}
    b_t = \exp\left\{ \frac{ - \exp\left\{ \theta_1(s, a^*(s)) \right\} }{ (A-1) \cdot \exp{ \Big\{ \frac{ \sum_{a \not= a^*(s)}{ \theta_1(s, a) } + \eta \cdot Q_{\min} \cdot (t-1) }{A-1} \Big\} } } \right\}.
\end{align}

\underline{Proof of Claim 2}:
Clearly, it suffices to show that for any sequence $(a_1(s),\dots,a_{t-1}(s))$ such that $a_k(s)\ne a^*(s)$, $\theta_t(s, \cdot) \coloneqq \gA(\theta_1(s, \cdot),a_1(s), Q^{\pi_{\theta_1}}(s, a_1(s)),\dots,a_{t-1}(s),Q^{\pi_{\theta_{t-1}}}(s, a_{t-1}(s)))$ is such that
$\pi_{\theta_t}(\gB |s ) \ge b_t$ with $b_t$ as defined in \cref{eq:failure_probability_softmax_natural_pg_general_on_policy_stochastic_gradient_intermediate_6}. 

We have, for each sub-optimal action $a \not= a^*(s)$,
\begin{align}
\label{eq:failure_probability_softmax_natural_pg_general_on_policy_stochastic_gradient_intermediate_7}
\MoveEqLeft
    \theta_t(s, a) = \theta_1(s, a) + \eta \cdot \sum_{k=1}^{t-1}{ \hat{Q}^{\pi_{\theta_k}}(s, a) } \qquad \left( \text{by \cref{alg:softmax_natural_pg_general_on_policy_stochastic_gradient}} \right) \\
    &= \theta_1(s, a) + \eta \cdot \sum_{k=1}^{t-1}{ \frac{ \sI\left\{ a_k(s) = a \right\} }{ \pi_{\theta_k}(a | s) } \cdot Q^{\pi_{\theta_k}}(s, a) } \qquad \left( \text{by \cref{def:on_policy_parallel_importance_sampling}} \right)  \\
    &\ge \theta_1(s, a) + \eta \cdot \sum_{k=1}^{t-1}{ \sI\left\{ a_k(s) = a \right\} \cdot Q^{\pi_{\theta_k}}(s, a) } \qquad \left( \pi_{\theta_k}(a | s) \in (0, 1), \text{ and \cref{eq:positive_Q_value}} \right) \\
    &\ge \theta_1(s, a) + \eta \cdot Q_{\min} \cdot \sum_{k=1}^{t-1}{ \sI\left\{ a_k(s) = a \right\} }, \qquad \left( Q_{\min} \coloneqq \min_{\pi} \min_{(s,a)}{ Q^{\pi}(s, a) } \right)
\end{align}
where $Q_{\min} \in (0, 1/(1 - \gamma)]$  according to \cref{eq:positive_Q_value}. Then we have,
\begin{align}
\label{eq:failure_probability_softmax_natural_pg_general_on_policy_stochastic_gradient_intermediate_8}
\MoveEqLeft
    \sum_{a \not= a^*(s)}{ \exp\left\{ \theta_t(s, a) \right\} } \ge (A-1) \cdot \exp{ \left\{ \frac{ \sum_{a \not= a^*}{ \theta_t(s, a) } }{A-1} \right\} } \qquad \left( \text{by Jensen's inequality} \right) \\
    &\ge (A-1) \cdot \exp{ \left\{ \frac{ \sum_{a \not= a^*(s)}{ \theta_1(s, a) } + \eta \cdot Q_{\min} \cdot \sum_{a \not= a^*(s)}{ \sum_{k=1}^{t-1}{ \sI\left\{ a_k(s) = a \right\} } } }{A-1} \right\} } \qquad \left( \text{by \cref{eq:failure_probability_softmax_natural_pg_general_on_policy_stochastic_gradient_intermediate_7}} \right) \\
    &= (A-1) \cdot \exp{ \left\{ \frac{ \sum_{a \not= a^*(s)}{ \theta_1(s, a) } + \eta \cdot Q_{\min} \cdot (t-1) }{A-1} \right\} }. \qquad \left( a_k(s) \not= a^*(s), \text{ for all } k \in \{ 1, 2, \dots, t-1 \} \right)
\end{align}
On the other hand, we have,
\begin{align}
\label{eq:failure_probability_softmax_natural_pg_general_on_policy_stochastic_gradient_intermediate_9}
    \theta_t(s, a^*(s)) &= \theta_1(s, a^*(s)) + \eta \cdot \sum_{k=1}^{t-1}{ \frac{ \sI\left\{ a_k(s) = a^*(s) \right\} }{ \pi_{\theta_k}(a^*(s) | s) } \cdot Q^{\pi_{\theta_k}}(s, a^*(s)) } \qquad \left( \text{by \cref{alg:softmax_natural_pg_general_on_policy_stochastic_gradient,def:on_policy_parallel_importance_sampling}} \right) \\
    &= \theta_1(s, a^*(s)). \qquad \left( a_k(s) \not= a^*(s) \text{ for all } k \in \left\{ 1, 2, \dots, t-1 \right\} \right)
\end{align}
Next, we have,
\begin{align}
\label{eq:failure_probability_softmax_natural_pg_general_on_policy_stochastic_gradient_intermediate_10}
\MoveEqLeft
    \sum_{a \not= a^*(s)}{ \pi_{\theta_t}(a | s)} = 1 - \pi_{\theta_t}(a^*(s) | s) \\
    &= 1 - \frac{ \exp\left\{ \theta_t(s, a^*(s)) \right\} }{ \sum_{a \not= a^*(s)}{ \exp\left\{ \theta_t(s, a) \right\} } + \exp\left\{ \theta_t(s, a^*(s)) \right\} } \\
    &\ge 1 - \frac{ \exp\left\{ \theta_1(s, a^*(s)) \right\} }{ (A-1) \cdot \exp{ \left\{ \frac{ \sum_{a \not= a^*(s)}{ \theta_1(s, a) } + \eta \cdot Q_{\min} \cdot (t-1) }{A-1} \right\} } + \exp\left\{ \theta_1(s, a^*(s)) \right\} } \qquad \left( \text{by \cref{eq:failure_probability_softmax_natural_pg_general_on_policy_stochastic_gradient_intermediate_8,eq:failure_probability_softmax_natural_pg_general_on_policy_stochastic_gradient_intermediate_9}} \right) \\
    &\ge \exp\left\{ \frac{-1 }{ \frac{ (A-1) \cdot \exp{ \big\{ \frac{ \sum_{a \not= a^*(s)}{ \theta_1(s, a) } + \eta \cdot Q_{\min} \cdot (t-1) }{A-1} \big\} } + \exp\left\{ \theta_1(s, a^*(s)) \right\} }{ \exp\left\{ \theta_1(s, a^*(s)) \right\} } - 1 } \right\} \qquad \left( \text{by \cref{lem:auxiliary_lemma_1}} \right) \\
    &= \exp\left\{ \frac{ - \exp\left\{ \theta_1(s, a^*(s)) \right\} }{ (A-1) \cdot \exp{ \Big\{ \frac{ \sum_{a \not= a^*(s)}{ \theta_1(s, a) } + \eta \cdot Q_{\min} \cdot (t-1) }{A-1} \Big\} } } \right\} = b_t,
\end{align}

Combining \cref{eq:failure_probability_softmax_natural_pg_general_on_policy_stochastic_gradient_intermediate_1} with 
the conclusions of Claim 1 and 2 together, we get
\begin{align}
\label{eq:failure_probability_softmax_natural_pg_general_on_policy_stochastic_gradient_intermediate_11}
\MoveEqLeft
\probability{(\gE)}
    \ge \prod_{t=1}^{\infty}{ \exp\left\{ \frac{ - \exp\left\{ \theta_1(s, a^*(s)) \right\} }{ (A-1) \cdot \exp{ \Big\{ \frac{ \sum_{a \not= a^*(s)}{ \theta_1(s, a) } + \eta \cdot Q_{\min} \cdot (t-1) }{A-1} \Big\} } } \right\} } \qquad \left( \text{by \cref{eq:failure_probability_softmax_natural_pg_general_on_policy_stochastic_gradient_intermediate_10}} \right) \\
    &\ge \exp{ \left\{ - \frac{ \exp\left\{ \theta_1(s, a^*(s)) \right\} }{ \exp{ \Big\{ \frac{ \sum_{a \not= a^*(s)}{ \theta_1(s, a) } }{A-1} \Big\} } } \cdot \frac{ \exp{ \big\{ \frac{ \eta \cdot Q_{\min} }{A-1} \big\} } }{ A - 1} \cdot \int_{t=0}^{\infty}{ \frac{1}{ \exp{ \big\{ \frac{ \eta \cdot Q_{\min} \cdot t }{A-1} \big\} } } dt } \right\} } \\
    &= \exp{ \left\{ - \frac{ \exp\left\{ \theta_1(s, a^*(s)) \right\} }{ \exp{ \Big\{ \frac{ \sum_{a \not= a^*(s)}{ \theta_1(s, a) } }{A-1} \Big\} } } \cdot \frac{ \exp{ \big\{ \frac{ \eta \cdot Q_{\min} }{A-1} \big\} } }{ A - 1} \cdot \frac{A-1}{ \eta \cdot Q_{\min} } \right\} } \\
    &= \exp{ \left\{ - \frac{ \exp\left\{ \theta_1(s, a^*(s)) \right\} }{ \exp{ \Big\{ \frac{ \sum_{a \not= a^*(s)}{ \theta_1(s, a) } }{A-1} \Big\} } } \cdot \frac{ \exp{ \big\{ \frac{ \eta \cdot Q_{\min} }{A-1} \big\} } }{ \eta \cdot Q_{\min} } \right\} }.
\end{align}
Note that $Q_{\min} \in \Theta(1)$, $\exp\left\{ \theta_1(s, a^*(s)) \right\} \in \Theta(1)$, $\eta \in \Theta(1)$, $\exp{ \big\{ \frac{ \eta \cdot Q_{\min} }{A-1} \big\} } \in \Theta(1)$ and,
\begin{align}
    \exp{ \bigg\{ \frac{ \sum_{a \not= a^*(s)}{ \theta_1(s, a) } }{A-1} \bigg\} } \in \Theta(1).
\end{align}
Therefore, we have under state $s \in \gS$, `` the probability of sampling sub-optimal actions forever using on-policy sampling $a_t(s) \sim \pi_{\theta_t}(\cdot | s)$'' is lower bounded by a constant of $ \frac{1}{ \exp\left\{ \Theta(1) \right\} } \in \Theta(1)$, which implies that with positive probability $\Theta(1)$, we have $\sum_{a \not= a^*(s)}{ \pi_{\theta_t}(a | s)} \to 1$ as $t \to \infty$.

\textbf{Second part.} For all $ (s,a) \in \gS \times \gA$, with positive probability, $\pi_{\theta_t}(a | s) \to 1$, as $t \to \infty$.

The proof is similar to the second part of \cref{thm:failure_probability_softmax_natural_pg_special_on_policy_stochastic_gradient}. Let $\gB = \{ a \}$.
For any $\pi(\cdot | s): \gA \to [0,1]$ map
we let $\pi_{\theta_t}(\gB | s)$ to stand for $\pi_{\theta_t}(a | s)$.
Define for $t\ge 1$ the event
$\gB_t = \{ a_t(s) = a \} (= \{ a_t(s) \in \gB \})$ and let $\gE_t = \gB_1 \cap \dots \cap \gB_t$.
Thus, $\gE_t$ is the event that $a$ was chosen in the first $t$ time steps.
Note that $\{\gE_t\}_{t\ge 1}$ is a nested sequence and thus, by the monotone convergence theorem, $\lim_{t\to\infty}\probability{(\gE_t)} = \probability{(\gE)}$, where $\gE = \cap_{t\ge 1}\gB_t$. We show that by letting
\begin{align}
\label{eq:failure_probability_softmax_natural_pg_general_on_policy_stochastic_gradient_intermediate_12}
    b_t = \exp\bigg\{ \frac{ - \sum_{a^\prime \not= a}{ \exp\{ \theta_1(s, a^\prime) \} } }{ \exp\{ \theta_1(s, a) + \eta \cdot Q_{\min} \cdot \left( t - 1 \right) \} } \bigg\},
\end{align}
we have \cref{eq:failure_probability_softmax_natural_pg_general_on_policy_stochastic_gradient_intermediate_2,eq:failure_probability_softmax_natural_pg_general_on_policy_stochastic_gradient_intermediate_3} hold using the arguments in the first part.

It suffices to show that for any sequence $(a_1(s),\dots,a_{t-1}(s))$ such that $a_k(s)= a$ for all $k \in \{ 1, 2, \dots, t-1 \}$,  $\theta_t(s, \cdot) \coloneqq \gA(\theta_1(s, \cdot),a_1(s), Q^{\pi_{\theta_1}}(s, a_1(s)),\dots,a_{t-1}(s),Q^{\pi_{\theta_{t-1}}}(s, a_{t-1}(s)))$ is such that
$\pi_{\theta_t}(\gB | s ) \ge b_t$ with $b_t$ as defined in \cref{eq:failure_probability_softmax_natural_pg_general_on_policy_stochastic_gradient_intermediate_12}. Now suppose $a_1(s) = a, a_2(s) = a, \cdots, a_{t-1}(s) = a$. We have,
\begin{align}
\label{eq:failure_probability_softmax_natural_pg_general_on_policy_stochastic_gradient_intermediate_13}
    \theta_{t}(s, a) &= \theta_1(s, a) + \eta \cdot \sum_{k=1}^{t-1}{ \hat{Q}^{\pi_{\theta_k}}(s, a) } \qquad \left( \text{by \cref{alg:softmax_natural_pg_general_on_policy_stochastic_gradient}} \right) \\
    &= \theta_1(s, a) + \eta \cdot \sum_{k=1}^{t-1}{ \frac{ \sI\left\{ a_k(s) = a \right\} }{ \pi_{\theta_k}(a | s) } \cdot Q^{\pi_{\theta_k}}(s, a) } \qquad \left( \text{by \cref{def:on_policy_parallel_importance_sampling}} \right) \\
    &= \theta_1(s, a) + \eta \cdot \sum_{k=1}^{t-1}{ \frac{ Q^{\pi_{\theta_k}}(s, a) }{ \pi_{\theta_k}(a | s) }  } \qquad \left( a_k(s) = a \text{ for all } k \in \left\{ 1, 2, \dots, t-1 \right\} \right) \\
    &\ge \theta_1(s, a) + \eta \cdot \sum_{k=1}^{t-1}{ Q^{\pi_{\theta_k}}(s, a) } \qquad \left( \pi_{\theta_k}(a | s) \in (0, 1) \right) \\
    &\ge \theta_1(s, a) +  \eta \cdot Q_{\min} \cdot \left( t - 1 \right). \qquad \left( Q_{\min} \coloneqq \min_{\pi} \min_{(s,a)}{ Q^{\pi}(s, a) } \right)
\end{align}
On the other hand, we have, for any other action $a^\prime \not= a$,
\begin{align}
\label{eq:failure_probability_softmax_natural_pg_general_on_policy_stochastic_gradient_intermediate_14}
    \theta_{t}(s, a^\prime) &= \theta_1(s, a^\prime) + \eta \cdot \sum_{k=1}^{t-1}{ \frac{ \sI\left\{ a_k(s) = a^\prime \right\} }{ \pi_{\theta_k}(a^\prime | s) } \cdot Q^{\pi_{\theta_k}}(s, a^\prime) } \qquad \left( \text{by \cref{alg:softmax_natural_pg_general_on_policy_stochastic_gradient,def:on_policy_parallel_importance_sampling}} \right) \\
    &= \theta_1(s, a^\prime). \qquad \left( a_k(s) \not= a^\prime \text{ for all } k \in \left\{ 1, 2, \dots, t-1 \right\} \right)
\end{align}
Therefore, we have,
\begin{align}
\label{eq:failure_probability_softmax_natural_pg_general_on_policy_stochastic_gradient_intermediate_15}
\MoveEqLeft
    \pi_{\theta_t}(a | s) = 1 - \sum_{a^\prime \not= a}{ \pi_{\theta_t}(a^\prime |s ) } \\
    &= 1 - \frac{ \sum_{a^\prime \not= a}{ \exp\{ \theta_t(s, a^\prime) \} } }{ \exp\{ \theta_t(s, a) \} + \sum_{a^\prime \not= a}{ \exp\{ \theta_t(s, a^\prime) \} } } \\
    &\ge 1 - \frac{ \sum_{a^\prime \not= a}{ \exp\{ \theta_1(s, a^\prime) \} } }{ \exp\{ \theta_1(s, a) + \eta \cdot Q_{\min} \cdot \left( t - 1 \right) \} + \sum_{a^\prime \not= a}{ \exp\{ \theta_1(s, a^\prime) \} } }. \qquad \left( \text{by \cref{eq:failure_probability_softmax_natural_pg_general_on_policy_stochastic_gradient_intermediate_13,eq:failure_probability_softmax_natural_pg_general_on_policy_stochastic_gradient_intermediate_14}} \right) \\
    &\ge \exp\left\{ \frac{-1 }{ \frac{ \exp\{ \theta_1(s, a) + \eta \cdot Q_{\min} \cdot \left( t - 1 \right) \} + \sum_{a^\prime \not= a}{ \exp\{ \theta_1(s, a^\prime) \} } }{ \sum_{a^\prime \not= a}{ \exp\{ \theta_1(s, a^\prime) \} } } - 1 } \right\} \qquad \left( \text{by \cref{lem:auxiliary_lemma_1}} \right) \\
    &= \exp\bigg\{ \frac{ - \sum_{a^\prime \not= a}{ \exp\{ \theta_1(s, a^\prime) \} } }{ \exp\{ \theta_1(s, a) + \eta \cdot Q_{\min} \cdot \left( t - 1 \right) \} } \bigg\} \\
    &= b_t.
\end{align}
Therefore we have,
\begin{align}
\label{eq:failure_probability_softmax_natural_pg_general_on_policy_stochastic_gradient_intermediate_16}
\MoveEqLeft
    \prod_{t=1}^{\infty}{ \pi_{\theta_t}(a | s) } \ge \prod_{t=1}^{\infty}{ \exp\bigg\{ \frac{ - \sum_{a^\prime \not= a}{ \exp\{ \theta_1(s, a^\prime) \} } }{ \exp\{ \theta_1(s, a) + \eta \cdot Q_{\min} \cdot \left( t - 1 \right) \} } \bigg\} } \qquad \left( \text{by \cref{eq:failure_probability_softmax_natural_pg_general_on_policy_stochastic_gradient_intermediate_15}} \right) \\
    &= \exp\bigg\{ - \sum_{a^\prime \not= a}{ \exp\{ \theta_1(s, a^\prime) \} } \cdot \frac{ \exp\{ \eta \cdot Q_{\min} \} }{ \exp\{ \theta_1(s, a) \} } \cdot  \sum_{t=1}^{\infty}{  \frac{ 1 }{ \exp\{ \eta \cdot Q_{\min} \cdot t \} } } \bigg\} \\
    &\ge \exp\bigg\{ - \sum_{a^\prime \not= a}{ \exp\{ \theta_1(s, a^\prime) \} } \cdot \frac{ \exp\{ \eta \cdot Q_{\min} \} }{ \exp\{ \theta_1(s, a) \} } \cdot \int_{t=0}^{\infty}{  \frac{ 1 }{ \exp\{ \eta \cdot Q_{\min} \cdot t \} } dt } \bigg\} \\
    &= \exp\bigg\{ - \frac{ \exp\{ \eta \cdot Q_{\min} \} }{ \eta \cdot Q_{\min}} \cdot \frac{ \sum_{a^\prime \not= a}{ \exp\{ \theta_1(s, a^\prime) \} } }{ \exp\{ \theta_1(s, a) \} } \bigg\} \\
    &\in \Omega(1),
\end{align}
where the last line is due to $Q_{\min} \in \Theta(1)$, $\exp\left\{ \theta_1(s, a) \right\} \in \Theta(1)$ for all $(s, a) \in \gS \times \gA$, and $\eta \in \Theta(1)$. With \cref{eq:failure_probability_softmax_natural_pg_general_on_policy_stochastic_gradient_intermediate_16}, we have under state $s \in \gS$, ``the probability of sampling action $a$ forever using on-policy sampling $a_t(s) \sim \pi_{\theta_t}(\cdot | s)$'' is lower bounded by a constant of $\Omega(1)$. Therefore, for all $(s, a) \in \gS \times \gA$, with positive probability $\Omega(1)$, $\pi_{\theta_t}(a | s) \to 1$, as $t \to \infty$.
\end{proof}

\subsubsection{GNPG}

Replacing $Q^{\pi_{\theta_t}}$ in \cref{alg:softmax_gnpg_general_true_gradient} with $\hat{Q}^{\pi_{\theta_t}}$ in \cref{def:on_policy_parallel_importance_sampling}, we have \cref{alg:softmax_gnpg_general_on_policy_stochastic_gradient}.
\begin{figure}[ht]
\centering
\vskip -0.2in
\begin{minipage}{.85\linewidth}
    \begin{algorithm}[H]
    \caption{Geometry-award normalized PG (GNPG), on-policy stochastic gradient}
    \label{alg:softmax_gnpg_general_on_policy_stochastic_gradient}
    \begin{algorithmic}
    \STATE {\bfseries Input:} Learning rate $\eta > 0$.
    \STATE {\bfseries Output:} Policies $\pi_{\theta_t} = \softmax(\theta_t)$.
    \STATE Initialize parameter $\theta_1(s,a)$ for all $(s,a) \in \gS \times \gA$.
    \WHILE{$t \ge 1$}
    \STATE Sample $a_t(s) \sim \pi_{\theta_t}(\cdot | s)$ for all $s \in \gS$.
    \STATE $\hat{Q}^{\pi_{\theta_t}}(s,a) \gets \frac{ \sI\left\{ a_t(s) = a \right\} }{ \pi_{\theta_t}(a | s) } \cdot Q^{\pi_{\theta_t}}(s,a)$. $\qquad \left( \text{by \cref{def:on_policy_parallel_importance_sampling}} \right)$
    \STATE $\hat{g}_t(s, \cdot) \gets \frac{ 1 }{1-\gamma} \cdot d_{\mu}^{\pi_{\theta_t}}(s) \cdot { \left[ \sum_{a} \frac{\partial \pi_{\theta_t}(a | s)}{\partial \theta_t(s, \cdot)} \cdot \hat{Q}^{\pi_{\theta_t}}(s,a) \right] }$.
    \STATE $\theta_{t+1} \gets \theta_t + \eta \cdot \hat{g}_t / \left\| \hat{g}_t \right\|_2$.
    \ENDWHILE
    \end{algorithmic}
    \end{algorithm}
\end{minipage}
\end{figure}

The following result generalizes \cref{thm:failure_probability_softmax_gnpg_special_on_policy_stochastic_gradient}. The complication of the proofs is from the difference between NPG of \cref{alg:softmax_natural_pg_general_on_policy_stochastic_gradient} and GNPG of \cref{alg:softmax_gnpg_general_on_policy_stochastic_gradient}. In NPG, only $Q^{\pi_{\theta_t}}(s, \cdot)$ appears in the update of $\theta_t(s, \cdot)$. While in GNPG, other states also contribute to the update of $\theta_t(s, \cdot)$ through the normalization factor $\| \hat{g}_t \|_2$.
\begin{theorem}
\label{thm:failure_probability_softmax_gnpg_general_on_policy_stochastic_gradient}
Using \cref{alg:softmax_gnpg_general_on_policy_stochastic_gradient}, for all $(s, a) \in \gS \times \gA$, with positive probability, we have, $\pi_{\theta_t}(a | s) \to 1$ as $t \to \infty$.
\end{theorem}
\begin{proof}
Consider any deterministic policy $\bar{\pi}: s \mapsto \bar{\pi}(s)$, i.e.,
\begin{align}
\label{eq:deterministic_policy}
    \bar{\pi}(a | s) = \begin{cases}
		1, & \text{if } a = \bar{\pi}(s), \\
		0. & \text{otherwise}
	\end{cases}
\end{align}
We then make the following three claims.
\begin{description}
    \item[(i)] Using any fixed deterministic policy $\bar{\pi}$ to sample, i.e., $a_t(s) = \bar{\pi}(s)$, for all $s \in \gS$ and for all $t \ge 1$, and using GNPG update to generate a deterministic sequence $\{ \ttheta_t \}_{t \ge 1}$, we have $\pi_{\ttheta_t}(a | s) \to \bar{\pi}(a | s)$ as $t \to \infty$ for all $(s, a) \in \gS \times \gA$.
    \item[(ii)] The speed of $\pi_{\ttheta_t}$ approaches $\bar{\pi}$ is exponential, i.e., $1 - \pi_{\ttheta_t}(\bar{\pi}(s) | s) \in O(e^{-c \cdot t})$ for some $c > 0$, for all $t \ge 1$ and all $s \in \gS$.
    \item[(iii)] Using on-policy GNPG of \cref{alg:softmax_gnpg_general_on_policy_stochastic_gradient}, i.e., $a_t(s) \sim  \pi_{\theta_t}(\cdot | s)$, for all $(s,a) \in \gS \times \gA$, with positive probability, $\pi_{\theta_t}(a | s) \to 1$ as $t \to \infty$ (i.e., the main claim).
\end{description}

\textbf{First part. (i).} We show that with the fixed sampling $a_t(s) = \bar{\pi}(s)$, the following holds,
\begin{align}
\label{eq:failure_probability_softmax_gnpg_general_on_policy_stochastic_gradient_intermediate_1}
    \ttheta_{t+1}(s, \bar{\pi}(s)) &>  \ttheta_t(s, \bar{\pi}(s)), \text{ and} \\
    \ttheta_{t+1}(s, a^\prime) &<  \ttheta_t(s, a^\prime), \text{ for all } a^\prime \not= \bar{\pi}(s)
\end{align}
which according to $\pi_{\ttheta}(\cdot | s) = \softmax(\ttheta(s, \cdot))$ implies that, for all $t \ge 1$,
\begin{align}
\label{eq:failure_probability_softmax_gnpg_general_on_policy_stochastic_gradient_intermediate_1b}
    \pi_{\ttheta_{t+1}}(\bar{\pi}(s) | s) >  \pi_{\ttheta_{t}}(\bar{\pi}(s) | s).
\end{align}
Since $\pi_{\ttheta_{t}}(\bar{\pi}(s) | s) \le 1$, according to the monotone convergence, $\pi_{\ttheta_{t}}(\bar{\pi}(s) | s) \to c > 0$ as $t \to \infty$. And $c$ has to be $1$, otherwise due to \cref{eq:failure_probability_softmax_gnpg_general_on_policy_stochastic_gradient_intermediate_1} the probability of action $\bar{\pi}(s)$ can be further improved, which is a contradiction with convergence. Next, we show that \cref{eq:failure_probability_softmax_gnpg_general_on_policy_stochastic_gradient_intermediate_1} holds.

By the assumption of using fixed sampling, we have, at iteration $t \ge 1$, $a_t(s) = \bar{\pi}(s)$. Then we have,
\begin{align}
\label{eq:failure_probability_softmax_gnpg_general_on_policy_stochastic_gradient_intermediate_2}
\MoveEqLeft
    \hat{g}_t(s, \bar{\pi}(s)) = \frac{ 1 }{1-\gamma} \cdot d_{\mu}^{\pi_{\ttheta_t}}(s) \cdot \pi_{\ttheta_t}(\bar{\pi}(s) | s) \cdot \left[ \frac{ \sI\left\{ a_t(s) = \bar{\pi}(s) \right\} }{ \pi_{\ttheta_t}(a_t(s) | s) } \cdot Q^{\pi_{\ttheta_t}}(s,a_t(s)) - \sum_{a^\prime \in \gA}{ \sI\left\{ a_t(s) = a^\prime \right\} \cdot Q^{\pi_{\ttheta_t}}(s,a^\prime) } \right] \\
    &=  \frac{ 1 }{1-\gamma} \cdot d_{\mu}^{\pi_{\ttheta_t}}(s) \cdot \left( 1 - \pi_{\ttheta_t}(\bar{\pi}(s) | s) \right) \cdot Q^{\pi_{\ttheta_t}}(s, \bar{\pi}(s)) \qquad \left( a_t(s) = \bar{\pi}(s) \right) \\
    &> 0, \qquad \left( d_{\mu}^{\pi_{\ttheta_t}}(s) > 0, \ \pi_{\ttheta_t}(\bar{\pi}(s) | s) \in (0, 1), \ Q^{\pi_{\ttheta_t}}(s, \bar{\pi}(s)) \in (0, 1/(1-\gamma)] \right)
\end{align}
where the last inequality is from \cref{eq:stationary_distribution_dominate_initial_state_distribution}, \cref{assump:pos_init}, and \cref{eq:positive_Q_value}. On the other hand, for any other action $a^\prime \not= \bar{\pi}(s)$, we have,
\begin{align}
\label{eq:failure_probability_softmax_gnpg_general_on_policy_stochastic_gradient_intermediate_3}
    \hat{g}_t(s, a^\prime) &= - \frac{ 1 }{1-\gamma} \cdot d_{\mu}^{\pi_{\ttheta_t}}(s) \cdot \pi_{\ttheta_t}(a^\prime | s) \cdot Q^{\pi_{\ttheta_t}}(s, \bar{\pi}(s)) < 0.
\end{align}
Therefore, according to the GNPG update, i.e., 
\begin{align}
    \ttheta_{t+1}(s, \cdot) \gets \ttheta_t(s, \cdot) + \eta \cdot \hat{g}_t(s, \cdot) / \| \hat{g}_t \|_2,
\end{align}
we have \cref{eq:failure_probability_softmax_gnpg_general_on_policy_stochastic_gradient_intermediate_1} holds.

\textbf{Second part. (ii).} We calculate the progress $\ttheta_t(s, \bar{\pi}(s))$ can get. We have,
\begin{align}
\label{eq:failure_probability_softmax_gnpg_general_on_policy_stochastic_gradient_intermediate_4}
\MoveEqLeft
    \left\| \hat{g}_t(s, \cdot) \right\|_2^2 = \hat{g}_t(s, \bar{\pi}(s))^2 + \sum_{a^\prime \not= \bar{\pi}(s)}{ \hat{g}_t(s, a^\prime)^2 } \\
    &\le \frac{ 1 }{(1-\gamma)^2} \cdot d_{\mu}^{\pi_{\ttheta_t}}(s)^2 \cdot Q^{\pi_{\ttheta_t}}(s, \bar{\pi}(s))^2 \cdot 2 \cdot \left( 1 - \pi_{\ttheta_t}(\bar{\pi}(s) | s) \right)^2. \qquad \left( \left\| x \right\|_2 \le \left\| x \right\|_1 \right)
\end{align}
At iteration $t \ge 1$, we define $\bar{s}_t$ as follows,
\begin{align}
\label{eq:failure_probability_softmax_gnpg_general_on_policy_stochastic_gradient_intermediate_5}
    \bar{s}_t = \argmax_{s^\prime \in \gS}{  \left( 1 - \pi_{\ttheta_t}(\bar{\pi}(s^\prime) | s^\prime) \right) \cdot Q^{\pi_{\ttheta_t}}(s^\prime, \bar{\pi}(s^\prime)) }.
\end{align}
According to \cref{eq:failure_probability_softmax_gnpg_general_on_policy_stochastic_gradient_intermediate_4}, we have,
\begin{align}
\label{eq:failure_probability_softmax_gnpg_general_on_policy_stochastic_gradient_intermediate_6}
\MoveEqLeft
    \left\| \hat{g}_t \right\|_2^2 = \left\| \hat{g}_t(\bar{s}_t, \cdot) \right\|_2^2 + \sum_{s^\prime \not= \bar{s}_t}{ \left\| \hat{g}_t(s^\prime, \cdot) \right\|_2^2 } \\
    &\le \frac{ 1 }{(1-\gamma)^2} \cdot d_{\mu}^{\pi_{\ttheta_t}}(\bar{s}_t)^2 \cdot Q^{\pi_{\ttheta_t}}(\bar{s}_t, \bar{\pi}(\bar{s}_t))^2 \cdot 2 \cdot \left( 1 - \pi_{\ttheta_t}(\bar{\pi}(\bar{s}_t) | \bar{s}_t) \right)^2 \\
    &\qquad + \sum_{s^\prime \not= \bar{s}_t}{ \frac{ 1 }{(1-\gamma)^2} \cdot d_{\mu}^{\pi_{\ttheta_t}}(s^\prime)^2 \cdot Q^{\pi_{\ttheta_t}}(s^\prime, \bar{\pi}(s^\prime))^2 \cdot 2 \cdot \left( 1 - \pi_{\ttheta_t}(\bar{\pi}(s^\prime) | s^\prime) \right)^2 } \qquad \left( \left\| x \right\|_2 \le \left\| x \right\|_1 \right) \\
    &\le \frac{ 1 }{(1-\gamma)^2} \cdot Q^{\pi_{\ttheta_t}}(\bar{s}_t, \bar{\pi}(\bar{s}_t))^2 \cdot 2 \cdot \left( 1 - \pi_{\ttheta_t}(\bar{\pi}(\bar{s}_t) | \bar{s}_t) \right)^2 \cdot \sum_{s^\prime \in \gS}{ d_{\mu}^{\pi_{\ttheta_t}}(s^\prime)^2 } \qquad \left( \text{by \cref{eq:failure_probability_softmax_gnpg_general_on_policy_stochastic_gradient_intermediate_5}} \right) \\
    &\le \frac{ 1 }{(1-\gamma)^2} \cdot Q^{\pi_{\ttheta_t}}(\bar{s}_t, \bar{\pi}(\bar{s}_t))^2 \cdot 2 \cdot \left( 1 - \pi_{\ttheta_t}(\bar{\pi}(\bar{s}_t) | \bar{s}_t) \right)^2, \qquad \left( \left\| x \right\|_2 \le \left\| x \right\|_1 \right)
\end{align}
which implies that,
\begin{align}
    \left\| \hat{g}_t \right\|_2 \le \frac{\sqrt{2}}{ 1-\gamma } \cdot \left( 1 - \pi_{\ttheta_t}(\bar{\pi}(\bar{s}_t) | \bar{s}_t) \right) \cdot Q^{\pi_{\ttheta_t}}(\bar{s}_t, \bar{\pi}(\bar{s}_t)).
\end{align}
According to \cref{eq:failure_probability_softmax_gnpg_general_on_policy_stochastic_gradient_intermediate_2,eq:failure_probability_softmax_gnpg_general_on_policy_stochastic_gradient_intermediate_6,eq:stationary_distribution_dominate_initial_state_distribution}, we have,
\begin{align}
    \ttheta_{t+1}(s, \bar{\pi}(\bar{s}_t)) &\gets \ttheta_{t}(s, \bar{\pi}(\bar{s}_t)) + \eta \cdot \frac{ \hat{g}_t(\bar{s}_t, \bar{\pi}(\bar{s}_t)) }{ \left\| \hat{g}_t \right\|_2 } \\
    &\ge \ttheta_{t}(s, \bar{\pi}(\bar{s}_t)) + \eta \cdot \frac{ 1 - \gamma }{ \sqrt{2} },
\end{align}
which implies that,
\begin{align}
\MoveEqLeft
\label{eq:failure_probability_softmax_gnpg_general_on_policy_stochastic_gradient_intermediate_7}
    1 - \pi_{\ttheta_{t+1}}(\bar{\pi}(\bar{s}_t) | \bar{s}_t) = \frac{ \sum_{a^\prime \not= \bar{\pi}(\bar{s}_t) } \exp\{ \ttheta_{t+1}(\bar{s}_t, a^\prime) \} }{ \exp\{ \ttheta_{t+1}(\bar{s}_t, \bar{\pi}(\bar{s}_t)) \}  + \sum_{a^\prime \not= \bar{\pi}(\bar{s}_t) } \exp\{ \ttheta_{t+1}(\bar{s}_t, a^\prime) \} } \\
    &\le \frac{ \sum_{a^\prime \not= \bar{\pi}(\bar{s}_t) } \exp\{ \ttheta_{t}(\bar{s}_t, a^\prime) \} }{ \exp\{ \ttheta_{t+1}(\bar{s}_t, \bar{\pi}(\bar{s}_t)) \}  + \sum_{a^\prime \not= \bar{\pi}(\bar{s}_t) } \exp\{ \ttheta_{t}(\bar{s}_t, a^\prime) \} } \qquad \left( \text{by \cref{eq:failure_probability_softmax_gnpg_general_on_policy_stochastic_gradient_intermediate_3}} \right) \\
    &\le \frac{ \sum_{a^\prime \not= \bar{\pi}(\bar{s}_t) } \exp\{ \ttheta_{t}(\bar{s}_t, a^\prime) \} }{ \exp\big\{ \frac{ \eta \cdot (1 - \gamma)}{ \sqrt{2} } \big\} \cdot \exp\{ \ttheta_{t+1}(\bar{s}_t, \bar{\pi}(\bar{s}_t)) \}  + \sum_{a^\prime \not= \bar{\pi}(\bar{s}_t) } \exp\{ \ttheta_{t}(\bar{s}_t, a^\prime) \} } \\
    &= \frac{ \sum_{a^\prime \not= \bar{\pi}(\bar{s}_t) } \pi_{\ttheta_{t}}(a^\prime | \bar{s}_t) }{ \left( \exp\big\{ \frac{ \eta \cdot (1 - \gamma)}{ \sqrt{2} } \big\} - 1 \right) \cdot \pi_{\ttheta_{t}}( \bar{\pi}(\bar{s}_t) | \bar{s}_t)  + 1 } \\
    &= \frac{ 1 }{ \left( \exp\big\{ \frac{ \eta \cdot (1 - \gamma)}{ \sqrt{2} } \big\} - 1 \right) \cdot \pi_{\ttheta_{t}}( \bar{\pi}(\bar{s}_t) | \bar{s}_t)  + 1 } \cdot \left( 1 - \pi_{\ttheta_{t}}( \bar{\pi}(\bar{s}_t) | \bar{s}_t) \right) \\
    &\le \frac{ 1 }{ \left( \exp\big\{ \frac{ \eta \cdot (1 - \gamma)}{ \sqrt{2} } \big\} - 1 \right) \cdot \pi_{\ttheta_{1}}( \bar{\pi}(\bar{s}_t) | \bar{s}_t)  + 1 } \cdot \left( 1 - \pi_{\ttheta_{t}}( \bar{\pi}(\bar{s}_t) | \bar{s}_t) \right), \qquad \left( \text{by \cref{eq:failure_probability_softmax_gnpg_general_on_policy_stochastic_gradient_intermediate_1b}} \right)
\end{align}
which means that after one update, $1$ minus the probability of the sampled action $\bar{\pi}(\bar{s}_t)$ under $\bar{s}_t$ (by \cref{eq:failure_probability_softmax_gnpg_general_on_policy_stochastic_gradient_intermediate_5}, there must be at least one such state) is reduced by a constant. As a consequence, if a state $s$ is the $\argmax$ in \cref{eq:failure_probability_softmax_gnpg_general_on_policy_stochastic_gradient_intermediate_5} for $O(t)$ times, then we have, 
\begin{align}
\label{eq:failure_probability_softmax_gnpg_general_on_policy_stochastic_gradient_intermediate_8}
    1 - \pi_{\ttheta_t}(\bar{\pi}(s) | s) \in O(e^{- c \cdot t}),
\end{align}
where $c > 0$. Now we argue by contradiction that \cref{eq:failure_probability_softmax_gnpg_general_on_policy_stochastic_gradient_intermediate_8} holds for all state $s \in \gS$. Suppose there is at least one state $s^\prime$ which has been selected as the $\argmax$ in \cref{eq:failure_probability_softmax_gnpg_general_on_policy_stochastic_gradient_intermediate_5} for only $o(t)$ times  during the first $t$ iterations. Then for $s^\prime$, the term $1 - \pi_{\ttheta_t}(\bar{\pi}(s^\prime) | s^\prime) $ in the r.h.s. of \cref{eq:failure_probability_softmax_gnpg_general_on_policy_stochastic_gradient_intermediate_5} is at order of $\omega(e^{-c \cdot t})$, dominating the corresponding terms of other actions, which are at order of $O(e^{- c \cdot t})$ according to \cref{eq:failure_probability_softmax_gnpg_general_on_policy_stochastic_gradient_intermediate_8}. This makes the other actions cannot be selected as the $\argmax$ in \cref{eq:failure_probability_softmax_gnpg_general_on_policy_stochastic_gradient_intermediate_5} for $O(t)$ times.

Therefore, there does not exist such an state $s^\prime$ which has been selected as the $\argmax$ in \cref{eq:failure_probability_softmax_gnpg_general_on_policy_stochastic_gradient_intermediate_5} for only $o(t)$ times. And for all state $s \in \gS$, we have \cref{eq:failure_probability_softmax_gnpg_general_on_policy_stochastic_gradient_intermediate_8}.

\textbf{Third part. (iii).} Using similar arguments in first part of \cref{thm:failure_probability_softmax_gnpg_special_on_policy_stochastic_gradient}, the probability of $\bar{\pi}(s)$ is sampled forever if we run GNPG with on-policy sampling $a_t(s) \sim \pi_{\theta_t}(\cdot | s)$ is lower bounded by $\prod_{t=1}^{\infty}{ \pi_{\ttheta_t}(\bar{\pi}(s) | s) }$, where $\{ \ttheta_t \}_{t \ge 1}$ is the deterministic sequence generated by fixed sampling with the deterministic policy $\bar{\pi}$. Using similar arguments in the second part of \cref{thm:failure_probability_softmax_gnpg_special_on_policy_stochastic_gradient} and according to \cref{eq:failure_probability_softmax_gnpg_general_on_policy_stochastic_gradient_intermediate_8}, we have $\prod_{t=1}^{\infty}{ \pi_{\ttheta_t}(\bar{\pi}(s) | s) } > 0$, which means that with positive probability, $\bar{\pi}(s)$ will be sampled for all $t \ge 1$ using on-policy sampling. This implies that for all $(s,a) \in \gS \times \gA$, with positive probability, $\pi_{\theta_t}(a | s) \to 1$ as $t \to \infty$.
\end{proof}

\subsection{Committal Rates}

The following \cref{def:committal_rate_general} generalizes \cref{def:committal_rate}. The difference is we now fix the sampling to be using one fixed deterministic policy to sample forever, and then define the committal rate for all the deterministic state action pairs, generalizing the proof idea of \cref{thm:failure_probability_softmax_gnpg_general_on_policy_stochastic_gradient}.
\begin{definition}[Committal Rate]
\label{def:committal_rate_general}
Fix a reward function $r \in (0, 1]^{S \times A}$ 
and an initial parameter vector $\theta_1 \in \sR^{S \times A}$.
Consider a policy optimization algorithm $\gA$. 
Consider any deterministic policy $\bar{\pi}: s \mapsto \bar{\pi}(s)$. Under state $s$, let action $\bar{\pi}(s)$ be the sampled action \textbf{forever} after initialization and let $\theta_t$ be the resulting parameter vector obtained by using $\gA$ on the first $t$ observations.
The committal rate of algorithm $\gA$ on the deterministic policy $\bar{\pi}$ (given $r$ and $\theta_1$) is then defined as  
\begin{align}
    \kappa(\gA, \bar{\pi}) = \min_{s \in \gS} \sup\left\{ \alpha \ge 0: \limsup_{t \to \infty}{ t^\alpha \cdot  \left[ 1 - \pi_{\theta_t}(\bar{\pi}(s) | s) \right] < \infty} \right\}.
\end{align}
\end{definition}
The following \cref{thm:committal_rate_main_theorem_general} generalizes \cref{thm:committal_rate_main_theorem}.
\begin{theorem}
\label{thm:committal_rate_main_theorem_general}
Consider a policy optimization method $\gA$, together with $r\in (0,1]^{S \times A}$ and an initial parameter vector $\theta_1\in \sR^{S \times A}$.
Then,
\begin{align}
\max_{\bar{\pi}: \text{ sub-optimal deterministic}, \  \pi_{\theta_1}(a | s)>0} \kappa(\gA, \bar{\pi}) \le 1
\end{align}
is a necessary condition for ensuring the almost sure convergence of the policies obtained using $\gA$ and online sampling to the global optimum starting from $\theta_1$, where $\bar{\pi}$ under the $\max$ is for all sub-optimal deterministic policies, i.e., $\bar{\pi}(s) \not= a^*(s)$ for at least one $s \in \gS$.
\end{theorem}
\begin{proof}
The proof is a extension of \cref{thm:committal_rate_main_theorem,thm:failure_probability_softmax_gnpg_general_on_policy_stochastic_gradient}.

Let $\{ \ttheta_t \}_{t \ge 1}$ be generated by using a sub-optimal deterministic policy $\bar{\pi}: s \mapsto \bar{\pi}(s)$ to sample, and using the algorithm $\gA$ to update with the sampled actions. And let $\{ \theta_t \}_{t \ge 1}$ be generated using on-policy sampling and updating with the algorithm $\gA$.

Using similar arguments in the first part of \cref{thm:committal_rate_main_theorem}, the probability of $\bar{\pi}(s)$ is sampled forever under state $s$ using on-policy sampling is lower bounded by $\prod_{t=1}^{\infty}{ \pi_{\ttheta_t}}(\bar{\pi}(s) | s)$.

Suppose the committal rate of one sub-optimal deterministic policy is strictly larger than $1$, i.e., $\kappa(\gA, \bar{\pi}) > 1$, where $\bar{\pi}(s) \not= a^*(s)$ for at least one state $s \in \gS$. According to similar arguments in the first and second parts of \cref{thm:failure_probability_softmax_gnpg_general_on_policy_stochastic_gradient}, we have $\prod_{t=1}^{\infty}{ \pi_{\ttheta_t}}(\bar{\pi}(s) | s) > 0$.

Combining the two arguments, we have the the probability of $\bar{\pi}(s)$ is sampled forever under state $s$ using on-policy sampling is positive, which implies that for at least one state $s \in \gS$, $\pi_{\theta_t}(\cdot | s)$ will converge to some sub-optimal deterministic policies.
\end{proof}

\subsection{Geometry-Convergence Trade-off}

First, we generalize the definition of  \emph{optimality-smart} from the main paper. A policy optimization method is said to be optimality-smart if for any $t\ge 1$,
$\pi_{\tilde \theta_{t}}(a^*(s) | s)\ge \pi_{\theta_t}(a^*(s) | s)$ holds 
where $\tilde \theta_t$ is the parameter vector obtained when $a^*(s)$ is chosen in every time step, starting at $\theta_1$, while $\theta_t$ is \emph{any} parameter vector that can be obtained with $t$ updates (regardless of the action sequence chosen), but also starting from $\theta_1$.

With this definition, results similar to \cref{prop:softmax_pg_npg_gnpg_optimality_smart} hold by using similar arguments, i.e., if $a_t(s) = a^*(s)$, we have $\pi_{\theta_{t+1}}(a^*(s) | s) \ge \pi_{\theta_{t}}(a^*(s) | s)$, and otherwise if $a_t(s) \not= a^*(s)$, we have $\pi_{\theta_{t+1}}(a^*(s) | s) \le \pi_{\theta_{t}}(a^*(s) | s)$.

Next, the following result generalizes \cref{thm:committal_rate_optimal_action_special}.
\begin{theorem}
\label{thm:committal_rate_optimal_action_general}
Let $\gA$ be optimality-smart and pick a MDP instance.
If $\gA$ together with on-policy sampling leads to $\{\theta_t\}_{t\ge 1}$ such that $\{V^{\pi_{\theta_t}}(\rho)\}_{t\ge 1}$ (where $\min_{s \in \gS} \rho(s) > 0$) converges to a globally optimal policy at a rate $O(1/t^\alpha)$ with positive probability, for $\alpha > 0$, then we have $\kappa(\gA, \pi^*) \ge \alpha$, where $\pi^*$ is the global optimal deterministic policy.  
\end{theorem}
\begin{proof}
Denote $\Delta^*(s,a) = Q^*(s,a^*(s)) - Q^*(s,a)$, $\Delta^*(s) = \min_{a \not= a^*(s)}{ \Delta^*(s,a) }$, and $\Delta^* = \min_{s \in \gS}{ \Delta^*(s) } > 0$ as the optimal value gap of the MDP. According to \cref{lem:value_suboptimality}, we have,
\begin{align}
\MoveEqLeft
\label{eq:committal_rate_optimal_action_general_intermediate_1}
    V^*(\rho) - V^{\pi_{\theta_t}}(\rho) = \frac{1}{1 - \gamma} \cdot \sum_{s}{ d_\rho^{\pi_{\theta_t}}(s) \cdot  \sum_{a}{ \left( \pi^*(a | s) - \pi_{\theta_t}(a | s) \right) \cdot Q^*(s,a) } } \\
    &= \frac{1}{1 - \gamma} \cdot \sum_{s}{ d_\rho^{\pi_{\theta_t}}(s) \cdot \left[ \sum_{a}{ \pi_{\theta_t}(a | s) \cdot Q^*(s, a^*(s)) } - \sum_{a}{ \pi_{\theta_t}(a | s) \cdot Q^*(s, a) } \right] } \\
    &= \frac{1}{1 - \gamma} \cdot \sum_{s}{ d_\rho^{\pi_{\theta_t}}(s) \cdot \left[ \sum_{a \not= a^*(s)}{ \pi_{\theta_t}(a | s) \cdot Q^*(s, a^*(s)) } - \sum_{a \not= a^*(s)}{ \pi_{\theta_t}(a | s) \cdot Q^*(s, a) } \right] } \\
    &= \frac{1}{1 - \gamma} \cdot \sum_{s}{ d_\rho^{\pi_{\theta_t}}(s) \cdot \left[ \sum_{a \not= a^*(s)}{ \pi_{\theta_t}(a | s) \cdot \Delta^*(s, a) } \right] } \\
    &\ge \frac{1}{1 - \gamma} \cdot \sum_{s}{ d_\rho^{\pi_{\theta_t}}(s) \cdot \left[ \sum_{a \not= a^*(s)}{ \pi_{\theta_t}(a | s)  } \right] \cdot \Delta^*(s) } \\
    &\ge \frac{1}{1 - \gamma} \cdot \sum_{s}{ d_\rho^{\pi_{\theta_t}}(s) \cdot \left[ \sum_{a \not= a^*(s)}{ \pi_{\theta_t}(a | s)  } \right] \cdot \Delta^* }. \qquad \left(  \Delta^* \le \Delta^*(s) \right)
\end{align}
Therefore we have,
\begin{align}
\label{eq:committal_rate_optimal_action_general_intermediate_2}
    V^*(\rho) - V^{\pi_{\theta_t}}(\rho) &\ge \sum_{s}{ \rho(s) \cdot \left[ \sum_{a \not= a^*(s)}{ \pi_{\theta_t}(a | s)  } \right] \cdot \Delta^* } \qquad \left( \text{by \cref{eq:committal_rate_optimal_action_general_intermediate_1,eq:stationary_distribution_dominate_initial_state_distribution}} \right) \\
    &= \sum_{s}{ \rho(s) \cdot \left( 1 - \pi_{\theta_t}(a^*(s) | s) \right) \cdot \Delta^* } \\
    &\ge \min_{s \in \gS }{ \rho(s) } \cdot \left( 1 - \pi_{\theta_t}(a^*(s) | s) \right) \cdot \Delta^*.
\end{align}
For $\alpha>0$ let $\gE_\alpha$ be the event when for all $t\ge 1$,
\begin{align}
\label{eq:committal_rate_optimal_action_general_intermediate_3}
    V^*(\rho) - V^{\pi_{\theta_t}}(\rho) \le \frac{C}{t^\alpha}.
\end{align}
By our assumption, there exists $\alpha>0$ such that $\probability(\gE_\alpha)>0$.
On this event, for any $t\ge 1$, we have, \begin{align}
\label{eq:committal_rate_optimal_action_general_intermediate_4}
    t^\alpha \cdot \left( 1 - \pi_{\theta_t}(a^*(s) | s) \right) &\le \frac{1}{ \min_{s \in \gS }{ \rho(s) } } \cdot \frac{ t^\alpha }{ \Delta^* } \cdot \left( V^*(\rho) - V^{\pi_{\theta_t}}(\rho) \right) \qquad \left( \text{by \cref{eq:committal_rate_optimal_action_general_intermediate_2}} \right) \\
    &\le \frac{1}{ \min_{s \in \gS }{ \rho(s) } } \cdot \frac{C}{ \Delta^* }. \qquad \left( \text{by \cref{eq:committal_rate_optimal_action_general_intermediate_3}} \right)
\end{align}
Let $\big\{ \ttheta_t \big\}_{t\ge 1}$ with $\ttheta_1 = \theta_1$ be the sequence obtained by using $\gA$ with fixed sampling on $r$, such that $a_t(s) = a^*(s)$ for all $t \ge 1$. Since, by the assumption, $\gA$ is optimality-smart, we have $\pi_{\ttheta_t}(a^*(s) | s) \ge \pi_{\theta_t}(a^*(s) | s)$. Then, on $\gE_\alpha$, for any $t\ge 1$
\begin{align}
    t^\alpha \cdot \left( 1 - \pi_{\ttheta_t}(a^*(s) | s) \right) &\le t^\alpha \cdot \left( 1 - \pi_{\theta_t}(a^*(s) | s) \right) \\
    &\le \frac{1}{ \min_{s \in \gS }{ \rho(s) } } \cdot \frac{C}{\Delta}, \qquad \left( \text{by \cref{eq:committal_rate_optimal_action_general_intermediate_4}} \right)\,.
\end{align}
Since $\mathbb{P}(\gE_\alpha)>0$ and $t^\alpha \cdot \left( 1 - \pi_{\ttheta_t}(a^*(s) | s) \right)$ is non-random, it follows that for any $t\ge 1$, 
$t^\alpha \cdot \left( 1 - \pi_{\ttheta_t}(a^*(s) | s) \right) \le C/\Delta$, which, by \cref{def:committal_rate_general}, means that $\kappa(\gA, \pi^*) \ge \alpha$, where $\pi^*$ is the global optimal deterministic policy.
\end{proof}

\section{Miscellaneous Extra Supporting Results}
\label{sec:supporting_lemmas}

\begin{lemma}
\label{lem:auxiliary_lemma_1}
We have, for all $x \in (0, 1)$,
\begin{align}
    1 - x \ge e^{ - 1 / ( 1/x - 1 ) }.
\end{align}
\end{lemma}
\begin{proof}
See the proof in \citep[Proposition 1]{chung2020beyond}. We include a  proof for completeness.

We have, for all $x \in (0, 1)$,
\begin{align}
    1 - x &= \exp\left\{ \log{(1 - x)} \right\} \\
    &\ge \exp\big\{ 1 - e^{ - \log{(1 - x)} } \big\} \qquad \left( y \ge 1 - e^{-y} \right) \\
    &= \exp\Big\{ \frac{-1 }{ 1/x - 1 } \Big\}. \qedhere
\end{align}
\end{proof}

\begin{lemma}
\label{lem:positive_infinite_series}
Let $\alpha >0$. We have,
\begin{description}
\item[(i)] if $\alpha \in (1, \infty)$, then for all $C > 0$,
\begin{align}
\label{eq:positive_infinite_series_result_1}
    \sum_{t=1}^{\infty}{ \frac{C}{t^\alpha} } < \infty,
\end{align}
which means the series $\sum_{t=1}^{\infty}{ \frac{C}{t^\alpha} }$ converges to a finite value.
\item[(ii)] if $\alpha \in (0, 1]$, then for all $C > 0$,
\begin{align}
\label{eq:positive_infinite_series_result_2}
    \sum_{t=1}^{\infty}{ \frac{C}{t^\alpha} } = \infty,
\end{align}
which means the series $\sum_{t=1}^{\infty}{ \frac{C}{t^\alpha} }$ diverges to positive infinity.
\item[(iii)] for all $C > 0$, $C^\prime > 0$,
\begin{align}
\label{eq:positive_infinite_series_result_3}
    \sum_{t=1}^{\infty}{ \frac{C}{\exp\{ C^\prime \cdot t\}} } < \infty,
\end{align}
which means the series $\sum_{t=1}^{\infty}{ \frac{C}{\exp\{ C^\prime \cdot t\}} }$ converges to a finite value.
\end{description}
\end{lemma}
\begin{proof}
It is easy to verify the results by calculating integrals. We include a proof for completeness.

\textbf{First part.} We have, for all $\alpha \in (1, \infty)$ and $C > 0$,
\begin{align}
    \sum_{t=1}^{\infty}{ \frac{C}{t^\alpha} } &\le C \cdot \left( 1 + \int_{t = 1}^{\infty}{ \frac{1}{t^\alpha} dt } \right) \\
    &= \frac{C \cdot \alpha}{ \alpha - 1}.
\end{align}
\textbf{Second part.} We have, for all $\alpha \in (0 ,1)$, $C > 0$, and $T \ge 1$,
\begin{align}
    \sum_{t=1}^{T}{ \frac{C}{t^\alpha} } &\ge \int_{t = 1}^{T+1}{ \frac{C}{t^\alpha} dt } \\
    &= \frac{C \cdot \left( (T+1)^{1 - \alpha} - 1 \right) }{ 1 - \alpha}.
\end{align}
Similarly, for $\alpha = 1$,
\begin{align}
    \sum_{t=1}^{T}{ \frac{C}{t} } &\ge \int_{t = 1}^{T+1}{ \frac{C}{t} dt } \\
    &= C \cdot \log{(T+1)}.
\end{align}
Therefore, the partial sum approaches to positive infinity as $T \to \infty$.

\textbf{Third part.} We have, for all $C > 0$ and $C^\prime > 0$,
\begin{align}
    \sum_{t=1}^{\infty}{ \frac{C}{\exp\{ C^\prime \cdot t\}} } &\le \int_{t=0}^{\infty}{ \frac{C}{\exp\{ C^\prime \cdot t\}} } \\
    &= \frac{C}{C^\prime}. \qedhere
\end{align}
\end{proof}

\begin{lemma}
\label{lem:positive_infinite_product_2}
Let $u_t \in (0, 1)$ for all $t \ge 1$. The infinite product $\prod_{t=1}^{\infty}{\left( 1 - u_t \right) }$ converges to a positive value if and only if the series $\sum_{t=1}^{\infty}{ u_t }$ converges to a finite value.
\end{lemma}
\begin{proof}
See \citet[p. 220]{knopp1947theory}. We include a proof for completeness.

Define the following partial products and partial sums,
\begin{align}
\label{eq:positive_infinite_product_2_intermediate_1a}
    p_T &\coloneqq \prod_{t=1}^{T}{\left( 1 - u_t \right) }, \\
\label{eq:positive_infinite_product_2_intermediate_1b}
    s_T &\coloneqq \sum_{t=1}^{T}{ u_t }.
\end{align}
Since $p_T$ is monotonically decreasing and non-negative, the infinite product converges to positive values, i.e.,
\begin{align}
    \prod_{t=1}^{\infty}{\left( 1 - u_t \right) } = \lim_{T \to \infty}{ \prod_{t=1}^{T}{\left( 1 - u_t \right) } } = \lim_{T \to \infty}{p_T} > 0,
\end{align}
if and only if $p_T$ is lower bounded away from zero (boundedness convergence criterion for monotone sequence) \citep[p. 80]{knopp1947theory}.

Similarly, since $s_T$ is monotonically increasing, the series converges to finite values, i.e.,
\begin{align}
    \sum_{t=1}^{\infty}{ u_t } = \lim_{T \to \infty}{ \sum_{t=1}^{T}{ u_t } } = \lim_{T \to \infty}{s_T} < \infty,
\end{align} if and only if $s_T$ is upper bounded.

\textbf{First part.} $\prod_{t=1}^{\infty}{\left( 1 - u_t \right) }$ converges to a positive value only if $\sum_{t=1}^{\infty}{ u_t }$ converges to a finite value.

Suppose $\prod_{t=1}^{\infty}{\left( 1 - u_t \right) }$ converges to a positive value. We have, for all $T \ge 1$,
\begin{align}
    q_T \ge q > 0.
\end{align}
Then we have,
\begin{align}
    q &\le q_T \\
    &= \exp\bigg\{ \log{ \bigg( \prod_{t=1}^{T}{\left( 1 - u_t \right) } \bigg) } \bigg\} \\
    &= \exp\bigg\{ \sum_{t=1}^{T}{ \log{ \left( 1 - u_t \right) } } \bigg\} \\
    &\le \exp\bigg\{ - \sum_{t=1}^{T}{ u_t } \bigg\} \qquad \left( \log{\left( 1 - x \right) } < -x \right) \\
    &= \exp\{ - s_T \},
\end{align}
which implies that,
\begin{align}
    s_T \le - \log{q} < \infty.
\end{align}
Therefore, we have $\sum_{t=1}^{\infty}{ u_t }$ converges to a finite value.

\textbf{Second part.} $\prod_{t=1}^{\infty}{\left( 1 - u_t \right) }$ converges to a positive value if $\sum_{t=1}^{\infty}{ u_t }$ converges to a finite value.

Suppose $\sum_{t=1}^{\infty}{ u_t }$ converges to a finite value. Then we have, $u_t \to 0$ as $t \to \infty$. There exists a finite number $t_0 \ge 1$, such that for all $t \ge t_0$, we have $u_t \le 1/2$. Also, we have, for all $T \ge 1$,
\begin{align}
    s_T \le s < \infty.
\end{align}
Then we have,
\begin{align}
    \prod_{t=t_0}^{T}{\left( 1 - u_t \right) } &= \exp\bigg\{ \sum_{t=t_0}^{T}{ \log{ \left( 1 - u_t \right) } } \bigg\} \\
    &\ge \exp\bigg\{ - \sum_{t=t_0}^{T}{ 2 \cdot u_t } \bigg\} \qquad \left( - 2 \cdot x \le \log{\left( 1 - x \right) } \text{ for all } x \in [0, 1/2] \right) \\
    &= \exp\{ - 2 \cdot s_T \},
\end{align}
which implies that, for all large enough $T \ge 1$,
\begin{align}
    q_T &= \left( \prod_{t=1}^{t_0-1}{\left( 1 - u_t \right) } \right) \cdot \left( \prod_{t=t_0}^{T}{\left( 1 - u_t \right) } \right) \\
    &\ge \left( \prod_{t=1}^{t_0-1}{\left( 1 - u_t \right) } \right) \cdot \exp\{ - 2 \cdot s_T \} \\
    &\ge \left( \prod_{t=1}^{t_0-1}{\left( 1 - u_t \right) } \right) \cdot \exp\{ - 2 \cdot s \} \\
    &> 0.
\end{align}
Therefore, we have $\prod_{t=1}^{\infty}{\left( 1 - u_t \right) }$ converges to a positive value.
\end{proof}

\begin{lemma}
\label{lem:positive_infinite_product_3}
Let $u_t \in (0, 1)$ for all $t \ge 1$. We have  $\prod_{t=1}^{\infty}{\left( 1 - u_t \right) } = \lim_{T \to \infty}{ \prod_{t=1}^{T}{\left( 1 - u_t \right) } } = 0 $
if and only if the series $\sum_{t=1}^{\infty}{ u_t }$ diverges to positive infinity.
\end{lemma}
\begin{proof}
\textbf{First part.} $\prod_{t=1}^{\infty}{\left( 1 - u_t \right) }$ diverges to $0$ only if $\sum_{t=1}^{\infty}{ u_t }$ diverges to positive infinity.

Suppose $\prod_{t=1}^{\infty}{\left( 1 - u_t \right) }$ diverges to $0$. According to \cref{lem:positive_infinite_product_2}, $\sum_{t=1}^{\infty}{ u_t }$ diverges. And since the partial sum $s_T \coloneqq \sum_{t=1}^{T}{ u_t }$ is monotonically increasing, we have $\sum_{t=1}^{\infty}{ u_t }$ diverges to positive infinity. 

\textbf{Second part.} $\prod_{t=1}^{\infty}{\left( 1 - u_t \right) }$ diverges to $0$ if $\sum_{t=1}^{\infty}{ u_t }$ diverges to a positive infinity.

Suppose $\sum_{t=1}^{\infty}{ u_t }$ diverges to positive infinity. According to \cref{lem:positive_infinite_product_2}, $\prod_{t=1}^{\infty}{\left( 1 - u_t \right) }$ diverges. And since the partial product $q_T \coloneqq \prod_{t=1}^{T}{\left( 1 - u_t \right) }$ is non-negative and monotonically decreasing, we have $\prod_{t=1}^{\infty}{\left( 1 - u_t \right) }$ diverges to $0$. 
\end{proof}

The following \cref{lem:positive_infinite_product_4} indicates that if $\pi_{\theta_t}(a)$ approaches $1$ \textbf{slowly}, i.e., no faster than $O(1/t)$, then the probability of sampling $a$ forever using on-policy sampling $a_t \sim \pi_{\theta_t}(\cdot)$ is \textbf{zero}, i.e., the other actions $a^\prime \not= a$ always have a chance to be sampled.
\begin{lemma}
\label{lem:positive_infinite_product_4}
Let $\pi_{\theta_t}(a) \in (0, 1)$ be the probability of sampling action $a$ using online sampling $a_t \sim \pi_{\theta_t}(\cdot)$, for all $t \ge 1$. If $1 - \pi_{\theta_t}(a) \in \Theta(1/t^\alpha)$ with $\alpha \in [0, 1]$, then $\prod_{t=1}^{\infty}{ \pi_{\theta_t}(a) } = 0$.
\end{lemma}
\begin{proof}
Suppose $1 - \pi_{\theta_t}(a) \in \Theta(1/t^\alpha)$ and $\alpha \in (0, 1]$. Let $u_t \coloneqq 1 - \pi_{\theta_t}(a) \in (0, 1)$ for all $t \ge 1$. According to \cref{lem:positive_infinite_series}, we have,
\begin{align}
    \sum_{t=1}^{\infty}{ u_t } = \sum_{t=1}^{\infty}{ \left( 1 - \pi_{\theta_t}(a) \right) } = \infty,
\end{align}
i.e., the series diverges to positive infinity. According to \cref{lem:positive_infinite_product_3}, we have,
\begin{align}
    \prod_{t=1}^{\infty}{ \pi_{\theta_t}(a) } =  \prod_{t=1}^{\infty}{ \left( 1 - u_t \right) } = 0,
\end{align}
which means it is impossible to sample $a$ forever using on-policy sampling $a_t \sim \pi_{\theta_t}(\cdot)$.
\end{proof}

\begin{lemma}[Performance difference lemma \citep{kakade2002approximately}]
\label{lem:performance_difference_general}
For any policies $\pi$ and $\pi^\prime$,
\begin{align}
    V^{\pi^\prime}(\rho) - V^{\pi}(\rho) &= \frac{1}{1 - \gamma} \cdot \sum_{s}{ d_\rho^{\pi^\prime}(s) \cdot \sum_{a}{ \left( \pi^\prime(a | s) - \pi(a | s) \right) \cdot Q^{\pi}(s,a) } }\\
    &= \frac{1}{1 - \gamma} \cdot \sum_{s}{ d_{\rho}^{\pi^\prime}(s) \cdot \sum_{a}{ \pi^\prime(a | s) \cdot A^{\pi}(s, a) } }.
\end{align}
\end{lemma}
\begin{proof}
According to the definition of value function,
\begin{align}
\MoveEqLeft
    V^{\pi^\prime}(s) - V^{\pi}(s) = \sum_{a}{ \pi^\prime(a | s) \cdot Q^{\pi^\prime}(s,a) } - \sum_{a}{ \pi(a | s) \cdot Q^{\pi}(s,a) } \\
    &= \sum_{a}{ \pi^\prime(a | s) \cdot \left( Q^{\pi^\prime}(s,a) - Q^{\pi}(s,a) \right) } + \sum_{a}{ \left( \pi^\prime(a | s) - \pi(a | s) \right) \cdot Q^{\pi}(s,a) } \\
    &= \sum_{a}{ \left( \pi^\prime(a | s) - \pi(a | s) \right) \cdot Q^{\pi}(s,a) } + \gamma \cdot \sum_{a}{ \pi^\prime(a | s) \cdot \sum_{s^\prime}{  \gP( s^\prime | s, a) \cdot \left[ V^{\pi^\prime}(s^\prime) -  V^{\pi}(s^\prime)  \right] } } \\
    &= \frac{1}{1 - \gamma} \cdot \sum_{s^\prime}{ d_{s}^{\pi^\prime}(s^\prime) \cdot \sum_{a^\prime}{ \left( \pi^\prime(a^\prime | s^\prime) - \pi(a^\prime | s^\prime) \right) \cdot Q^{\pi}(s^\prime, a^\prime) }  } \\
    &= \frac{1}{1 - \gamma} \cdot \sum_{s^\prime}{ d_{s}^{\pi^\prime}(s^\prime) \cdot \sum_{a^\prime}{ \pi^\prime(a^\prime | s^\prime) \cdot \left( Q^{\pi}(s^\prime, a^\prime) - V^{\pi}(s^\prime) \right) }  } \\
    &= \frac{1}{1 - \gamma} \cdot \sum_{s^\prime}{ d_{s}^{\pi^\prime}(s^\prime) \cdot \sum_{a^\prime}{ \pi^\prime(a^\prime | s^\prime) \cdot A^{\pi}(s^\prime, a^\prime) } }. \qedhere
\end{align}
\end{proof}

\begin{lemma}[Value sub-optimality lemma]
\label{lem:value_suboptimality}
For any policy $\pi$,
\begin{align}
    V^*(\rho) - V^{\pi}(\rho) = \frac{1}{1 - \gamma} \cdot \sum_{s}{ d_\rho^{\pi}(s) \cdot \sum_{a}{ \left( \pi^*(a | s) - \pi(a | s) \right) \cdot Q^*(s,a) } }.
\end{align}
\end{lemma}
\begin{proof}
See the proof in \citep[Lemma 21]{mei2020global}. We include a proof for completeness.

We denote $V^*(s) \coloneqq V^{\pi^*}(s)$ and $Q^*(s,a) \coloneqq Q^{\pi^*}(s,a)$ for conciseness. We have, for any policy $\pi$,
\begin{align}
\MoveEqLeft
    V^*(s) - V^{\pi}(s) = \sum_{a}{ \pi^*(a | s) \cdot Q^*(s,a) } - \sum_{a}{ \pi(a | s) \cdot Q^{\pi}(s,a) } \\
    &= \sum_{a}{ \left( \pi^*(a | s) - \pi(a | s) \right) \cdot Q^*(s,a) } + \sum_{a}{ \pi(a | s) \cdot \left(  Q^*(s,a) - Q^{\pi}(s,a) \right) } \\
    &= \sum_{a}{ \left( \pi^*(a | s) - \pi(a | s) \right) \cdot Q^*(s,a) } + \gamma \cdot \sum_{a}{ \pi(a | s) \cdot \sum_{s^\prime}{  \gP( s^\prime | s, a) \cdot \left[ V^{\pi^*}(s^\prime) -  V^{\pi}(s^\prime)  \right] } } \\
    &= \frac{1}{1 - \gamma} \cdot \sum_{s^\prime}{ d_{s}^{\pi}(s^\prime) \cdot \sum_{a^\prime}{ \left( \pi^*(a^\prime | s^\prime) - \pi(a^\prime | s^\prime) \right) \cdot Q^*(s^\prime, a^\prime) }  }. \qedhere
\end{align}
\end{proof}

\section{The Intuition of Committal Rate Definition}

The following \cref{fig:committal_rate_intuition} illustrates the intuition of the committal rate definition. Using fixed sampling, we decouple the coupling between sampling and updating, and then focus on only characterizing the aggressiveness of different update rules.

\begin{figure*}[ht]
\centering
\vskip -0.05in
\includegraphics[width=0.9\linewidth]{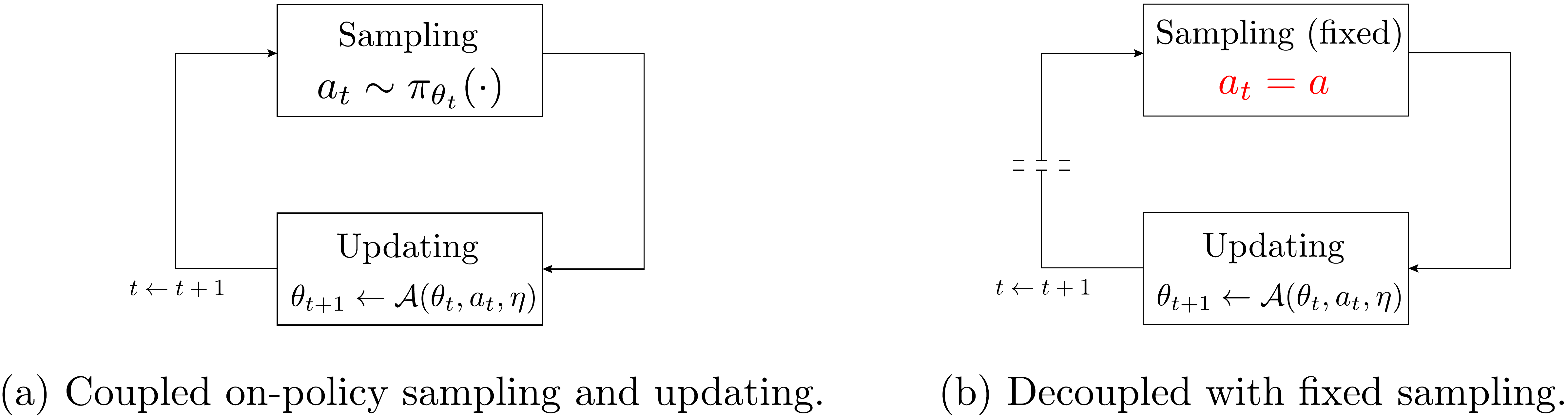}
\vskip -0.05in
\caption{
An illustration for the intuition of the committal rate definition.
}
\label{fig:committal_rate_intuition}
\vskip -0.05in
\end{figure*}

\end{document}